\documentclass[11pt]{article}

\usepackage[margin=1in]{geometry}
\usepackage{setspace}
\usepackage{amsmath,amsfonts,amssymb}
\usepackage{amsthm}
\usepackage{graphicx}
\usepackage{booktabs}
\usepackage{algorithm}
\usepackage{algorithmic}
\usepackage{array}
\usepackage{textcomp}
\usepackage{mathtools}
\usepackage{xcolor}
\usepackage{nicefrac}
\usepackage{microtype}
\usepackage{hyperref}

\RequirePackage[
   backend=biber,
   style=authoryear,
   giveninits=true,
   uniquename=init
   ]{biblatex}

\theoremstyle{plain}
\newtheorem{theorem}{Theorem}[section]
\newtheorem{proposition}[theorem]{Proposition}

\newtheorem{remark}[theorem]{Remark}

\theoremstyle{definition}
\newtheorem{definition}[theorem]{Definition}

\title{\Large \textbf{RLAF: Reinforcement Learning from Automaton Feedback}}

\author{
\textbf{Mahyar Alinejad}, University of Central Florida, \texttt{mahyar.alinejad@ucf.edu} \and
\textbf{Alvaro Velasquez}, University of Colorado Boulder, \texttt{alvaro.velasquez@colorado.edu} \and
\textbf{Yue Wang}, University of Central Florida, \texttt{yue.wang@ucf.edu} \and
\textbf{George Atia}, University of Central Florida, \texttt{george.atia@ucf.edu}
}

\date{}

\begin{document}
\sloppy
\maketitle

\begin{abstract}
Reinforcement Learning (RL) in environments with complex, history-dependent reward structures poses significant challenges for traditional methods. In this work, we introduce a novel approach that leverages automaton-based feedback to guide the learning process, replacing explicit reward functions with preferences derived from a deterministic finite automaton (DFA). Unlike conventional approaches that use automata for direct reward specification, our method employs the structure of the DFA to generate preferences over trajectories that are used to learn a reward function, eliminating the need for manual reward engineering. Our framework introduces a static approach that uses the learned reward function directly for policy optimization and a dynamic approach that involves continuous refining of the reward function and policy through iterative updates until convergence.
  Our experiments in both discrete and continuous environments demonstrate that our approach enables the RL agent to learn effective policies for tasks with temporal dependencies, outperforming traditional reward engineering and automaton-based baselines such as reward machines and LTL-guided methods. Our results highlight the advantages of automaton-based preferences in handling non-Markovian rewards, offering a scalable, efficient, and human-independent alternative to traditional reward modeling. We also provide a convergence guarantee showing that under standard assumptions our automaton-guided preference-based framework learns a policy that is near-optimal with respect to the true non-Markovian objective.
\end{abstract}

\section{Introduction}
Reinforcement Learning (RL) has achieved remarkable success across various domains, ranging from game playing~\cite{Mnih2015HumanLevel, Silver2016Go, Vinyals2019AlphaStar} to robotic control~\cite{Levine2016EndToEnd, Lillicrap2016Continuous} and autonomous driving~\cite{Kendall2019Learning}. However, most RL algorithms rely on explicit reward functions carefully designed to specify the agent's objectives~\cite{SuttonBarto}. In practice, designing such reward functions can be difficult, especially in environments where the desired outcomes depend on long-term dependencies or complex sequences of actions~\cite{Littman2017Environment, Camacho2019LTLMOP}. Such environments exhibit \textit{history-dependent} or \textit{non-Markovian} reward structures, where rewards are not solely determined by the current state but by the trajectory leading to that state~\cite{Bacchus1996LearningML, Icarte2022reward}. These 
structures pose significant challenges to traditional RL methods 
designed to operate within a Markovian framework~\cite{ToroIcarte2018UsingRL, SuttonBarto}.

To address this gap, we propose a novel approach that bypasses the need for manually specifying explicit reward functions. Instead, we leverage automaton-based feedback to generate preferences over trajectory segments. Our method uses a deterministic finite automaton (DFA) to capture complex temporal dependencies in tasks~\cite{Oncina1992, Angluin1987Learning}, providing structured feedback to guide the learning process. The agent then scores complete trajectories by their alignment with subgoals, yielding rewards that mirror the automaton’s encoded preferences~\cite{ToroIcarte2018UsingRL,Camacho2019LTLMOP,Li2017Reinforcement}.

\smallbreak
\noindent\textbf{Motivation and Challenges.} In many real-world scenarios, reward signals are sparse or difficult to specify~\cite{Christiano2017Deep, Andrychowicz2017Hindsight}. For instance, robotic assembly requires executing steps in order, with little feedback from intermediate actions~\cite{Kaelbling1993Learning, Hart2009Learning}. Such settings naturally induce non-Markovian rewards, making it hard for RL algorithms to infer optimal policies from state-action transitions alone~\cite{Littman2017Environment, Bacchus1996LearningML}. Traditional approaches often assume dense rewards or rely on manual reward shaping~\cite{Ng1999PolicyInvariant, Taylor2009TransferLearning}, which is labor-intensive and can introduce bias~\cite{Rusu2015PolicyDistillation, Barto2003HRL}.

We abstract task specifications into a DFA and use its structure to \emph{automatically generate trajectory preferences}. From these preferences we then learn a reward function, which is used for policy optimization. Crucially, the automaton-derived score is \emph{not} a reward: it is only an ordinal signal that ranks which trajectory is closer to satisfying the specification, and it is never added to returns or used as shaping. This removes the manual calibration of transition bonuses and penalties while still yielding a learned reward amenable to standard RL algorithms.

\noindent\textbf{Relation to Prior Formal-Methods RL.} 
\noindent\textbf{Relation to Prior Formal-Methods RL.} Prior formal-methods approaches include reward machines (RMs), which encode task structure as finite-state reward automata and typically require mapping specifications to numeric rewards~\cite{ToroIcarte2018UsingRL, Icarte2022reward}, and LTL-based RL that compiles temporal-logic formulas into automata and optimizes probabilities of satisfaction, again using numeric rewards or acceptance conditions~\cite{Li2017TLTL, Hasanbeig2019LTL}. In contrast, we avoid numeric reward design: the automaton is used solely to \emph{rank} trajectories by adherence to the specification, and the reward used for learning is \emph{learned from these preferences} rather than engineered. This matters because even with a formal specification, assigning suitable numeric values often requires calibration to prevent unintended behaviors. Moreover, while expert reward labels are scarce, natural-language guidance can be converted to DFAs (e.g., NL2TL~\cite{Chen2023NL2TL}). Prior work has used DFAs as advice to resolve conflicting preferences or to fine-tune policies without ranking~\cite{Neider2021AdviceGuided, Yang2023FineTuning}; here, automata \emph{generate} the preferences that drive reward learning, aligning policies with task logic without direct reward shaping.

\smallbreak
\noindent\textbf{Contributions.} Our work presents a new framework combining formal methods with modern RL techniques to address tasks with complex, history-dependent requirements, connecting task specification and policy optimization. Our contributions are summarized as follows.

\noindent\textbf{1. Automaton-based preference learning:} We introduce a framework for generating preferences over trajectory segments using a DFA, enabling the RL agent to learn reward functions that capture task-specific temporal dependencies without human intervention. 

\noindent\textbf{2. Scoring functions for preference generation:} We propose various scoring functions derived from the DFA representation namely, (i) \emph{subtask-based scoring} in which preferences are assessed based on the number of subtasks the agent has completed and a distance metric to the next subtask in a sequence of decomposable tasks, and (ii) \emph{automaton transition value-based scoring} where we 
assign values to specific transitions within the automaton. The latter approach not only facilitates preference generation but also supports efficient transfer learning between different environments that share the same objective, enhancing the agent’s ability to generalize and adapt to new settings.

\noindent\textbf{3. Static and dynamic learning variants:} We present two variants of our method. In the static version, the learned reward function is applied directly for policy optimization. In the dynamic version, the reward function and policy are iteratively refined in a loop until convergence, providing a robust framework for complex tasks~\cite{Sutton1999Options, Bacon2017OptionCritic}. 

\noindent\textbf{4. Theoretical guarantee:} We prove that, assuming preference consistency, reward expressivity, and sufficient exploration, our automaton-guided preference-based framework learns a policy that is $\varepsilon$-optimal with respect to the true non-Markovian objective (See Theorem 
\ref{thm:main-convergence}).

\noindent\textbf{5. Empirical validation:} We evaluate our approach in both grid-based and continuous domains, showing that it learns effective policies under non-Markovian rewards. Results demonstrate that automaton-based preferences provide a scalable, efficient, and human-independent alternative to traditional reward shaping, often outperforming methods like reward machines~\cite{Ng1999PolicyInvariant, Taylor2009TransferLearning} and LTL-based approaches.

\section{Related Work}
\noindent\textbf{Reward learning from preferences:}
Preference-based learning offers an alternative to explicit reward signals. The seminal work \cite{Christiano2017Deep} introduced reward learning from human-ranked trajectory segments, reducing reliance on manual reward design but still requiring human input. Building on this, \cite{Zhu2023Principled} proposed a principled RLHF framework using pairwise or K-wise comparisons. Recent advances include Direct Preference Optimization (DPO) \cite{Rafailov2023DirectPO}, Contrastive Preference Learning \cite{Park2022SURF}, active preference elicitation \cite{Biyik2020Active}, ensemble uncertainty estimation \cite{Gleave2022Uncertainty}, and multi-objective alignment \cite{Leike2018Scalable}. Classic approaches include learning from demonstrations \cite{Ho2016Generative} and inverse RL \cite{Abbeel2004Apprenticeship}.
Our method extends this line by replacing human feedback with automaton-based preferences \cite{Oncina1992, Angluin1987Learning}, eliminating human involvement in reward learning while retaining structured feedback. Though DFAs require initial domain knowledge, our approach removes humans from the learning loop itself.

\noindent\textbf{Non-Markovian reward decision processes:}
Non-Markovian rewards pose challenges in RL, where rewards typically depend only on the current state~\cite{SuttonBarto}. Solutions include augmenting state representations with history~\cite{Bakker2002Reinforcement, Wierstra2007Solving} or designing reward functions over action sequences~\cite{Dupont1996Incremental, Li2017Reinforcement}. For instance, \cite{Littman2017Environment} translated temporal logic into automata for RL integration. Reward machines~\cite{Icarte2022reward} represent reward structures via finite-state machines, with QRM enabling efficient learning through counterfactual reasoning. However, these methods still require manually converting logic into numeric rewards, which can be error-prone and require tuning. Our approach instead learns rewards from automaton-based preferences~\cite{Walkinshaw2016Inferring, Xu2020Learning}, avoiding reward translation while leveraging the automaton’s structure, eliminating manual reward design.

\noindent\textbf{Logic-based RL:}
LTL has been widely used to formalize task requirements in RL. SPECTRL \cite{Jothimurugan2019Composable} uses a depth-based measure for sequential tasks, while methods like LPOPL \cite{Hasanbeig2018LogicGuided} and TLTL \cite{Li2017TLTL} embed temporal logic into RL objectives. Recent work includes goal-conditioned RL with temporal logic \cite{Qiu2023Instructing}, compositional RL from logic \cite{Jothimurugan2021Compositional}, and automata embeddings for RL \cite{Yalcinkaya2024Compositional}. Unlike these methods, which embed logic directly into rewards or acceptance criteria, often causing reward sparsity or scale issues, our approach uses the automaton to generate trajectory preferences. This enables learning reward functions with structured guidance while avoiding manual reward engineering.

\noindent\textbf{Neurosymbolic RL:}
Symbolic structures like automata have been used to formalize task specifications in RL, particularly under temporal logic constraints~\cite{Camacho2019LTLMOP, ToroIcarte2018UsingRL}. They enable agents to reason about action sequences and enforce constraints~\cite{Baier2008Principles, Li2017Reinforcement}. \cite{Hasanbeig2021DeepSynth} integrated automata with deep RL for safety, while shielding methods use automata to block unsafe actions~\cite{Alshiekh2018Safe}. Yang et al.~\cite{Yang2023FineTuning} fine-tuned language models with automaton-based controllers guided by natural language, and \cite{Neider2021AdviceGuided} used DFAs as advice to resolve conflicting preferences in sparse-reward settings. Reward-conditioning guides behavior via reward-to-go, and hierarchical RL~\cite{Barto2003HRL, Bacon2017OptionCritic} decomposes tasks using formal guidance. While these methods use automata to constrain or guide policies, our approach uniquely leverages automata to generate trajectory preferences for reward learning, aligning policies with task logic without direct reward shaping~\cite{Hahn2019Omega, Icarte2022reward}. 

\noindent\textbf{Transfer learning in structured environments:}
Transfer in RL has been studied via policy distillation~\cite{Rusu2015PolicyDistillation}, feature transfer~\cite{Barreto2017SuccessorFeatures}, and meta-learning~\cite{Finn2017MAML}. For structured tasks, modular policies guided by task sketches~\cite{Andreas2017Modular} and transfer across LTL-specified tasks~\cite{Neider2021AdviceGuided} have been explored. Closest to our work, reward machines~\cite{Icarte2022reward} enable sharing reward structures across tasks. We extend this line by showing that automaton-based preferences can effectively transfer task knowledge across environments using our automaton transition value-based scoring. Unlike prior methods that transfer full policies or value functions, our approach transfers knowledge via automaton transitions, offering a more modular and interpretable mechanism.

\begin{figure}[!t]
\centering
\includegraphics[width=5in]{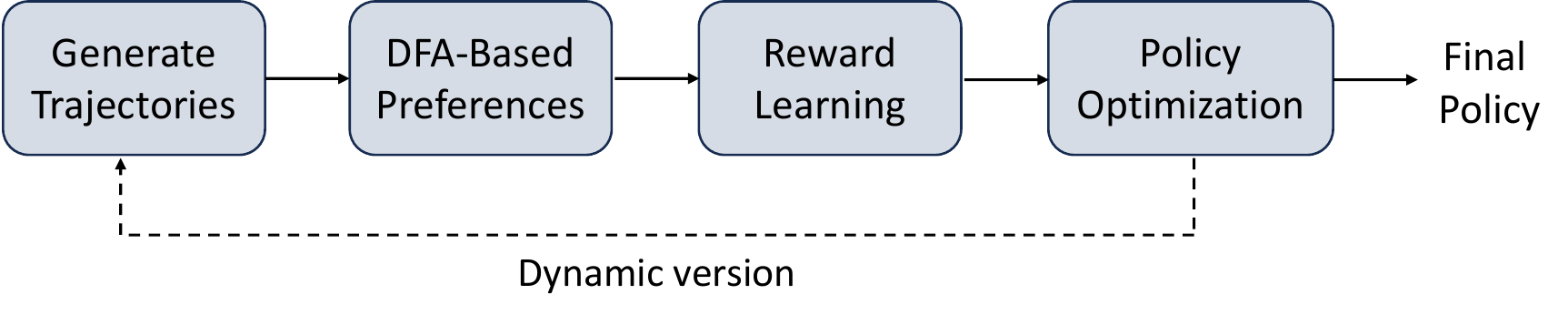}
\vspace{-0cm}
\caption{Automaton-based RL framework: The agent generates trajectories, evaluated by the DFA to produce preferences for reward learning via pairwise ranking loss. The learned reward guides policy optimization. In the dynamic version (dashed arrow), this loop repeats until convergence; in the static version, the policy is optimized once after reward learning.}
\label{fig:method_diagram}
\end{figure}

\section{Problem Setup}
\begin{definition}
A Non-Markovian Reward Decision Process (NMRDP) is a decision process defined by the tuple $\mathcal{M} = (S, s_0, A, T, r)$, where $S$ is a finite set of states, $s_0 \in S$ is the initial state, $A$ is the finite set of actions, $T: S \times A  \rightarrow \Delta(S)$ is the transition probability function, specifying the probability of reaching state $s' \in S$ given state $s \in S$ and action $a \in A$, and $r: (S \times A)^* \rightarrow\mathbb{R}$ is a reward function that depends on the history of states and actions. 
\end{definition}

The reward signal of an NMRDP incorporates non-Markovian dependencies, allowing it to rely on the full sequence of previous states and actions rather than solely on the current state-action pair. 

We also define a labeling function $L:S\to 2^{AP}$ that maps each state to the set of atomic propositions true in that state, where $AP$ is a finite set of environment predicates (e.g., $\{\textsf{Goal},\textsf{Hazard},\textsf{DoorOpen},\textsf{CarryingItem}\}$). For example, in a navigation domain, $L(s)$ might be $\{\textsf{DoorOpen}\}$ or $\{\textsf{Goal},\textsf{CarryingItem}\}$. The label sequence $(L(s_0),L(s_1),\ldots)$ forms a word over $2^{AP}$ on which a temporal-logic spec (e.g., LTL) is interpreted and compiled to a DFA used by our method.

The agent policy $\pi: S \rightarrow \Delta(A)$, where $\Delta(A)$ is the probability simplex on $A$, determines the probability of selecting a given action in a certain state. The goal is to maximize the expected cumulative reward:
\begin{align}
V_{R}(\pi) 
\;=\;
\mathbb{E}_{\tau\sim\pi}\!\Bigl[R(\tau)\Bigr]
\;=\;
\mathbb{E}_{\tau\sim\pi}\!\Bigl[\sum_{t=0}^{\infty} \gamma^t \, r(\tau_t)\Bigr]\:,
\end{align}
where $\tau$ is a trajectory, $\tau_t = (s_0, a_0, \dots, s_t, a_t)$ is its prefix up to time $t$, and $0 < \gamma < 1$ is a discount factor. 

However, in our framework, the reward function $R$ is not explicitly defined. Instead, the agent relies on automaton-based preferences over trajectory segments, which are used to guide learning. This setup is particularly effective for handling environments with non-Markovian rewards, where task objectives depend on the entire sequence of states and actions rather than just the current state.

\begin{definition}[Deterministic Finite Automaton (DFA)] A DFA is a finite-state machine defined by the tuple $\mathcal{A} = (\Sigma, Q, q_0, \delta, F)$, where $\Sigma$ is the alphabet of the input language, $Q$ is a finite set of automaton states with starting state $q_0$, 
$\delta: Q \times \Sigma \rightarrow Q$ is the transition function, 
and $F \subseteq Q$ is the set of accepting states.
\end{definition}
The DFA defines a high-level structure for evaluating trajectories based on the sequence of state-action pairs. It provides a formalized representation of the task objectives, enabling trajectory evaluation by observing transitions through the DFA states. The sequence $\tau = (s_0, a_0, s_1, a_1, \dots, s_T)$ results in corresponding DFA transitions, allowing the agent to track progress and alignment with task goals.

\begin{definition}[Product MDP]
 The Product MDP is represented by $\mathcal{M}_{\mathrm{prod}} = (S \times Q, A, T_{\mathrm{prod}}, (s_0, q_0), R_{\mathrm{prod}})$, where $S$ is the state space of the NMRDP,  
 $Q$ is the DFA state space, 
 $A$ is the set of available actions, $ T_{\mathrm{prod}}: (S \times Q) \times A \rightarrow \Delta(S \times Q)$ is the transition function, 
 and $ R_{\mathrm{prod}}: Q\times\Sigma\to\mathbb{R}$ is a reward function, where $\Sigma$ is the alphabet of the input language of the DFA.  
\end{definition}

The Product MDP integrates the NMRDP and DFA, allowing the agent to act while adhering to task-specific temporal constraints. The transition function $T_{\mathrm{prod}}$ combines environment dynamics with DFA transitions, reflecting subgoal progress. The DFA state $q\in Q$ tracks task progression, while the NMRDP state $s\in \mathcal{S}$ represents the environment's physical state. The reward function $R_{\mathrm{prod}}$ either reflects subgoal achievements in a known reward setup or is derived through preference learning. Embedding DFA logic into the environment's dynamics enables standard RL methods to handle non-Markovian tasks. Our learning approaches leverage the Product MDP to enforce temporal objectives while learning preference-based rewards.

\section{Methodology}
In this section, we describe our automaton-based RL framework, focusing on its core components: (i) Preference elicitation from a DFA, (ii) Reward function learning, and (iii) Policy optimization. We 
also introduce our static and dynamic learning variants. Figure \ref{fig:method_diagram} illustrates our framework.

\subsection{Preference Elicitation}
Our method generates preferences over pairs of trajectories $(\tau_1, \tau_2)$ by evaluating their adherence to task specifications using the DFA. Unlike RLHF, where preferences are derived from human input, our approach leverages the DFA to assess trajectories without direct access to rewards. To rank trajectories, we define a composite score function based on the task structure encoded in the DFA. Importantly, the score is used only to induce an ordering over trajectories: it need not be calibrated or accurate in magnitude -- only consistent enough to indicate which trajectory is closer to satisfying the specification. This ordinal criterion avoids numeric reward shaping. We present two distinct approaches for generating these preference scores.

\noindent\textbf{Subtask-based scoring.}
Here, trajectories are scored based on two factors:  
(i) \textit{subgoal completion}, which is the number $N_s$ of subgoals completed in the correct order as tracked by the DFA, and  
(ii) \textit{distance to the next subgoal or goal}, which is  a distance measure $d$ from the last state of the trajectory to the next required subgoal or goal.
Formally, the score for a trajectory $\tau$ is computed as:
\begin{align}
\text{score}(\tau) = w_s \cdot N_s(\tau) - w_d \cdot d(\tau),
\label{eq:score_subtask_based}
\end{align}
where $w_s$ and $w_d$ are weighting factors, with $w_s \gg w_d$, prioritizing subgoal completion over distance.

Given two trajectories $\tau_1$ and $\tau_2$, their scores are compared to establish a preference ($\tau_1 \succ \tau_2$ or vice-versa). These preferences are then used to learn a reward function (see Section~\ref{sec:learn_reward}).

\begin{remark}
The DFA plays a central role in scoring by tracking subgoal completion and distance calculation. In particular, the DFA maps the trajectory to a sequence of states, recording transitions as the agent completes subgoals in order, providing $N_s(\tau)$, and is used to determine the next required subgoal, enabling the computation of $d(\tau)$.

\end{remark}

\begin{remark}

We initially set $d$ as the Manhattan distance for discrete gridworlds, however, our framework is adaptable to various environments. For continuous state spaces, Euclidean distance is a more suitable metric.

For more complex domains with non-trivial dynamics, learned cost functions or domain-specific metrics can be integrated. For instance, in robotic manipulation tasks, a geodesic distance in configuration space might be more appropriate than direct Euclidean distance.
\end{remark}

\noindent\textbf{Automaton transition value-based scoring.}
In this approach, we integrate \emph{automaton transition values} into the trajectory scoring function. Assume that for each transition $(q, \sigma)$ in the DFA, we have access to estimates $\overline{Q}_{\text{dfa}}(q, \sigma)$, representing the desirability of the automaton transition. These values can either be derived directly during training or obtained via a transfer learning process, as detailed later. 

Given a trajectory $\tau = (s_0, a_0, s_1, a_1, \dots, s_T)$, the corresponding trace in the DFA is $(q_0, \sigma_0, q_1, \sigma_1, \dots, q_T)$, where $q_0$ is the initial automaton state, $q_{t+1} = \delta(q_t, \sigma_t)$, for $t = 0, \dots, T-1$, and $\sigma_t = L(s_{t+1})$, recalling that  $\delta$ is the DFA transition function and $L$ is the labeling function which maps a given state to the set of atomic propositions that hold in said state.

The value-based trajectory score is computed as:
\begin{align}
\label{eq:score_qvalue_based}
\text{score}(\tau)
=
\sum_{t=0}^{T-1} 
Q_{\text{dfa}}\bigl(q_t, \sigma_t)\bigr).
\end{align}
Each automaton transition $(q_t, \sigma_t)$ represents a partial condition being satisfied in the environment. The value $Q_{\text{dfa}}(q_t, \sigma_{t})$ reflects how desirable this transition is, based on previously learned or transferred knowledge. Therefore, by summing over all transitions, the scoring function in \eqref{eq:score_qvalue_based} captures the cumulative desirability of the trajectory with respect to the automaton’s objectives. 

\begin{remark}
\label{rem:transition_value_heuristic}
The scoring function in Equation~\eqref{eq:score_qvalue_based} leverages teacher knowledge to generate trajectory preferences. 
Specifically, for each automaton transition $(q, \sigma)$, we compute:
\begin{align*}
Q_{\text{dfa}}(q, \sigma) = \frac{1}{|\mathcal{T}_{q,\sigma}|}\sum_{((s,q),a) \in \mathcal{T}_{q,\sigma}} Q_{\text{teacher}}((s,q), a)\:,
\end{align*}
where $\mathcal{T}_{q,\sigma}$ is the set of product state-action pairs in the teacher's experience that induce the automaton transition $(q, \sigma)$, and $Q_{\text{teacher}}((s,q), a)$ is the Q-value for the state-action pair $((s,q), a)$ in the teacher environment. This transition value captures the average desirability of achieving that logical transition based on teacher experience.
The trajectory score $\sum_{t=0}^{T-1} Q_{\text{dfa}}(q_t, \sigma_t)$ aggregates these transition values to quantify the cumulative desirability of the symbolic path. Trajectories that progress through more valuable automaton transitions receive higher scores. This scoring mechanism provides three key benefits: (i) it transfers task-relevant knowledge from teacher to student environments sharing the same logical structure, (ii) it remains interpretable as each term corresponds to a meaningful subgoal transition, and (iii) it scales efficiently as the complexity depends on the compact automaton structure rather than the full state space.
This preference-generation score 
aims to induce a consistent preference ordering that aligns with task objectives. Our empirical validation in Section~\ref{sec:exp} 
confirms this through two key findings: 
(1) the scores correlate strongly with trajectory quality metrics (cumulative reward and task completion), and (2) critically, trajectories with higher scores leave the agent in states with significantly better future task completion potential (future success rate, future reward).

\end{remark} 

As before, preferences between two trajectories, $\tau_1$ and $\tau_2$, are determined by comparing their scores in \eqref{eq:score_qvalue_based}. Specifically, $\tau_1 \succ \tau_2$ if $\text{score}(\tau_1) > \text{score}(\tau_2)$, and vice versa. In the case of a tie, they are considered equally preferable.  These trajectory-level preferences are then used to train a parametric reward model (see Section~\ref{sec:learn_reward}). 

\emph{But how can we obtain automaton Q-value estimates $Q_{\text{dfa}}(q, \sigma)$ if these values are not readily available?}

\noindent\textbf{Obtaining automaton transition values: Knowledge distillation.} 
Automaton transition values can be distilled from a simpler teacher environment in which the agent shares the same task objective. We employ a teacher-student transfer learning framework, where the knowledge encoded in automaton transitions is transferred from the teacher to the student. 

Consider an agent trained in a simpler teacher environment using standard RL methods to optimize its behavior according to the task’s temporal logic. During training, experiences $\mathcal{D}$ are stored in the form of samples $((s, q), a, r, (s', q'))$, where $s$ and $s'$ are environment states, $q$ and $q'$ are automaton states, and $a$ is the action taken.
To distill knowledge, the frequency of automaton transitions is tracked, defined for automaton state $q$ and the atomic propositions $\sigma$ that label transitions:
\begin{align}
\label{eq:freq_transitions}
n_{\text{teacher}}(q, \sigma) = \bigl| \bigl\{ ((s, q), a, r, (s', q')) \in \mathcal{D} \; \big| \; L(s') = \sigma \bigr\} \bigr|.
\end{align}

The value for each automaton transition $(q, \sigma)$, where $\sigma = L(s')$, is then computed as $Q_{\text{dfa}}(q, \sigma)
=\overline{Q}_{\text{teacher}}(q, \sigma)$, where
\begin{align}
\overline{Q}_{\text{teacher}}(q, \sigma):= \frac{\sum_{((s, q), a, r, (s', q')) \in \mathcal{D}, L(s') = \sigma} Q_{\text{teacher}}((s,q), a)}{n_{\text{teacher}}(q, \sigma)}\:.
\label{eq:aut_qvals}
\end{align}

\begin{remark}
We could also use a combined trajectory scoring function that integrates the subtask-based and value-based approaches defined in \eqref{eq:score_subtask_based} and \eqref{eq:score_qvalue_based}, respectively. The combined score function is defined as:
\begin{align}
\label{eq:score_transfer_learning}
\text{score}(\tau) = w_s N_s(\tau) - w_d d(\tau) + w_q \sum_{t=0}^{T-1} \overline{Q}_{\text{teacher}}\bigl(q_t, \sigma_t\bigr),
\end{align}
where $w_s$, $w_d$, and $w_q$ are weighting factors for subgoal completion, distance, and automaton Q-values, respectively, which can be tuned to prioritize
different aspects of the task. 
\end{remark}

\subsection{Learning the Reward Function}
\label{sec:learn_reward}
Using the preferences derived from the DFA, next we aim to learn a reward function $\hat{r}_\theta((s, q), a)$, parameterized by $\theta$, that captures the task’s objectives. This reward function is trained via a pairwise ranking loss, ensuring that trajectories preferred by the DFA receive higher cumulative rewards.

\textbf{Pairwise ranking loss.}
We learn the reward function $\hat{r}_\theta$ by minimizing the pairwise ranking loss:
\begin{align}
L(\theta) = \sum_{(\tau_p, \tau_n)} \max \left(0, m - (\hat{R}_\theta(\tau_p) - \hat{R}_\theta(\tau_n))\right),
\label{eq:ranking_loss}
\end{align}
where the sum is over pairs of preferred and non-preferred trajectories, $\tau_p$ and $\tau_n$, respectively, and $m$ is a margin hyperparameter, ensuring a sufficient difference between the cumulative rewards of these trajectories. By minimizing this loss, we ensure that the learned reward function aligns with the DFA-derived preferences, thereby embedding the task’s temporal structure into the reward model. For each pair of trajectories $(\tau_p, \tau_n)$, the cumulative discounted reward is computed as:
\begin{align}
\hat{R}_\theta(\tau) = \sum_{t=0}^{T-1} \gamma^t \hat{r}_\theta((s_t, q_t), a_t),
\end{align}
where $\hat{r}_\theta((s, q), a)$ is the learned reward function and $T$ is the length of the trajectory.

\subsection{Policy Optimization: Static and Dynamic Variants}
After learning the reward function $\hat{r}_\theta$, the policy $\pi$ is optimized using standard RL techniques to maximize the expected cumulative reward. In our experiments, we employ Q-learning for discrete environments and Twin Delayed DDPG (TD3) \cite{Fujimoto2018Td3} for continuous domains.

We propose two learning variants (See Fig. \ref{fig:method_diagram}). Specifically, in the \textbf{static version}, the reward function is learned once from an initial set of preferences. The policy is then optimized using this fixed reward function. This approach is computationally efficient and well-suited for tasks where the reward structure remains constant over time. In the \textbf{dynamic version}, the reward function and policy are refined iteratively. At each iteration, trajectories from the current policy are evaluated by the DFA to produce preferences, which update the reward function. The policy is then re-optimized using the updated reward. This loop continues until convergence and suits tasks needing ongoing refinement due to evolving dynamics or incomplete initial preferences.

We provide algorithms for our preference-based RL approach. Algorithm \ref{alg:app_unified_aut_pref} summarizes the main procedure, and Algorithm \ref{alg:app_preference_computation} describes the preference computation based on the DFA. 

\begin{algorithm}
\caption{RL with Automaton-Based Preferences and Pairwise Ranking Loss}
\label{alg:app_unified_aut_pref}
\begin{algorithmic}
\REQUIRE MDP $(S, A, P, \gamma)$, DFA $\mathcal{A}$, $\alpha$, $\beta$, $m$, Maximum iterations $K$ (Dynamic) or episodes $E$ (Static)
\ENSURE Optimal policy $\pi^*$
\STATE \textbf{Phase 1: Preference Generation \& Reward Learning}
\IF{\textbf{Static mode}}
    \STATE Initialize random policy $\pi_{\mathrm{rand}}$
    \STATE Generate trajectories $\{\tau_i\}$ via $\pi_{\mathrm{rand}}$
    \STATE Assign preferences $P(\tau_i, \tau_j)$ for sampled trajectory pairs using DFA $\mathcal{A}$ (Algorithm \ref{alg:app_preference_computation})
    \STATE Train reward model $\hat{r}_\theta$ via pairwise ranking loss
\ELSIF{\textbf{Dynamic mode}}
    \STATE Initialize policy $\pi$
    \FOR{$k = 1$ to $K$}
        \STATE Collect trajectories $\{\tau_i\}$ via current $\pi$
        \STATE Assign preferences $P(\tau_i, \tau_j)$ for trajectory pairs using DFA $\mathcal{A}$
        \STATE Update reward model $\hat{r}_\theta$ with new preferences
        \STATE Optimize policy $\pi$ via RL (Q-learning or TD3) on $\hat{r}_\theta$
        \IF{policy performance is stable}
            \STATE \textbf{break}
        \ENDIF
    \ENDFOR
\ENDIF
\STATE \textbf{Phase 2: Policy Optimization (Static Mode Only)}
\IF{\textbf{Static mode}}
    \FOR{$e = 1$ to $E$}
        \STATE Optimize policy $\pi$ via RL using $\hat{r}_\theta$
    \ENDFOR
\ENDIF
\RETURN final policy $\pi^*$
\end{algorithmic}
\end{algorithm}

\begin{algorithm}
\caption{Preference Computation Based on Subtask Completion and Distance}
\label{alg:app_preference_computation}
\begin{algorithmic}
\REQUIRE Trajectory pairs $\{(\tau_i, \tau_j)\}$, DFA $\mathcal{A}$, Subgoals $\mathcal{G}=[g_1, g_2, \dots, g_N]$, Goal state $s_{\text{goal}}$, Weights $w_s$, $w_d$
\ENSURE Preferences $P(\tau_i, \tau_j)$ for each trajectory pair
\FOR{each trajectory pair $(\tau_i, \tau_j)$}
    \FOR{each trajectory $\tau \in \{\tau_i, \tau_j\}$}
        \STATE Initialize DFA $\mathcal{A}$ to its initial state
        \STATE $N_s(\tau) \gets$ Number of subtasks completed in order by simulating $\mathcal{A}$ with $\tau$
        \STATE $s_T \gets$ Last state of trajectory $\tau$
        \IF{$N_s(\tau) < N$}
            \STATE $s_{\text{next}} \gets$ Next subgoal $g_{N_s(\tau)+1}$
        \ELSE
            \STATE $s_{\text{next}} \gets s_{\text{goal}}$
        \ENDIF
        \STATE $d(\tau) \gets \text{ComputeDistance}(s_T, s_{\text{next}})$
        \STATE $\text{score}(\tau) \gets w_s \cdot N_s(\tau) - w_d \cdot d(\tau)$
    \ENDFOR
    \STATE Assign $P(\tau_i, \tau_j)$ to the trajectory with the higher score, or mark as indifferent if scores are equal
\ENDFOR
\RETURN $P(\tau_i, \tau_j)$ for all trajectory pairs
\end{algorithmic}
\end{algorithm}



\section{Convergence Analysis}
\label{sec:thm}
In this section, we provide a convergence analysis of our DFA-guided preference-based RL framework. We derive a theorem establishing that, under some assumptions, pairwise preference learning guided by a DFA allows an agent to learn a reward function whose optimal policy is close to that of the (unknown) ground-truth objective. 


\begin{theorem}[Convergence and $\varepsilon$-Optimality in Finite Product MDP]
\label{thm:main-convergence}

Consider a finite non-Markovian environment 
$\mathcal{M} = (S, s_0, A, T, R^*)$, where $S$ is a finite set of states, 
$A$ is a finite set of actions, $T : S \times A \to \Delta(S)$ 
is a transition function, and $R^*(\tau)$ is a \emph{trajectory-return} 
encoding a non-Markovian objective.  
Let $\mathcal{A} = (\Sigma, Q, q_0, \delta, F)$ be a finite DFA 
that captures the temporal constraints or subgoal structure. 
We construct the \emph{product MDP} $\mathcal{M}_{\mathrm{prod}}$ 
with state space $\widetilde{S} = S \times Q$, actions $A$, 
and transitions combining $T$ with $\delta$.  
Since $\widetilde{S}$ and $A$ are finite, $\mathcal{M}_{\mathrm{prod}}$ 
is a finite MDP. Let $\Pi$ be the set of all stationary policies on $\mathcal{M}_{\mathrm{prod}}$. 
Assume:

\begin{enumerate}
\item[\textbf{(1)}] \textbf{Preference Consistency.}  
There exists a real-valued function $R^*(\tau)$ such that 
if $\tau_1 \succ \tau_2$ (according to DFA-based preferences), 
then 
\[
R^*(\tau_1) \;-\; R^*(\tau_2) \;\;\ge\;\; \delta_0 \;>\; 0.
\]

\item[\textbf{(2)}] \textbf{Sufficient Expressivity.}  
We can learn a reward 
$\hat{r}_\theta : \widetilde{S} \times A \to \mathbb{R}$  
such that its induced return  
$\hat{R}_\theta(\tau) 
= \sum_{t=0}^{T-1} \gamma^t \,\hat{r}_\theta(\widetilde{s}_t, a_t)$  
can represent the same preference ordering 
as $R^*$ on the relevant set of trajectories.

\item[\textbf{(3)}] \textbf{Correct Preference Training.}  
By sampling a sufficiently large set of labeled trajectory pairs 
$\{(\tau_p, \tau_n)\}$ from a coverage distribution, 
we minimize a pairwise ranking loss so that, 
with probability at least $1-\zeta$, 
\[
\Bigl|\hat{R}_\theta(\tau) \;-\; R^*(\tau)\Bigr| 
\;\;\le\; \varepsilon_r
\]
for most trajectories $\tau$ under that distribution.
(This ensures the ordering $\tau_1 \succ \tau_2$ is preserved 
except on a small fraction of pairs.)

\item[\textbf{(4)}] \textbf{Persistent Exploration (Tabular).}  
In tabular Q-learning on $\mathcal{M}_{\mathrm{prod}}$ using $\hat{r}_\theta$, 
\emph{each state--action pair} $(\widetilde{s}, a) \in \widetilde{S} \times A$ 
is visited infinitely often.
\end{enumerate}

Let $\hat{\pi}$ be the greedy policy w.r.t.\ the Q-values 
learned from $\hat{r}_\theta$.  Then, there is an 
$\varepsilon = \varepsilon_r + \varepsilon_t > 0$ 
such that, with probability at least $1-\zeta$,
\[
V_{R^*}(\hat{\pi})
\;\;\ge\;
\max_{\pi \,\in\, \Pi}\, V_{R^*}(\pi)
\;-\;
\varepsilon,
\]
i.e.\ $\hat{\pi}$ is $\varepsilon$-optimal w.r.t.\ the true 
non-Markovian objective $R^*$.

\end{theorem}

\begin{proof}

\textbf{Step 1: Consistency of the Learned Return.}

By assumption \textbf{(3)}, the pairwise ranking loss is minimized 
on a sufficiently large set of trajectory pairs 
$\{(\tau_p, \tau_n)\}$ consistent with $R^*(\tau)$.  
Hence, with probability at least $1-\zeta$, 
we obtain a function $\hat{r}_\theta$ such that its cumulative 
$\hat{R}_\theta(\tau)$ satisfies:
\[
\tau_1 \succ \tau_2
\quad\Longrightarrow\quad
\hat{R}_\theta(\tau_1) \;>\; \hat{R}_\theta(\tau_2),
\]
except possibly on a small fraction of pairs.  
Moreover, $\bigl|\hat{R}_\theta(\tau) - R^*(\tau)\bigr| 
\le \varepsilon_r$ 
for most $\tau$ under the sampling distribution.  
The margin $\delta_0>0$ in \textbf{(1)} ensures that small 
errors in $\hat{R}_\theta(\tau)$ do not flip the preference 
unless those errors exceed $\delta_0$.  

\textbf{Step 2: Q-Learning Convergence in the Finite Product MDP.}

Define the reward in $M_{\mathrm{prod}}$ as 
$\hat{r}_\theta(\widetilde{s},a)$ 
for $\widetilde{s}\in \widetilde{S},\, a\in A$.  
Since $\widetilde{S}$ and $A$ are finite, tabular Q-learning applies.  
By assumption \textbf{(4)}, \emph{each} pair $(\widetilde{s},a)$ 
is visited infinitely often.  
According to Watkins' theorem \cite{Watkins1992}, 
the Q-values converge almost surely to the unique fixed point 
$Q^*_{\hat{r}_\theta}$ satisfying the Bellman optimality 
equation for $\hat{r}_\theta$.  
Let $\hat{\pi}$ be the greedy policy w.r.t.\ $Q^*_{\hat{r}_\theta}$; 
then $\hat{\pi}$ maximizes 
\[
V_{\hat{R}_\theta}(\pi)
\;=\;
\mathbb{E}_{\tau\sim\pi}\!\Bigl[\hat{R}_\theta(\tau)\Bigr].
\]

\textbf{Step 3: Relating $\hat{R}_\theta$-Optimality to $R^*$-Optimality.}

Let $\pi^* \in \Pi$ be any policy (potentially optimal) w.r.t.\ $R^*$.  
We show:
\[
V_{R^*}(\hat{\pi})
\;\;\ge\;
V_{R^*}(\pi^*)
\;-\;
(\varepsilon_r + \varepsilon_t).
\]

Since $\hat{\pi}$ is optimal for $\hat{R}_\theta$, we have
\[
V_{\hat{R}_\theta}(\hat{\pi})
\;=\;
\max_{\pi\in\Pi} V_{\hat{R}_\theta}(\pi).
\]
On the trajectories visited by $\hat{\pi}$, the difference 
$\bigl|\hat{R}_\theta(\tau) - R^*(\tau)\bigr|$ is at most 
$\varepsilon_r$ (with high probability), so
\[
V_{R^*}(\hat{\pi})
\;\;\ge\;\;
V_{\hat{R}_\theta}(\hat{\pi})
\;-\;
\varepsilon_t,
\]
where $\varepsilon_t$ captures any residual misalignment or 
low-probability outliers.  Next,
\[
V_{\hat{R}_\theta}(\hat{\pi})
\;=\;
\max_{\pi\in\Pi} V_{\hat{R}_\theta}(\pi)
\;\;\ge\;
V_{\hat{R}_\theta}(\pi^*).
\]
Again relating $\hat{R}_\theta$ to $R^*$ on $\pi^*$'s trajectories,
\[
V_{\hat{R}_\theta}(\pi^*)
\;\;\ge\;\;
V_{R^*}(\pi^*) \;-\; \varepsilon_r.
\]
Combining,
\begin{align}
V_{R^*}(\hat{\pi})
&\ge
\bigl[V_{\hat{R}_\theta}(\hat{\pi}) - \varepsilon_t\bigr] \nonumber \\
&\ge
\bigl[V_{\hat{R}_\theta}(\pi^*) - \varepsilon_t\bigr] \nonumber \\
&\ge
\bigl[V_{R^*}(\pi^*) - (\varepsilon_t + \varepsilon_r)\bigr].
\end{align}
Hence,
\[
V_{R^*}(\hat{\pi})
\;\;\ge\;
V_{R^*}(\pi^*)
\;-\;
(\varepsilon_t + \varepsilon_r).
\]
Defining $\varepsilon := \varepsilon_r + \varepsilon_t$ 
yields the claimed \(\varepsilon\)-suboptimality bound 
for $\hat{\pi}$ with high probability $(1 - \zeta)$.  
Since $\pi^*$ was arbitrary, in particular
\[
V_{R^*}(\hat{\pi})
\;\;\ge\;
\max_{\pi\in\Pi} V_{R^*}(\pi)
\;-\; \varepsilon.
\]
This completes the proof.
\end{proof}

\section{Experimental Results}
\label{sec:exp}
We evaluate our framework using six methods: static, dynamic, known reward function, distillation with reward shaping, reward machine, and Logic-Guided Policy Optimization (LPOPL). The goal is to optimize policies in gridworlds with subgoals, obstacles, and a final goal, and analyzing performance on tasks with sequential objectives across both discrete and continuous state spaces.

For the \emph{known reward function} method, the agent has access to a hand-crafted reward and directly optimizes its policy, serving as a performance benchmark. Inspired by structured task decomposition~\cite{Ng1999PolicyInvariant, singh2010intrinsically, konidaris2009skill}, the reward encourages completing subgoals in order: $+3$ for correct subgoals, $-1$ for out-of-order ones, $+9$ for reaching the final goal if all subgoals were completed in order ($-3$ otherwise), and $-0.1$ per step to promote efficiency. This design reflects hierarchical RL ideas that leverage prior task structure~\cite{dietterich2000hierarchical}.

We introduce a  method termed \emph{distillation with reward shaping}, adapted from automaton-based RL~\cite{Icarte2022reward, ToroIcarte2018UsingRL, singireddy2023automaton}. It neither relies on preference comparisons nor manual rewards. It uses the given DFA for minimal numeric shaping: $+1$ for successful subgoal transitions and $-0.1$ otherwise.  
The agent refines its policy solely from DFA signals.

The \emph{reward machine} baseline implements the approach from \cite{Icarte2022reward}, where the automaton directly defines rewards based on state transitions. Each valid transition in the automaton that represents progression toward the goal receives a positive reward ($+1$), while invalid transitions receive no reward. Additionally, a small step penalty ($-0.1$) encourages efficient solutions. This approach directly translates the automaton structure into a reward function without learning from preferences. 

The \emph{LPOPL} baseline~\cite{Hasanbeig2018LogicGuided} shapes rewards using LTL specifications converted into automata, assigning rewards based on progress toward satisfying the logic formula. In contrast to our preference-based approach, LPOPL directly encodes logical progression into the reward function (see Appendices \ref{appdx:distill}, \ref{appdx:rm}, and \ref{appdx:lpopl} for further details about these baselines).

\textbf{Experimental setup.} 
We evaluate our approach in six environments spanning diverse domain types to demonstrate broad applicability and scalability. Four gridworld environments (Minecraft Iron Sword Quest, Dungeon Quest, Blind Craftsman, and Minecraft Building Bridge) are visualized in Fig.~\ref{fig:combined_environment_visualization} 
along with their DFAs in Figs~\ref{fig:dfa_minecraft}-\ref{fig:dfa_bridge}. The first involves a single sequence of subgoals, while the others introduce complexity through multiple valid paths and loops. To validate broader applicability, we extend evaluation to two additional environments representing fundamentally different domain types: Mountain Car Collection, a physics-based terrain navigation task with energy management constraints, and Warehouse Robotics, implementing a realistic 12-dimensional continuous state space, highlighting scalability for robotic applications (See Fig.~\ref{fig:dfa_mountain_car}-\ref{fig:dfa_warehouse} for the corresponding DFAs). All six tasks require completing subgoals in specific temporal orders, defined by DFAs that encode non-Markovian objectives. These environments highlight the limitations of simple state-based rewards and demonstrate the advantages of DFA-based preferences across discrete navigation, physics-based control, and high-dimensional continuous domains.

We provide the description and visualizations of the six environments along with their DFA representations. Each DFA encodes the logical structure of subgoal dependencies and task constraints, guiding the agent in task completion. The states represent distinct progress levels, with transitions triggered by specific subgoal achievements. These automata serve as an interpretable abstraction of the environment dynamics, ensuring structured exploration and efficient policy learning.

\begin{figure}[!htbp]
\centering

\begin{minipage}{0.24\textwidth}
  \centering
  \includegraphics[width=\linewidth]{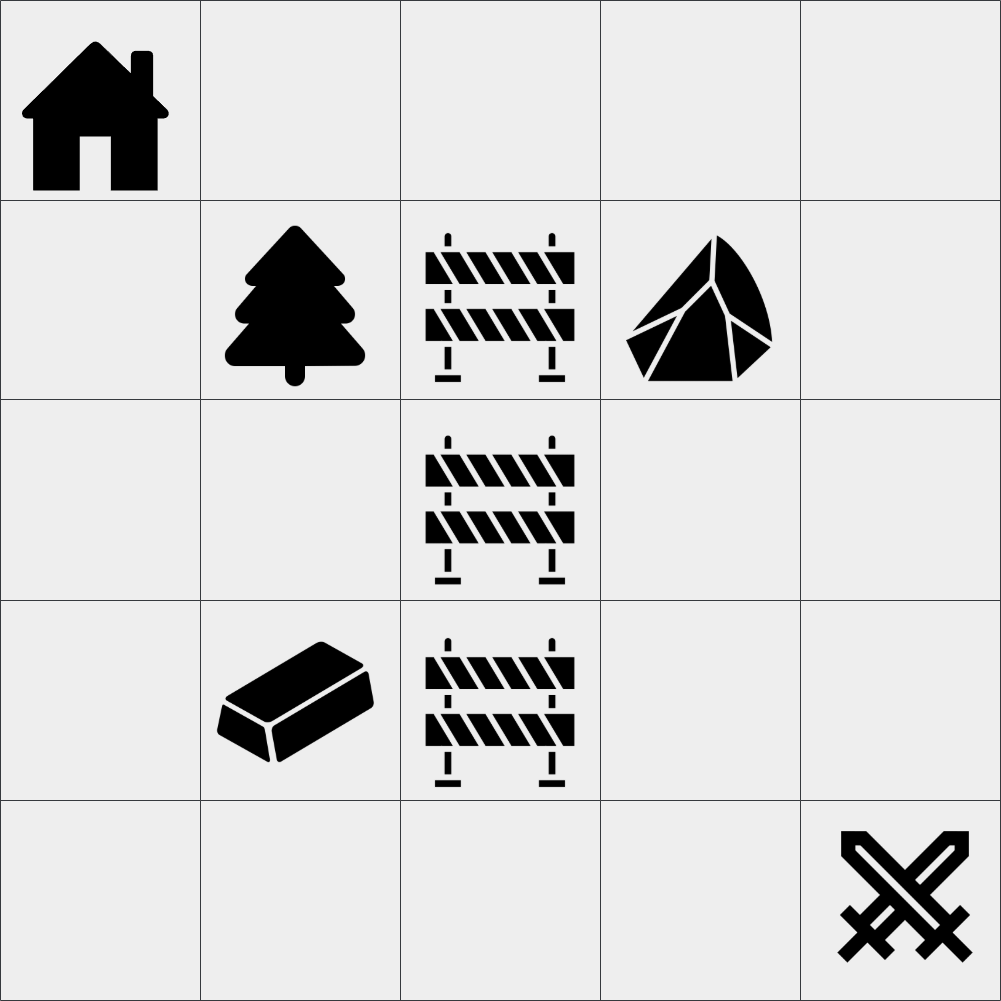}
  \vspace{0.1cm}
  \label{fig:minecraft_grid_world}
\end{minipage}%
\hfill
\begin{minipage}{0.24\textwidth}
  \centering
  \includegraphics[width=\linewidth]{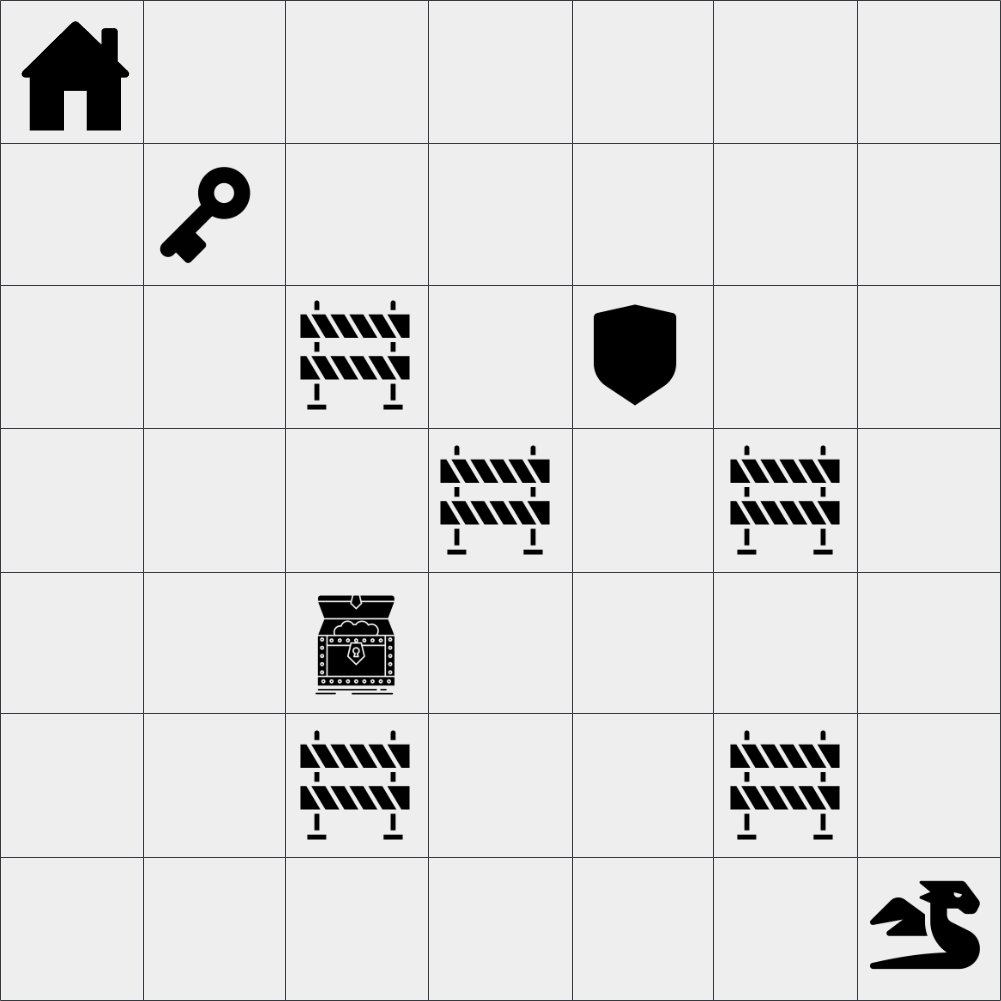}
  \vspace{0.1cm}
  \label{fig:dungeon_quest_grid_world}
\end{minipage}%
\hfill
\begin{minipage}{0.24\textwidth}
  \centering
  \includegraphics[width=\linewidth]{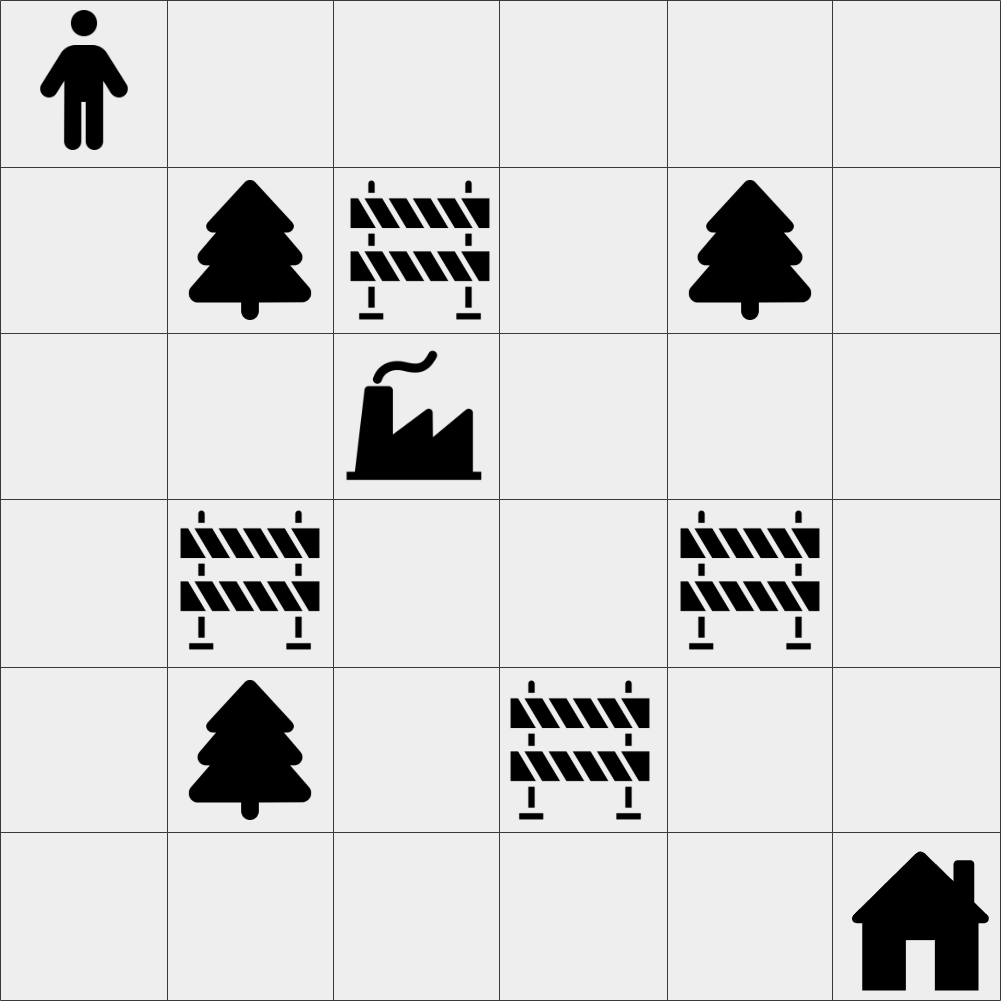}
  \vspace{0.1cm}
  \label{fig:blind_craftsman_grid_world}
\end{minipage}%
\hfill
\begin{minipage}{0.24\textwidth}
  \centering
  \includegraphics[width=\linewidth]{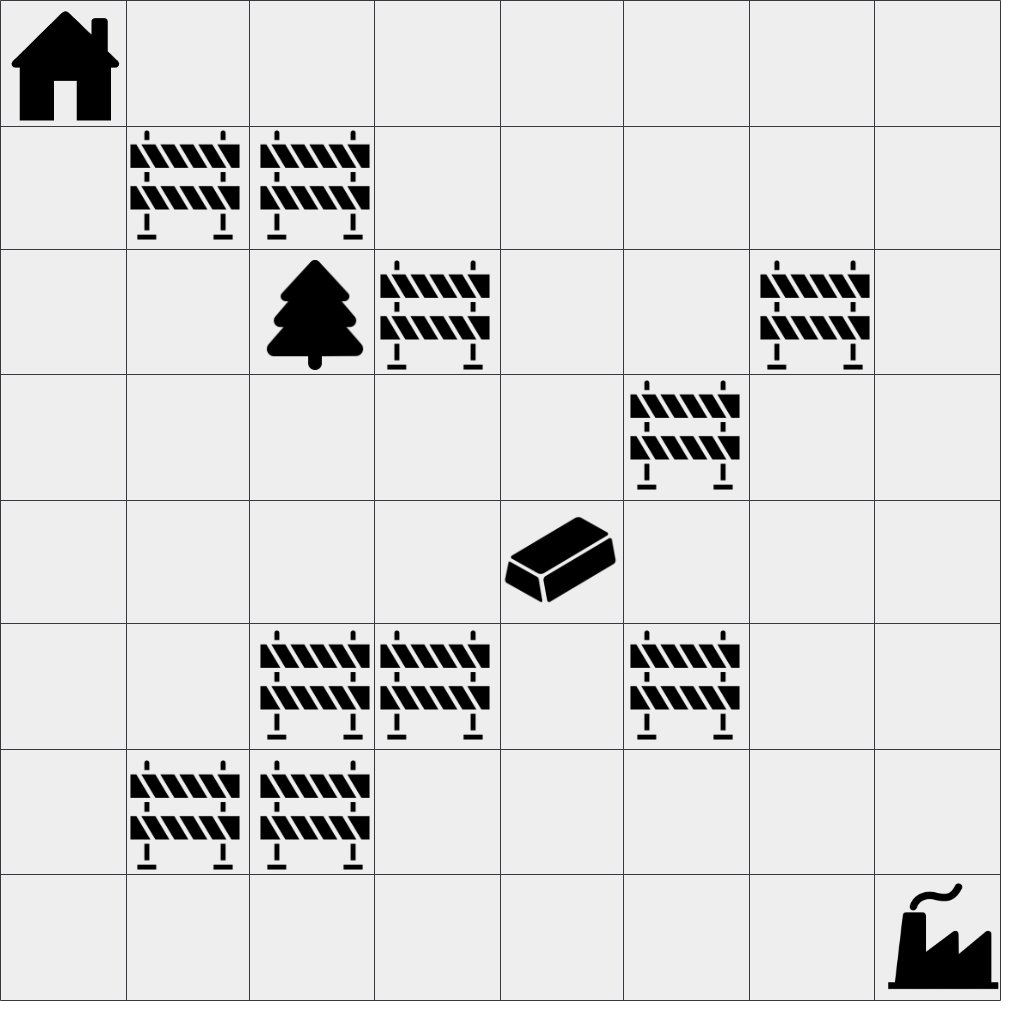}
  \vspace{0.1cm}
  \label{fig:building_bridge_grid_world}
\end{minipage}
\caption{\small{Visualization of the four environments: Minecraft Iron Sword Quest, Dungeon Quest, Blind Craftsman, and Minecraft Building Bridge (left to right).}} 
\label{fig:combined_environment_visualization}
\end{figure}

\textbf{1) Minecraft Iron Sword Quest (5$\times$5).}
The agent starts at a \emph{Home Base} and must reach a \emph{Crafting Table}, completing subgoals in a strict order while navigating obstacles.  
It must gather wood to craft a Wooden Pickaxe, use it to mine stone and craft a Stone Pickaxe, and finally mine iron ore to craft the sword at the Crafting Table. This type of structured environment
has been widely explored~\cite{Andreas2017Modular}, including tasks like \emph{Make Bed} and \emph{Make Axe}.

\begin{figure}[!t]
    \centering
    \includegraphics[width=5in]{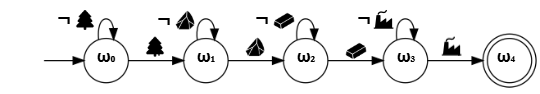}
    \caption{DFA for \emph{Minecraft Iron Sword Quest}. The automaton enforces the correct order of resource collection before crafting the sword.}
    \label{fig:dfa_minecraft}
\end{figure}

\textbf{2) Dungeon Quest (7$\times$7).}
The agent navigates a grid to collect items and defeat a dragon, following a strict sequence of subgoals \cite{singireddy2023automaton, alinejad2025bidirectional}.  
The agent must first obtain a \emph{Key} to unlock the \emph{Chest}, retrieve the \emph{Sword} from the Chest, and collect the \emph{Shield} for protection. The Dragon can only be defeated with both the Shield and Sword. These dependencies are captured in the DFA, with transitions triggered by item acquisition. 

\begin{figure}[!t]
    \centering
    \includegraphics[width=5in]{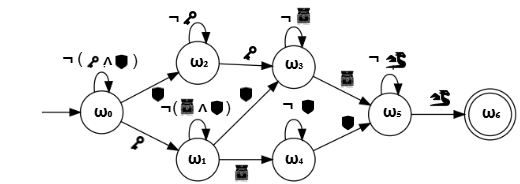}
    \caption{DFA for \emph{Dungeon Quest}. The sequence requires acquiring the Key, Shield, and Chest before facing the Dragon.}
    \label{fig:dfa_dungeon}
\end{figure}

\textbf{3) Blind Craftsman (6$\times$6).}
This grid world
features multiple paths and potential loops within its subgoal structure. The agent must gather wood, craft tools, and return home. The environment includes wood sources (up to two pieces collected at a time), the factory (for crafting tools), and the home (the final goal, accessible only after crafting all tools). The agent alternates between wood collection and factory visits to craft three tools before returning home. Task dependencies ensure tools require sufficient wood, and the mission is incomplete until all three tools are crafted. The DFA encodes these dependencies. 

\begin{figure}[!t]
    \centering
    \includegraphics[width=6in]{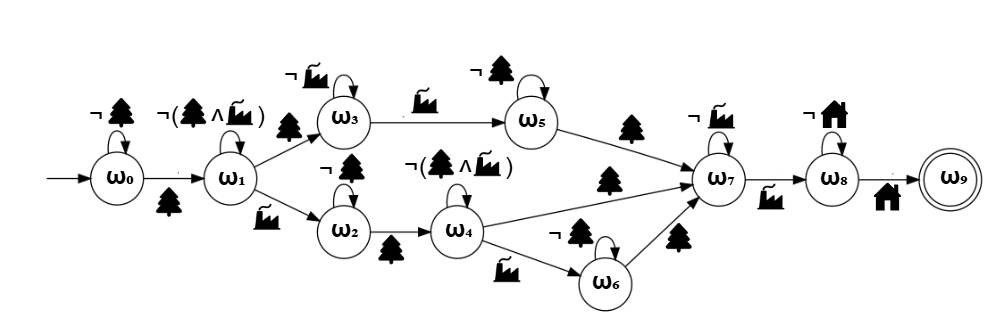}
    \caption{DFA for \emph{Blind Craftsman}. The agent must alternate between collecting wood and visiting the Factory to craft tools before returning Home.}
    \label{fig:dfa_craftsman}
\end{figure}

\textbf{4) Minecraft Building Bridge (8$\times$8).}
The agent must build a bridge by collecting \emph{Wood} and \emph{Iron} and using them at the \emph{Factory}, the final subgoal \cite{Andreas2017Modular,ToroIcarte2018UsingRL}. Starting at a \emph{Home Base}, the agent navigates the grid, overcoming obstacles like rivers, rocks, and trees. Unlike previous tasks, \emph{Wood} and \emph{Iron} can be collected in any order before reaching the \emph{Factory}, allowing multiple valid sequences. The DFA encodes these dependencies, ensuring the task is completed only after both materials are collected and the bridge is constructed. 

\begin{figure}[!t]
    \centering
    \includegraphics[width=3.8in]{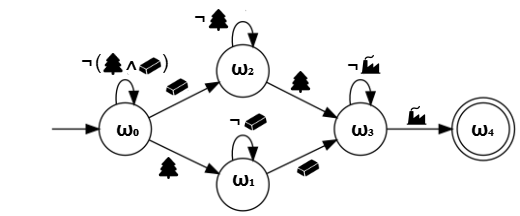}
    \caption{DFA for \emph{Minecraft Building Bridge}. The agent can collect Wood and Iron in any order before utilizing the Factory to complete the bridge.}
    \label{fig:dfa_bridge}
\end{figure}

\begin{figure}[!htbp]
\centering

\begin{minipage}{0.24\textwidth}
  \centering
  \includegraphics[width=\linewidth]{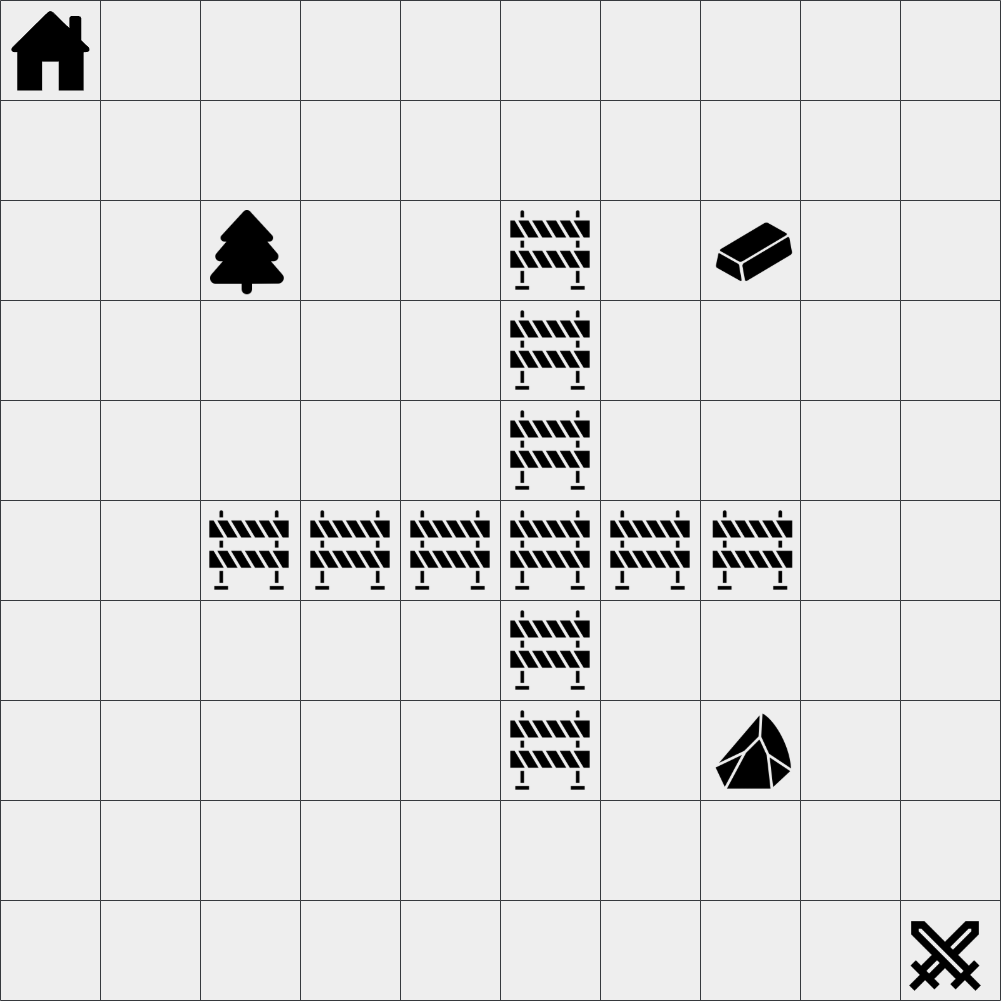}

  \label{fig:student_minecraft}
\end{minipage}%
\hfill
\begin{minipage}{0.24\textwidth}
  \centering
  \includegraphics[width=\linewidth]{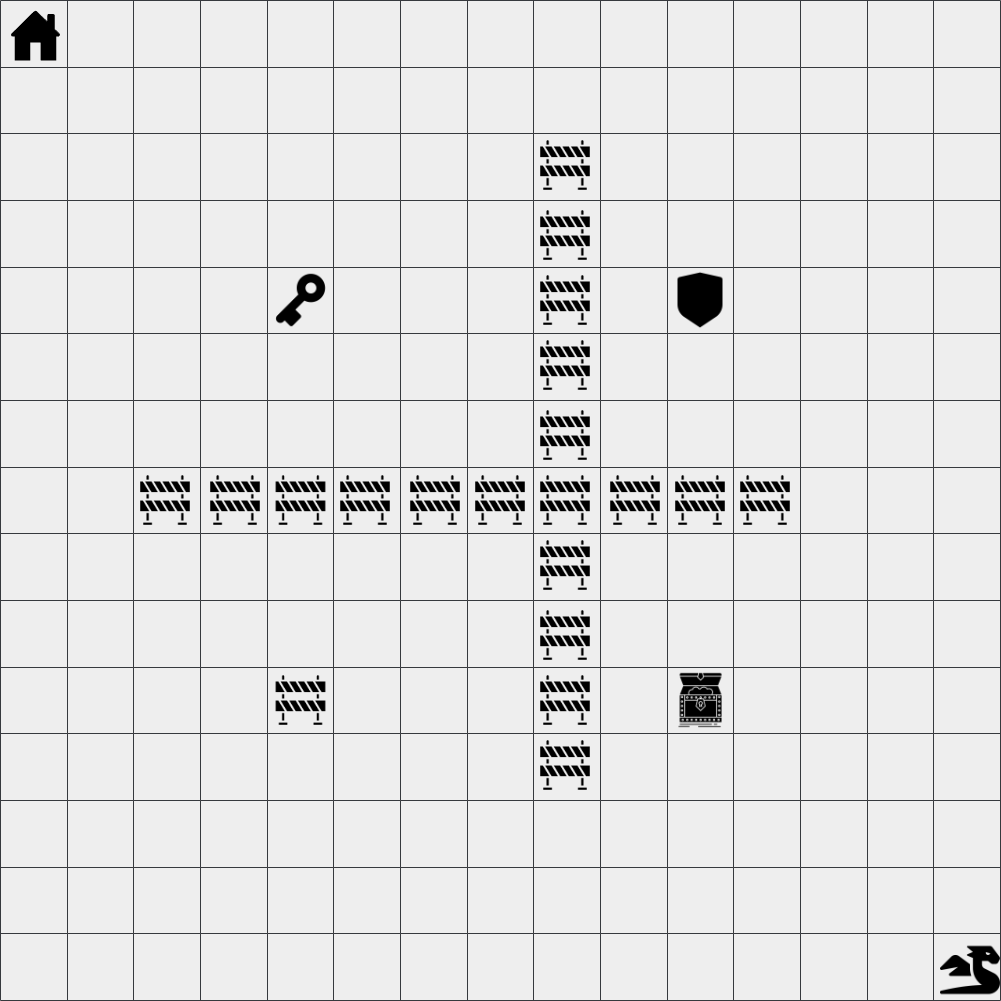}

  \label{fig:student_dungeon}
\end{minipage}%
\hfill
\begin{minipage}{0.24\textwidth}
  \centering
  \includegraphics[width=\linewidth]{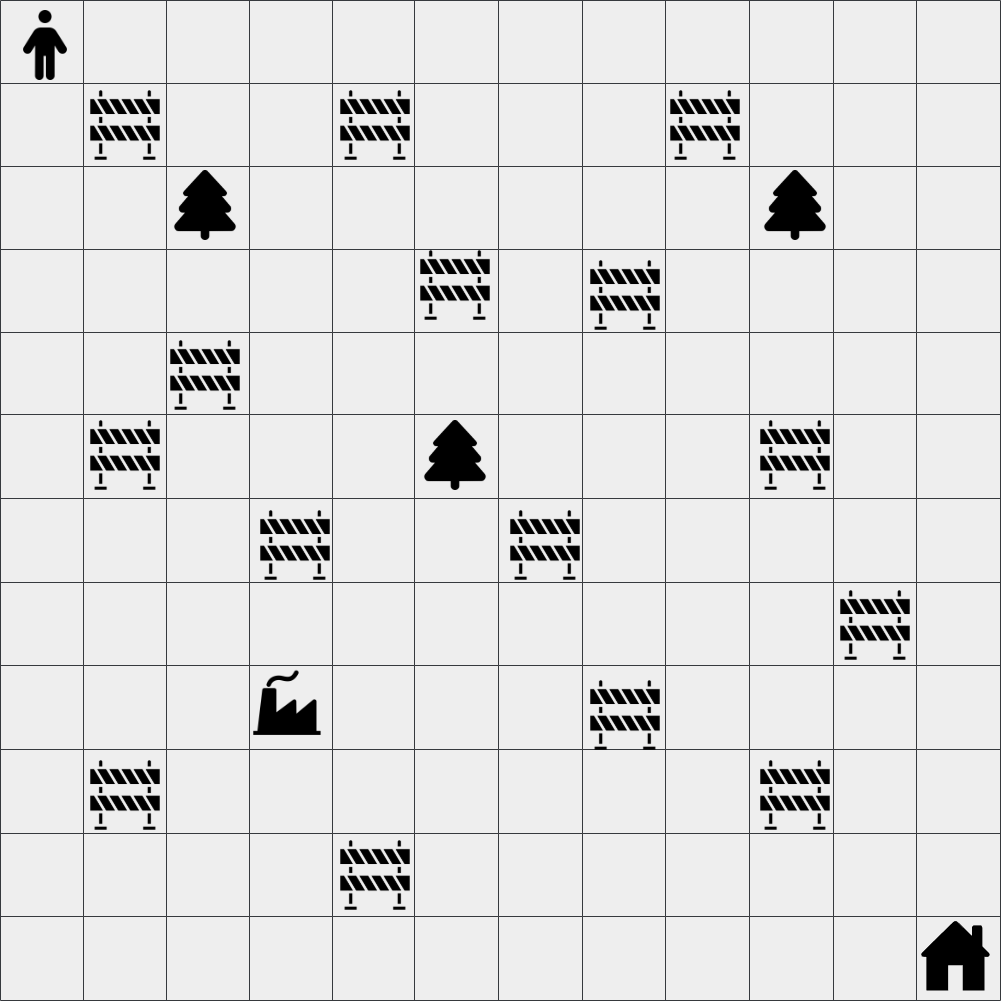}

  \label{fig:student_craftsman}
\end{minipage}%
\hfill
\begin{minipage}{0.24\textwidth}
  \centering
  \includegraphics[width=\linewidth]{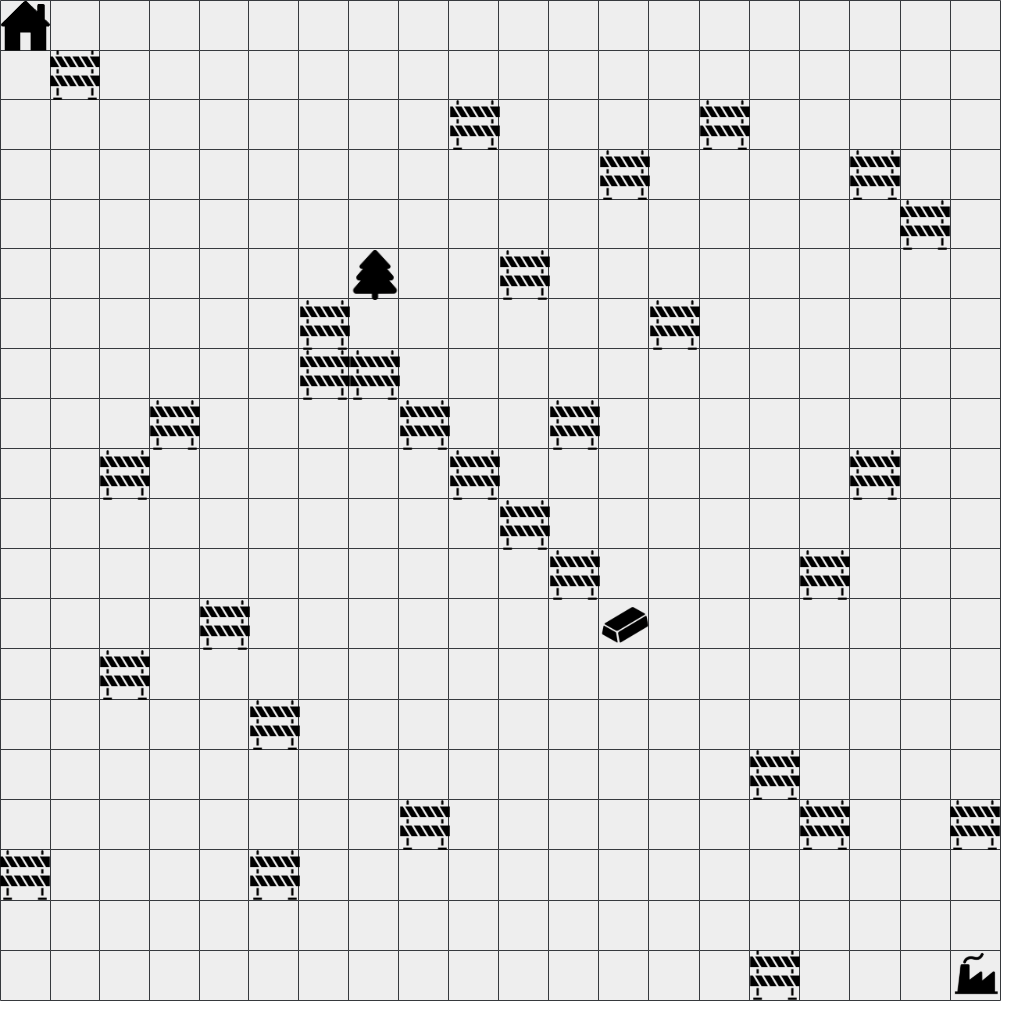}

  \label{fig:student_bridge}
\end{minipage}
\caption{\small{Visualization of the four student environments.}} 
\label{fig:student_environments}
\end{figure}

\textbf{5) Mountain Car Collection (Physics-Based Navigation).}
The agent operates in a 1D mountainous terrain with 20 discrete positions, where movement is constrained by energy levels and terrain difficulty. The state space is 9-dimensional, comprising normalized position (1D), energy level (5D one-hot encoding), and inventory status for three collectible items (3D binary). The environment features a challenging landscape with varying heights that affect energy consumption during navigation. The agent must sequentially collect four items: Power Cell, Sensor Array, Data Crystal, and Base Station, following strict temporal dependencies encoded in the DFA. The agent has five energy states (depleted, low, medium, high, max) that determine movement capabilities, with energy consumption based on terrain elevation changes and obstacle encounters. Obstacles at positions 5, 10, and 15 impose additional energy penalties. The agent's movement distance is limited by its current energy level, and the rest action provides energy recovery. This environment validates our approach's effectiveness in physics-based domains with resource management constraints, extending beyond discrete navigation tasks to multi-dimensional state spaces with continuous constraint satisfaction problems.

\begin{figure}[!t]
    \centering
    \includegraphics[width=4.5in]{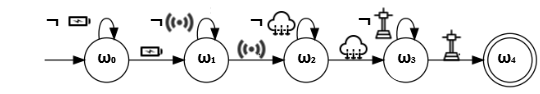}
    \caption{DFA for \emph{Mountain Car Collection}. The automaton enforces sequential collection of power cell, sensor array, data crystal, and base station return in strict order.}
    \label{fig:dfa_mountain_car}
\end{figure}

\textbf{6) Warehouse Robotics (High-Dimensional Continuous).}
This environment implements a realistic 12-dimensional continuous state space representing a warehouse robotics scenario, comprising normalized robot coordinates (2D), scanner possession status (1D), normalized scanner battery level (1D), individual task completion flags for zone scanning, item scanning, item pickup, and delivery (4D), normalized overall task progress (1D), normalized step count (1D), and binary proximity indicators for scan zone and shipping dock (2D). The robot operates in a 6×8 grid with continuous position coordinates, requiring precise navigation and state management. The task involves a 5-step sequential operation: obtaining a scanner from the station, navigating to the scan zone to perform inventory scanning, returning the scanner to the charging station for battery replenishment, collecting the identified item from its location, and delivering it to the shipping dock. Each subtask must be completed in the correct order as tracked by the automaton state progression. This high-dimensional representation mirrors real-world robotic systems where agents must simultaneously track position, equipment status, battery levels, individual task completion states, and environmental proximity, creating a challenging continuous control problem. This environment highlights the scalability of our framework and its capability to handle realistic industrial automation scenarios with complex temporal constraints and multi-modal state representations typical of modern warehouse management systems.

\begin{figure}[!t]
    \centering
    \includegraphics[width=4.5in]{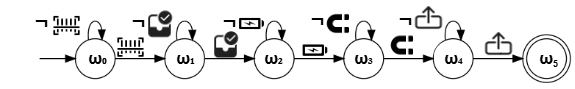}
    \caption{DFA for \emph{Warehouse Robotics}. The 6-state automaton encodes the complete warehouse workflow: scanner acquisition, zone scanning, scanner return, item pickup, and delivery completion.}
    \label{fig:dfa_warehouse}
\end{figure}

\textbf{Results.}
First, we report results using the subtask-based scoring in \eqref{eq:score_subtask_based}, evaluated across the four environments on three metrics: 1) reward per episode, 2) steps per episode, and 3) reward per cumulative steps, which best reflects learning efficiency. We focus on the third metric in the main text; the others appear in Appendix \ref{appdx:exp}.

Figure~\ref{fig:cumulative_steps_all} shows the \textbf{reward per cumulative steps}, which highlights overall learning efficiency. In two environments, the dynamic method achieves superior efficiency after sufficient training, balancing exploration with preference refinement. In the other two environments, both the static and dynamic methods maintain competitive performance, often outperforming the other baselines.

\begin{figure}[!htbp]
\centering
\begin{minipage}{0.49\columnwidth}
  \centering
  \includegraphics[width=0.99\columnwidth]{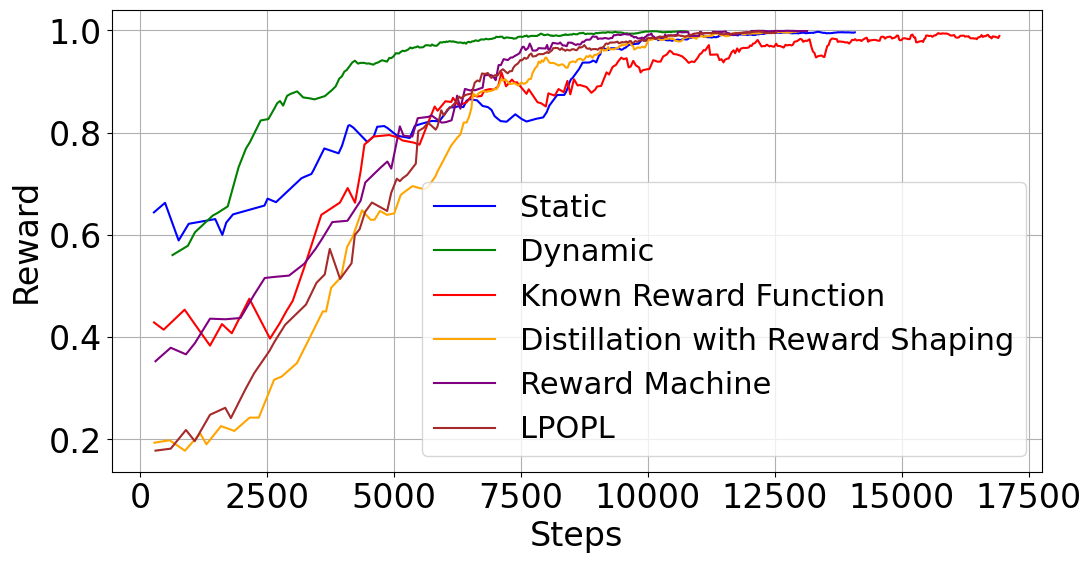}

  \label{fig:cumulative_minecraft}
\end{minipage}%
\hfill
\begin{minipage}{0.49\columnwidth}
  \centering
  \includegraphics[width=0.99\columnwidth]{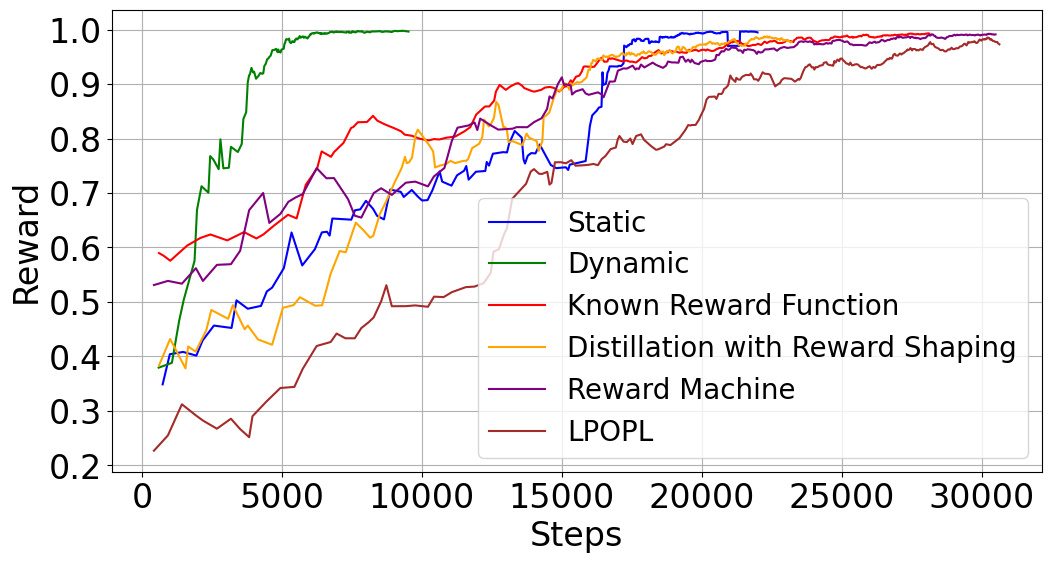}

  \label{fig:cumulative_dungeon}
\end{minipage}

\begin{minipage}{0.49\columnwidth}
  \centering
  \includegraphics[width=0.99\columnwidth]{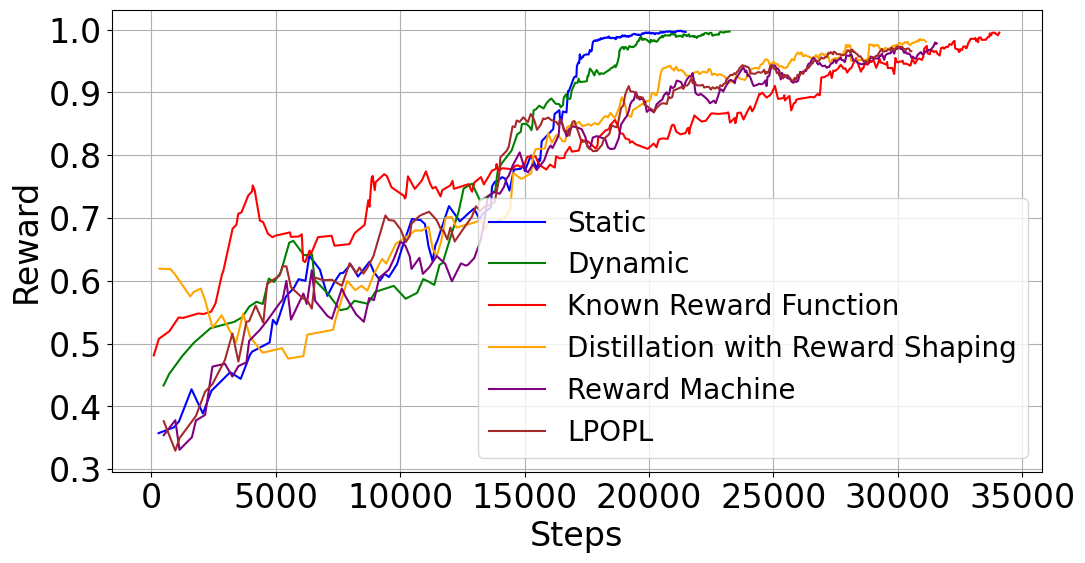}

  \label{fig:cumulative_craftsman}
\end{minipage}%
\hfill
\begin{minipage}{0.49\columnwidth}
  \centering
  \includegraphics[width=0.99\columnwidth]{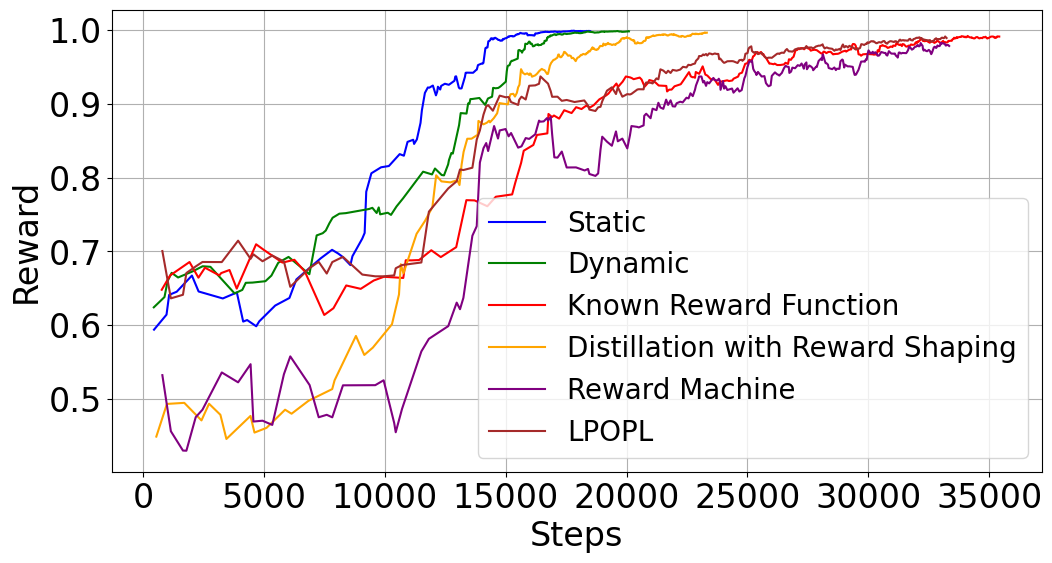}

  \label{fig:cumulative_bridge}
\end{minipage}
\caption{\small{Reward per cumulative steps for all environments.  Minecraft Iron Sword Quest (top left), Dungeon Quest (top right), Blind Craftsman (bottom left), and Minecraft Building Bridge (bottom right).}}

\label{fig:cumulative_steps_all}
\end{figure}

The dynamic method excels through iterative reward refinement, while the static method performs well with a fixed reward. Both surpass the known reward, reward-shaping, RM, and LPOPL baselines, underscoring the advantage of adaptive, automaton-based alignment in subgoal-driven RL tasks.

Figure~\ref{fig:extended_domains_results} shows results for the high-dimensional environments, Mountain Car Collection (left) and Warehouse Robotics, demonstrating effectiveness across diverse domain types. Both static and dynamic methods outperform all baselines, demonstrating scalability to realistic applications beyond discrete gridworlds.

\textbf{Continuous environment results.} To assess generalizability, we tested our approach in continuous versions of the four environments, where agents navigate 2D planes with real-valued coordinates. We replaced Manhattan distance with Euclidean distance and used TD3 for policy optimization.

Figure~\ref{fig:continuous_reward_step_all} shows reward per cumulative steps, demonstrating that our approach generalizes well to continuous state spaces. The dynamic method outperforms others as training progresses, while both preference-based methods exceed the known reward and DFA shaping baselines. This highlights the value of preference-based learning in continuous domains, where reward engineering is harder and nuanced feedback is essential. The minimal changes needed (updating distance metrics and policy optimization methods) underscore the flexibility and robustness of our automaton-based preference framework, supporting its scalability to more realistic settings.

\begin{figure}[!htbp]
\centering
\begin{minipage}{0.49\columnwidth}
  \centering
  \includegraphics[width=0.99\columnwidth]{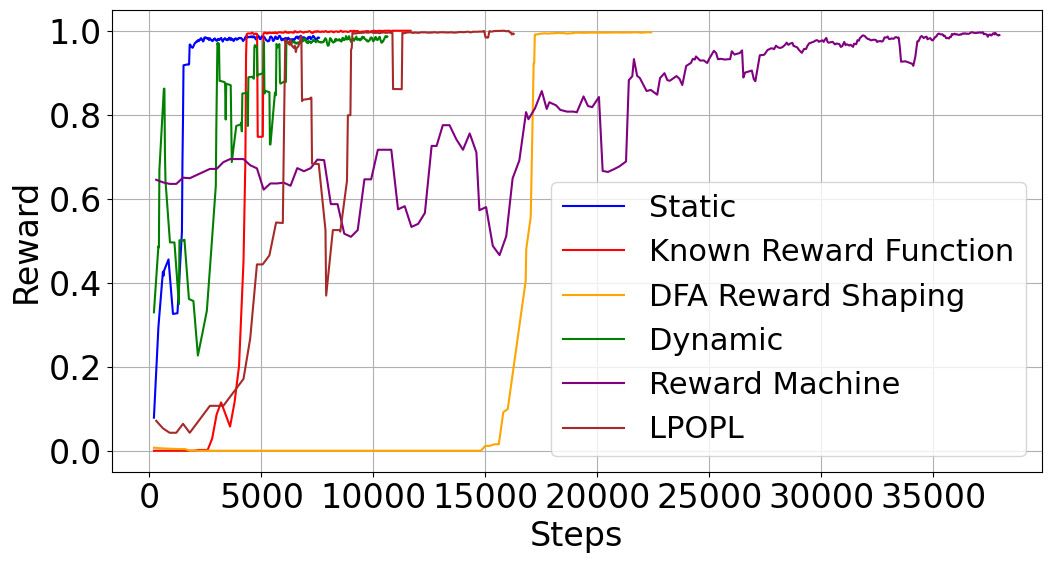}
  \label{fig:continuous_cumulative_minecraft}
\end{minipage}%
\hfill
\begin{minipage}{0.49\columnwidth}
  \centering
  \includegraphics[width=0.99\columnwidth]{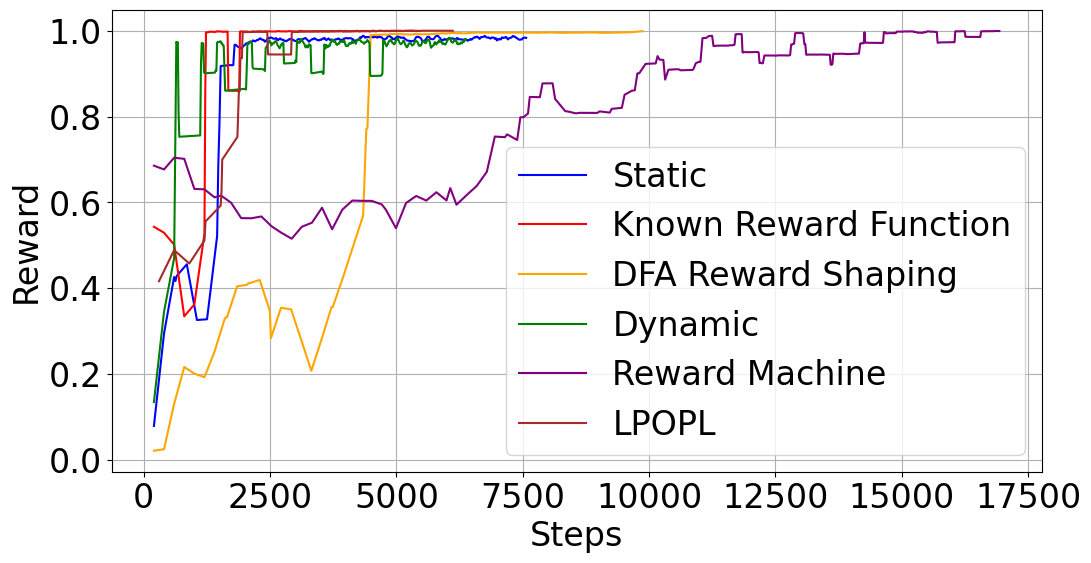}
  \label{fig:continuous_cumulative_dungeon}
\end{minipage}

\begin{minipage}{0.49\columnwidth} 
  \centering
  \includegraphics[width=0.99\columnwidth]{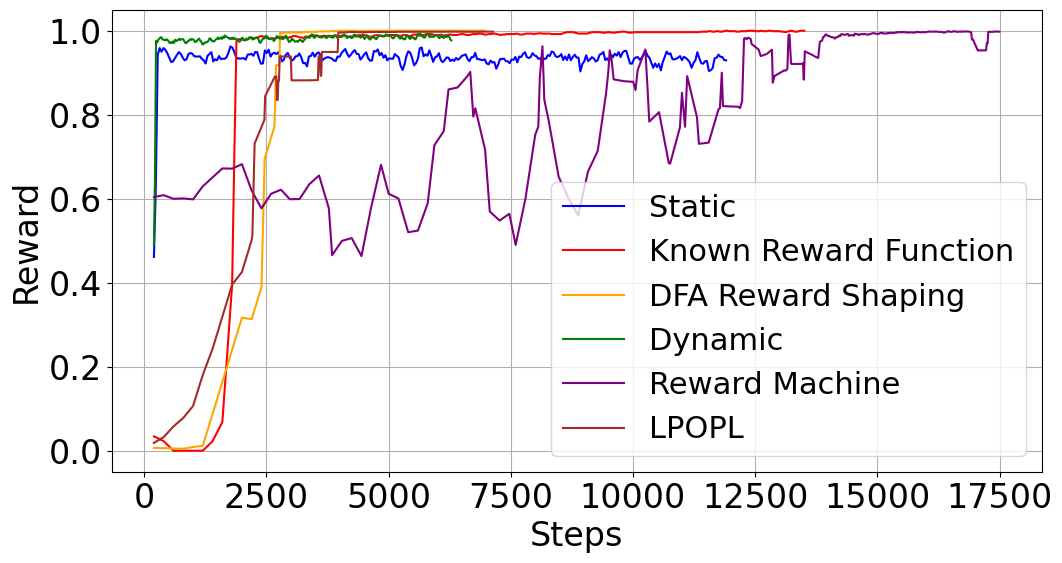}
  \label{fig:continuous_cumulative_craftsman}
\end{minipage}%
\hfill
\begin{minipage}{0.49\columnwidth}
  \centering
  \includegraphics[width=0.99\columnwidth]{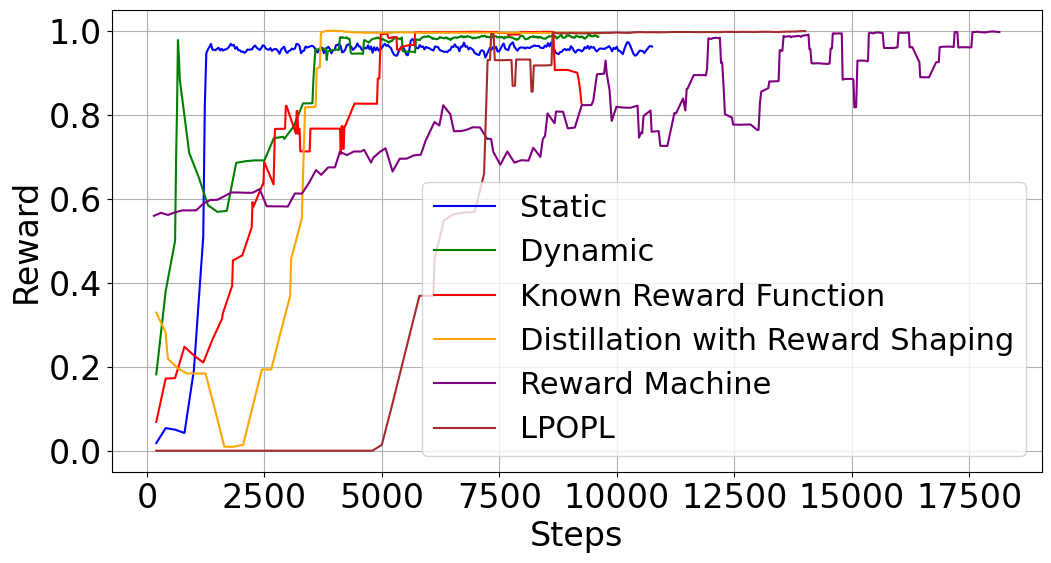}
  \label{fig:continuous_cumulative_bridge}
\end{minipage}
\caption{\small{Reward per cumulative steps for all environments in continuous state spaces.}}
\label{fig:continuous_reward_step_all}
\end{figure}


\begin{figure}[!htbp]
\centering
\begin{minipage}{0.49\columnwidth}
  \centering
  \includegraphics[width=0.99\columnwidth]{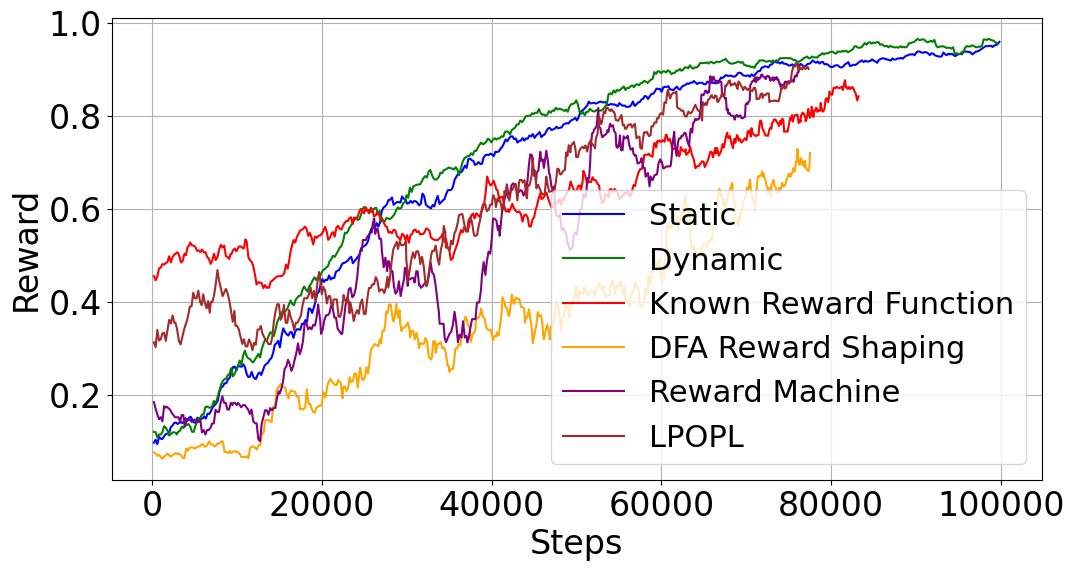}
  \label{fig:extended_mountain_car}
\end{minipage}%
\hfill
\begin{minipage}{0.49\columnwidth}
  \centering
  \includegraphics[width=0.99\columnwidth]{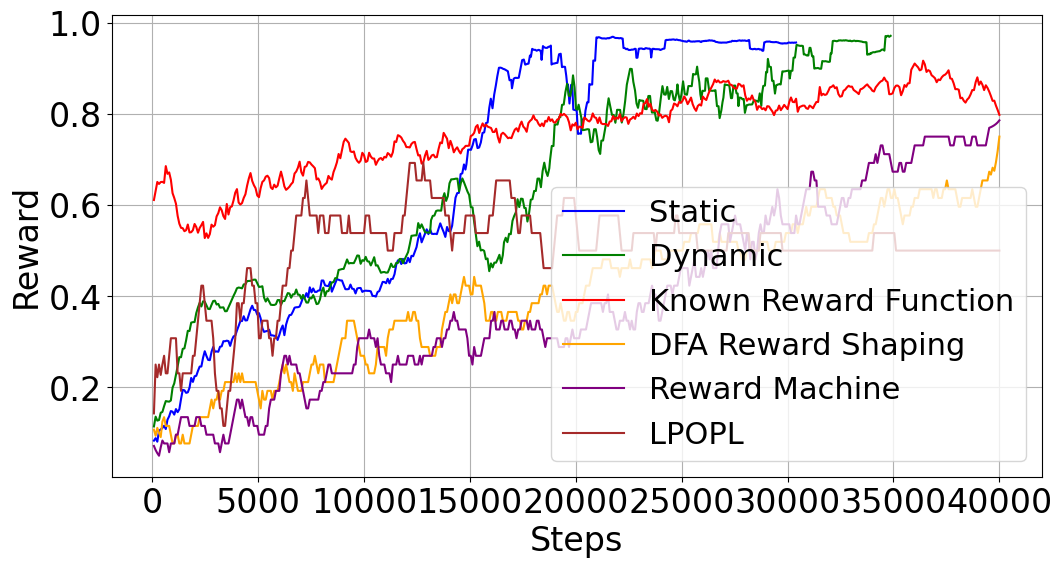}
  \label{fig:extended_warehouse}
\end{minipage}

\caption{\small{Reward per cumulative steps for Mountain Car Collection (left) and Warehouse Robotics 12D environment (right).}}
\label{fig:extended_domains_results}
\end{figure}

\textbf{Combined scoring with transfer.}
We evaluate the effectiveness of combining subtask-based preferences with transition-value-based knowledge transferred from a teacher agent using the scoring function \eqref{eq:score_transfer_learning}. 
By integrating the teacher’s automaton values into preference scoring, we enhance a student agent's learning in a more complex environment,
enabling the student agent to perform well in challenging scenarios where environment-provided rewards are unavailable. 
The DFA is augmented with values distilled from a teacher agent trained in the original environments (see \eqref{eq:aut_qvals}). We then leverage these learned DFA Q-values to student variants of the four environments, scaled up to larger grids (10$\times$10, 15$\times$15, 12$\times$12, 20$\times$20) and featuring additional obstacles or modified subgoal placements as shown in Fig.~\ref{fig:student_environments}. 
Although they retain the same subgoal structure, these expanded layouts pose a greater navigational challenge and demand more complex policy adaptation.

To evaluate our approach, we compare the performance of several student agents. In \textbf{Static (Pref)} and \textbf{Dynamic (Pref)}, the student leverages distilled teacher preferences under the static and dynamic methods, respectively. We also introduce \textbf{Static (Pref+Plan)} and \textbf{Dynamic (Pref+Plan)}, where the DFA’s Q-values guide both preference scoring and planning. Specifically, the student’s Q-learning update rule is modified by an annealing mechanism:
\begin{align}
Q'_{\text{student}}((s,q), a) 
= \beta(q, L(s')) \,\overline{Q}_{\text{teacher}}(q, L(s'))
+ \bigl(1 - \beta(q, L(s'))\bigr)\,Q_{\text{target}},
\end{align}
where $Q'_{\text{student}}$ is the adjusted Q-value in the product MDP. The weight, 
$
\beta(q, L(s')) = \rho^{\,n_{\text{student}}(q, L(s'))}
$,
decays over time (with $\rho \in (0,1)$), controlling the degree to which the student relies on the teacher’s knowledge, where $n_\text{student}(q,\sigma)$ is the frequency of automaton transition $(q,\sigma)$ 
in the student environment. The standard target Q-value 
is
$
Q_{\text{target}} 
= r + \gamma \,\max_{a'}\,Q_{\text{student}}((s',q'), a')$. By relying more on the teacher’s Q-values early in training and gradually shifting to its own experience, the student smoothly transitions from teacher-guided to autonomous learning. Here, $r$ refers to the output of the learned reward model (trained on DFA-derived preferences) rather than environment rewards. See Appendix \ref{appdx:knowledge_transfer} for more details.

\begin{figure}[!htbp]
\centering

\begin{minipage}{0.49\columnwidth}
  \centering
  \includegraphics[width=0.99\columnwidth]{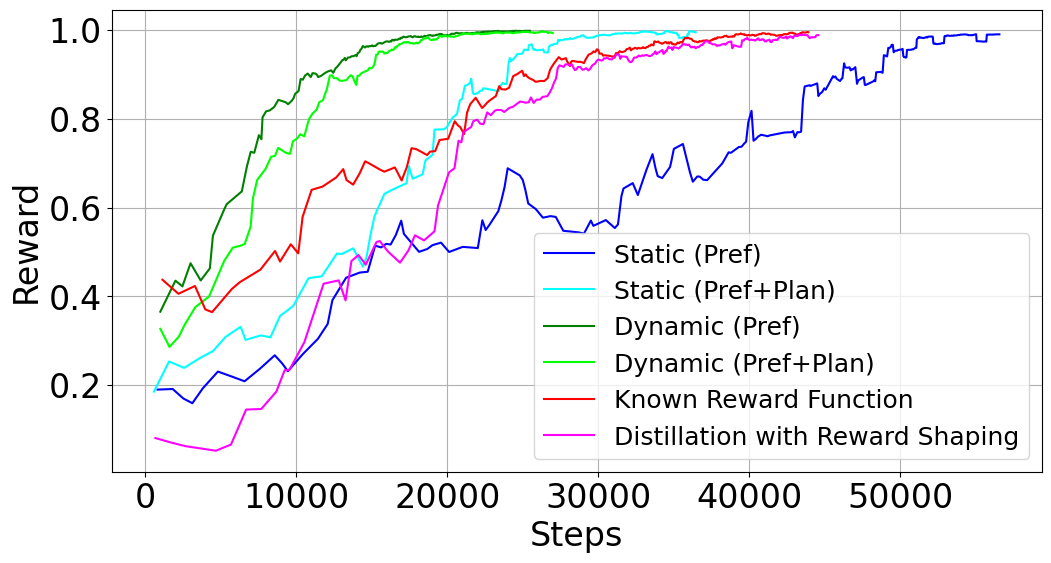}

  \label{fig:adjusted_reward_per_steps_env1}
\end{minipage}%
\hfill
\begin{minipage}{0.49\columnwidth}
  \centering
  \includegraphics[width=0.99\columnwidth]{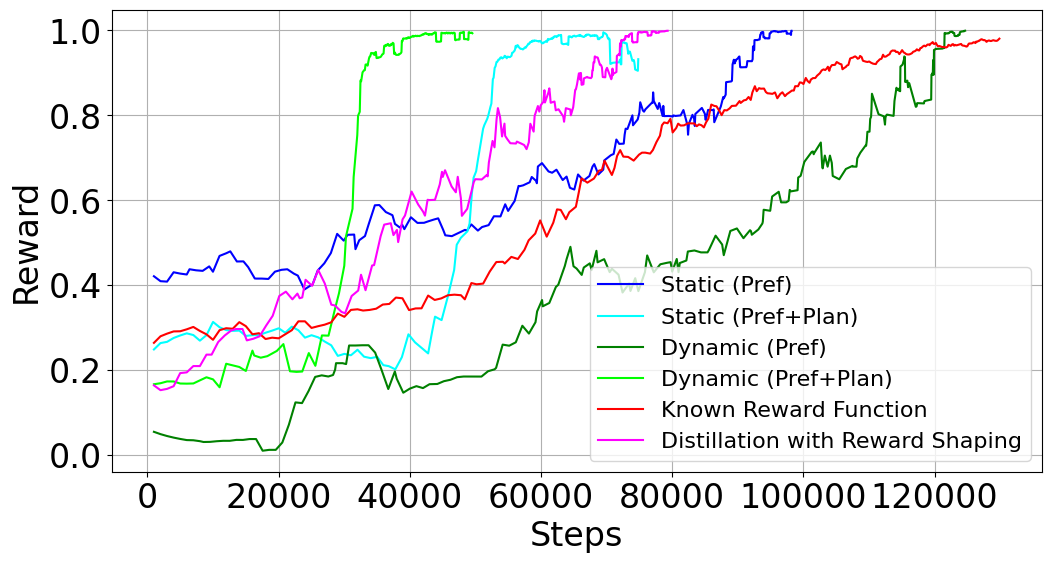}

  \label{fig:adjusted_reward_per_steps_env2}
\end{minipage}

\begin{minipage}{0.49\columnwidth}
  \centering
  \includegraphics[width=0.99\columnwidth]{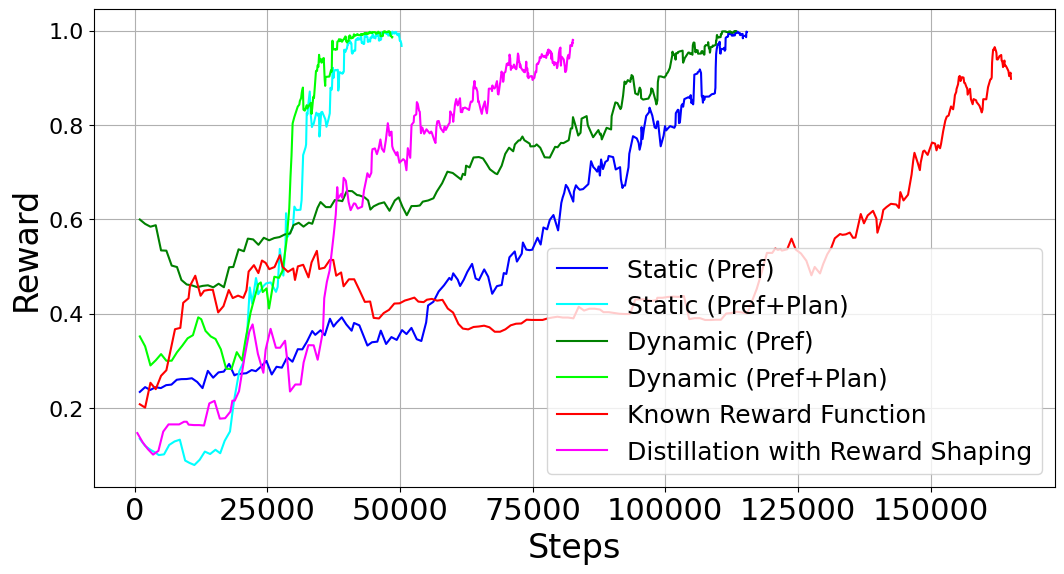}

  \label{fig:adjusted_reward_per_steps_env3}
\end{minipage}%
\hfill
\begin{minipage}{0.49\columnwidth}
  \centering
  \includegraphics[width=0.99\columnwidth]{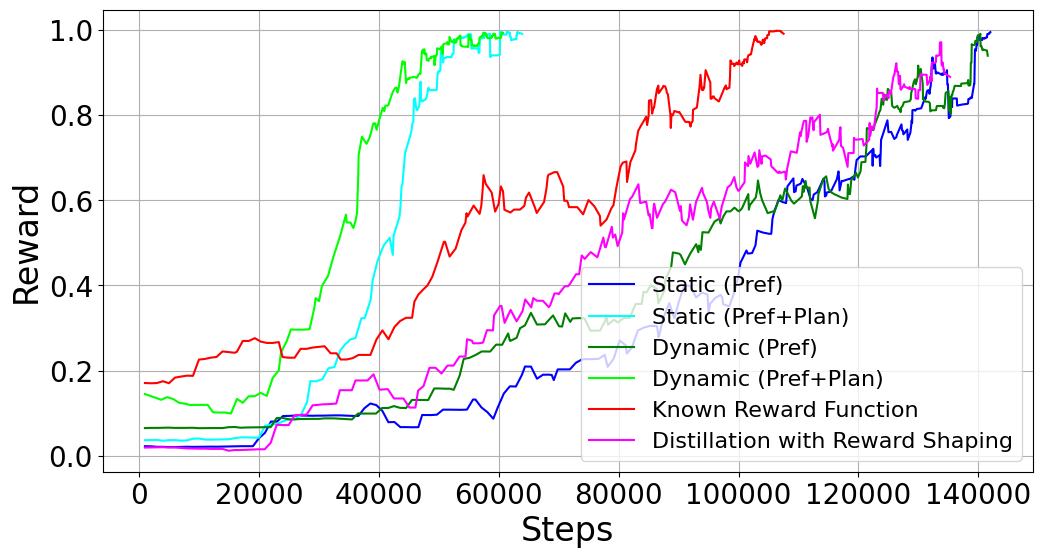}

  \label{fig:adjusted_reward_per_steps_env4}
\end{minipage}
\caption{\small{Reward per step for different methods and environments.}}
\label{fig:adjusted_reward_per_steps_all}

\end{figure}

Figure~\ref{fig:adjusted_reward_per_steps_all} shows that distillation agents consistently outperform non-distilled ones in reward per cumulative steps, highlighting the efficiency gained from teacher-derived automaton knowledge. The distillation process transfers subgoal and automaton transition expertise to the student, guiding it via distilled Q-values in both preference generation and policy optimization. This leads to faster convergence by aligning trajectories with the task's temporal structure.

\textbf{Validation of Automaton Transition Scoring.}
\label{validation}
To validate that our automaton transition scoring produces meaningful preference orderings, we conducted two complementary empirical analyses in the Dungeon Quest environment.

\textit{Trajectory Quality Correlation Analysis:}
We first collected 200 trajectories with varying degrees of task completion and computed (i) the automaton transition score from Equation~\eqref{eq:score_transfer_learning}, (ii) the actual cumulative reward received, and (iii) the task completion level (number of subgoals achieved).

The automaton transition scores exhibit strong positive correlation with both cumulative reward (Pearson $r=0.87$, $p<0.001$) and task completion (Spearman $\rho=0.92$, $p<0.001$). Notably, trajectories that complete more subgoals consistently receive higher scores, and among trajectories with equal subgoal completion, those with more efficient paths (higher cumulative reward) score higher.

\textit{Endpoint Future Performance Analysis:}
To more rigorously test whether our scoring captures meaningful task progress, we conducted an endpoint future performance analysis. The key question is: \emph{Do trajectories with higher automaton transition scores leave the agent in states from which it is easier to reach acceptance (task completion)?}

For 200 partial trajectories of varying quality, we:
\begin{enumerate}
    \item Computed the automaton transition score $\sum_{t=0}^{T-1} Q_{\text{dfa}}(q_t, \sigma_t)$ for each trajectory
    \item From each trajectory's endpoint (final state and DFA state), ran 50 greedy policy rollouts
    \item Measured: (a) future success rate (fraction reaching acceptance), (b) average future cumulative reward
\end{enumerate}

The results strongly validate our scoring
\begin{itemize}
    \item Automaton scores strongly predict future success rate ($r=0.84$, $p<0.001$)
    \item Automaton scores positively correlate with future cumulative reward ($r=0.85$, $p<0.001$)
\end{itemize}

Comparing high-scoring endpoints (above median score) to low-scoring endpoints reveals substantial performance differences. High-scoring trajectories achieve a 29\% higher success rate (0.94 vs 0.73, Mann-Whitney $U$-test $p<0.001$).

These results demonstrate that $\sum_t Q_{\text{dfa}}(q_t, \sigma_t)$ 
successfully captures meaningful task progress and induces preference orderings that align with both immediate task objectives (reward and completion) and future task completion potential.

\section{Discussion and Limitations}

While our automaton-based preference framework performs well across tasks, there are some limitations. Constructing DFAs for large-scale environments can be difficult, motivating future work on learning or refining automata from demonstrations or natural language. Our preference generation uses simple heuristics like subgoal completion and distance, which may not suffice in complex domains with non-Euclidean spaces or interleaved goals; richer models offer a natural extension. Although we generalize to continuous and high-dimensional domains, scaling to very high-dimensional spaces poses challenges for both DFA representation and preference learning. Appendix~\ref{appdx:hi_dim} outlines theoretical tools (including automaton-guided attention and transfer learning) with formal guarantees, though empirical validation is pending. Our current framework also assumes full observability; extending to partially observable settings will require belief tracking or state estimation as discussed in Appendix~\ref{appdx:pomdp}.

\section{Conclusion}

We introduced an RL framework that achieves automaton-based alignment, using DFAs for preference rankings to train predictive reward models. We explored static and dynamic variants, benchmarking them against policies using a designed reward function, distillation with reward shaping, reward machines, and LTL-based approaches. Our method demonstrates superior performance across both discrete and continuous environments, highlighting the advantage of preference-based learning over direct reward specification. Our approach captures task-specific temporal dependencies, translating complex sequential objectives into structured, preference-based rewards. The results highlight the advantage of automata-derived reward shaping, particularly in tasks with multiple subgoals and evolving policy requirements. The success of both dynamic and static methods over the baselines underscores the value of integrating structured preferences for more adaptive and efficient RL.

\section{Acknowledgments}
This work was supported by DARPA under Agreement No. HR0011-24-9-0427 and NSF under Award CCF-2106339.

\printbibliography

@article{Mnih2015HumanLevel,
  title={Human-level control through deep reinforcement learning},
  author={Mnih, Volodymyr and Kavukcuoglu, Koray and Silver, David and Rusu, Andrei A. and Veness, Joel and Bellemare, Marc G. and Graves, Alex and Riedmiller, Martin and Fidjeland, Andreas K. and Ostrovski, Georg and Petersen, Stig and Beattie, Charles and Sadik, Amir and Antonoglou, Ioannis and King, Helen and Kumaran, Dharshan and Wierstra, Daan and Legg, Shane and Hassabis, Demis},
  journal={Nature},
  volume={518},
  number={7540},
  pages={529--533},
  year={2015},
  publisher={Nature Publishing Group}
}

@article{Silver2016Go,
  title={Mastering the game of Go with deep neural networks and tree search},
  author={Silver, David and Huang, Aja and Maddison, Chris J. and Guez, Arthur and Sifre, Laurent and Van Den Driessche, George and Schrittwieser, Julian and Antonoglou, Ioannis and Panneershelvam, Veda and Lanctot, Marc and Dieleman, Sander and Grewe, Dominik and Nham, John and Kalchbrenner, Nal and Sutskever, Ilya and Lillicrap, Timothy and Leach, Madeleine and Kavukcuoglu, Koray and Graepel, Thore and Hassabis, Demis},
  journal={Nature},
  volume={529},
  number={7587},
  pages={484--489},
  year={2016},
  publisher={Nature Publishing Group}
}

@article{Vinyals2019AlphaStar,
  title={Grandmaster level in StarCraft II using multi-agent reinforcement learning},
  author={Vinyals, Oriol and Babuschkin, Igor and Czarnecki, Wojciech M. and Mathieu, Micha{\"e}l and Dudzik, Andrew and Chung, Junyoung and Choi, David H. and Powell, Richard and Ewalds, Timo and Georgiev, Petko and Oh, Junhyuk and Horgan, Dan and Kroiss, Manuel and Danihelka, Ivo and Huang, Aja and Sifre, Laurent and Cai, Trevor and Agapiou, John P. and Jaderberg, Max and Vezhnevets, Alexander S. and Leblond, R{\'e}mi and Pohlen, Tobias and Dalibard, Valentin and Budden, David and Sulsky, Yury and Molloy, James and Paine, Tom L. and Gulcehre, Caglar and Wang, Ziyu and Pfaff, Tobias and Wu, Yuhuai and Ring, Roman and Yogatama, Dani and W{\"u}nsch, Dario and McKinney, Katrina and Smith, Oliver and Schaul, Tom and Lillicrap, Timothy and Kavukcuoglu, Koray and Hassabis, Demis and Apps, Chris and Silver, David},
  journal={Nature},
  volume={575},
  number={7782},
  pages={350--354},
  year={2019},
  publisher={Nature Publishing Group}
}

@article{Levine2016EndToEnd,
  title={End-to-end training of deep visuomotor policies},
  author={Levine, Sergey and Finn, Chelsea and Darrell, Trevor and Abbeel, Pieter},
  journal={Journal of Machine Learning Research},
  volume={17},
  number={39},
  pages={1--40},
  year={2016}
}

@inproceedings{Lillicrap2016Continuous,
  title={Continuous control with deep reinforcement learning},
  author={Lillicrap, Timothy P. and Hunt, Jonathan J. and Pritzel, Alexander and Heess, Nicolas and Erez, Tom and Tassa, Yuval and Silver, David and Wierstra, Daan},
  booktitle={Proceedings of the International Conference on Learning Representations},
  year={2015}
}

@inproceedings{Kendall2019Learning,
  title={Learning to drive in a day},
  author={Kendall, Alex and Hawke, Jeffrey and Janz, David and Mazur, Przemyslaw and Reda, Daniele and Allen, John-Mark and Lam, Vinh-Dieu and Bewley, Alex and Shah, Amar},
  booktitle={Proceedings of the IEEE International Conference on Robotics and Automation},
  pages={8248--8254},
  year={2019}
}

@book{SuttonBarto,
  title={Reinforcement Learning: An Introduction},
  author={Sutton, Richard S. and Barto, Andrew G.},
  edition={2nd},
  publisher={MIT Press},
  address={Cambridge, MA, USA},
  year={2018}
}

@article{Littman2017Environment,
  title={Environment-independent task specifications via GLTL},
  author={Littman, Michael L. and Topcu, Ufuk and Fu, Jie and Isbell, Charles and Wen, Min and MacGlashan, James},
  journal={arXiv preprint arXiv:1704.04341},
  year={2017}
}

@inproceedings{Camacho2019LTLMOP,
  title={Learning interpretable models expressed in linear temporal logic},
  author={Camacho, Alberto and McIlraith, Sheila A.},
  booktitle={Proceedings of the International Conference on Automated Planning and Scheduling},
  volume={29},
  pages={621--630},
  year={2019}
}

@inproceedings{Bacchus1996LearningML,
  title={Graphical models for preference and utility},
  author={Bacchus, Fahiem and Grove, Adam},
  booktitle={Proceedings of the Eleventh Conference on Uncertainty in Artificial Intelligence},
  pages={3--10},
  year={1995}
}

@article{Icarte2022reward,
  title={Reward machines: Exploiting reward function structure in reinforcement learning},
  author={Toro Icarte, Rodrigo and Klassen, Toryn Q. and Valenzano, Richard and McIlraith, Sheila A.},
  journal={Journal of Artificial Intelligence Research},
  volume={73},
  pages={173--208},
  year={2022}
}

@inproceedings{Li2017TLTL,
author = {Li, Xiao and Vasile, Cristian-Ioan and Belta, Calin},
title = {Reinforcement learning with temporal logic rewards},
year = {2017},
publisher = {IEEE Press},
booktitle = {2017 IEEE/RSJ International Conference on Intelligent Robots and Systems (IROS)},
pages = {3834–3839},
location = {Vancouver, BC, Canada}
}

@inproceedings{ToroIcarte2018UsingRL,
  title={Using reward machines for high-level task specification and decomposition in reinforcement learning},
  author={Toro Icarte, Rodrigo and Klassen, Toryn Q. and Valenzano, Richard and McIlraith, Sheila A.},
  booktitle={Proceedings of the 35th International Conference on Machine Learning},
  pages={2107--2116},
  year={2018}
}

@inproceedings{Oncina1992,
  title={Inferring regular languages in polynomial update time},
  author={Oncina, Jose and Garc{\'\i}a, Pedro},
  booktitle={Pattern Recognition and Image Analysis: Selected Papers from the IV Spanish Symposium},
  pages={49--61},
  year={1992}
}

@article{Angluin1987Learning,
  title={Learning regular sets from queries and counterexamples},
  author={Angluin, Dana},
  journal={Information and Computation},
  volume={75},
  number={2},
  pages={87--106},
  year={1987}
}

@article{Li2017Reinforcement,
  title={Learning to decode for future success},
  author={Li, Jiwei and Monroe, Will and Jurafsky, Dan},
  journal={arXiv preprint arXiv:1701.06549},
  year={2017}
}

@inproceedings{Christiano2017Deep,
  title={Deep reinforcement learning from human preferences},
  author={Christiano, Paul F. and Leike, Jan and Brown, Tom and Martic, Miljan and Legg, Shane and Amodei, Dario},
  booktitle={Advances in Neural Information Processing Systems},
  volume={30},
  pages={4299--4307},
  year={2017}
}

@inproceedings{Andrychowicz2017Hindsight,
  title={Hindsight experience replay},
  author={Andrychowicz, Marcin and Wolski, Filip and Ray, Alex and Schneider, Jonas and Fong, Rachel and Welinder, Peter and McGrew, Bob and Tobin, Josh and Abbeel, Pieter and Zaremba, Wojciech},
  booktitle={Advances in Neural Information Processing Systems},
  pages={5048--5058},
  year={2017}
}

@inproceedings{Kaelbling1993Learning,
  title={Learning to achieve goals},
  author={Kaelbling, Leslie Pack},
  booktitle={Proceedings of the 13th International Joint Conference on Artificial Intelligence},
  volume={2},
  pages={1094--1099},
  year={1993}
}

@article{Hart2009Learning,
  title={Learning generalizable control programs},
  author={Hart, Stephen and Grupen, Roderic},
  journal={IEEE Transactions on Autonomous Mental Development},
  volume={1},
  number={1},
  pages={1--16},
  year={2011}
}

@inproceedings{Ng1999PolicyInvariant,
  title={Policy invariance under reward transformations: Theory and application to reward shaping},
  author={Ng, Andrew Y. and Russell, Stuart},
  booktitle={Proceedings of the 16th International Conference on Machine Learning},
  pages={278--287},
  year={1999}
}

@article{Taylor2009TransferLearning,
  title={Transfer learning for reinforcement learning domains: A survey},
  author={Taylor, Matthew E. and Stone, Peter},
  journal={Journal of Machine Learning Research},
  volume={10},
  pages={1633--1685},
  year={2009}
}

@inproceedings{Rusu2015PolicyDistillation,
  title={Policy distillation},
  author={Rusu, Andrei A. and Gomez Colmenarejo, Sergio and Gulcehre, Caglar and Desjardins, Guillaume and Kirkpatrick, James and Pascanu, Razvan and Mnih, Volodymyr and Kavukcuoglu, Koray and Hadsell, Raia},
  booktitle={Proceedings of the 4th International Conference on Learning Representations},
  year={2015},
  note={arXiv preprint arXiv:1511.06295}
}

@article{Barto2003HRL,
  title={Recent advances in hierarchical reinforcement learning},
  author={Barto, Andrew G. and Mahadevan, Sridhar},
  journal={Discrete Event Dynamic Systems},
  volume={13},
  number={4},
  pages={341--379},
  year={2003}
}

@article{Chen2023NL2TL,
  title={NL2TL: Transforming natural languages to temporal logics using large language models},
  author={Chen, Zhe and Yang, Qidong and Li, Kun and Xu, Zheng},
  journal={arXiv preprint arXiv:2305.07766},
  year={2023}
}

@inproceedings{Neider2021AdviceGuided,
  title={Advice-guided reinforcement learning in a non-Markovian environment},
  author={Neider, Daniel and Gaglione, Jean-Rapha{\"e}l and Gavran, Ivan and Topcu, Ufuk and Wu, Bo and Xu, Zheng},
  booktitle={Proceedings of the 35th AAAI Conference on Artificial Intelligence},
  year={2021}
}

@article{Yang2023FineTuning,
  title={Fine-tuning language models using formal methods feedback},
  author={Yang, Yinqiao and Bhatt, Naman P. and Ingebrand, Theo and Ward, William and Carr, Steven and Wang, Zhangir and Topcu, Ufuk},
  journal={arXiv preprint arXiv:2310.18239},
  year={2023}
}

@article{Sutton1999Options,
  title={Between MDPs and semi-MDPs: A framework for temporal abstraction in reinforcement learning},
  author={Sutton, Richard S. and Precup, Doina and Singh, Satinder},
  journal={Artificial Intelligence},
  volume={112},
  number={1-2},
  pages={181--211},
  year={1999}
}

@inproceedings{Bacon2017OptionCritic,
  title={The option-critic architecture},
  author={Bacon, Pierre-Luc and Harb, Jean and Precup, Doina},
  booktitle={Proceedings of the AAAI Conference on Artificial Intelligence},
  volume={31},
  number={1},
  year={2017}
}

@inproceedings{Zhu2023Principled,
  title={Principled reinforcement learning with human feedback from pairwise or K-wise comparisons},
  author={Zhu, Banghua and Jordan, Michael I. and Jiao, Jiantao},
  booktitle={Proceedings of the 40th International Conference on Machine Learning},
  volume={202},
  pages={43037--43067},
  year={2023}
}

@article{Rafailov2023DirectPO,
  title={Direct preference optimization: Your language model is secretly a reward model},
  author={Rafailov, Rafael and Sharma, Archit and Mitchell, Eric and Ermon, Stefano and Manning, Christopher D. and Finn, Chelsea},
  journal={arXiv preprint arXiv:2305.18290},
  year={2023}
}

@inproceedings{Park2022SURF,
  title={SURF: Semi-supervised reward learning with data augmentation for feedback-efficient preference-based reinforcement learning},
  author={Park, Jinyoung and Chun, Jongmin and Bae, Seungjae and Kim, Kyungjae and Yoo, Kee-Eung},
  booktitle={Proceedings of the International Conference on Learning Representations},
  year={2022}
}

@inproceedings{Biyik2020Active,
  title={Active preference-based learning of reward functions},
  author={Sadigh, Dorsa and Dragan, Anca D. and Sastry, S. Shankar and Seshia, Sanjit A.},
  booktitle={Proceedings of Robotics: Science and Systems},
  year={2017}
}

@article{Gleave2022Uncertainty,
  title={Uncertainty estimation for language reward models},
  author={Gleave, Adam and Irving, Geoffrey},
  journal={arXiv preprint arXiv:2203.07472},
  year={2022}
}

@article{Leike2018Scalable,
  title={Scalable agent alignment via reward modeling: A research direction},
  author={Leike, Jan and Krueger, David and Everitt, Tom and Martic, Miljan and Maini, Vishal and Legg, Shane},
  journal={arXiv preprint arXiv:1811.07871},
  year={2018}
}

@inproceedings{Ho2016Generative,
  title={Generative adversarial imitation learning},
  author={Ho, Jonathan and Ermon, Stefano},
  booktitle={Advances in Neural Information Processing Systems},
  pages={4565--4573},
  year={2016}
}

@inproceedings{Abbeel2004Apprenticeship,
  title={Apprenticeship learning via inverse reinforcement learning},
  author={Abbeel, Pieter and Ng, Andrew Y.},
  booktitle={Proceedings of the 21st International Conference on Machine Learning},
  pages={1},
  year={2004}
}

@inproceedings{Bakker2002Reinforcement,
  title={Reinforcement learning with long short-term memory},
  author={Bakker, Bram},
  booktitle={Advances in Neural Information Processing Systems},
  pages={931--938},
  year={2001}
}

@inproceedings{Wierstra2007Solving,
  title={Solving deep memory POMDPs with recurrent policy gradients},
  author={Wierstra, Daan and Foerster, Alexander and Peters, Jan and Schmidhuber, J{\"u}rgen},
  booktitle={Proceedings of the International Conference on Artificial Neural Networks},
  pages={697--706},
  year={2007}
}

@inproceedings{Dupont1996Incremental,
  title={What is the search space of the regular inference?},
  author={Dupont, Pierre and Miclet, Laurent and Vidal, Enrique},
  booktitle={Grammatical Inference and Applications},
  pages={25--37},
  year={1994}
}

@inproceedings{Walkinshaw2016Inferring,
  title={Inferring finite-state models with temporal constraints},
  author={Walkinshaw, Neil and Derrick, John and Guo, Qiang},
  booktitle={Proceedings of the IEEE/ACM International Conference on Automated Software Engineering},
  pages={248--257},
  year={2008}
}

@inproceedings{Xu2020Learning,
  title={Joint inference of reward machines and policies for reinforcement learning},
  author={Xu, Zheng and Gavran, Ivan and Ahmad, Yousef and Majumdar, Rupak and Neider, Daniel and Topcu, Ufuk and Wu, Bo},
  booktitle={Proceedings of the International Conference on Automated Planning and Scheduling},
  volume={30},
  pages={590--598},
  year={2020}
}

@article{Jothimurugan2019Composable,
  title={A composable specification language for reinforcement learning tasks},
  author={Jothimurugan, Kishor and Alur, Rajeev and Bastani, Osbert},
  journal={arXiv preprint arXiv:2008.09293},
  year={2020}
}

@article{Hasanbeig2018LogicGuided,
  title={Logically-constrained reinforcement learning},
  author={Hasanbeig, Mohammadhosein and Abate, Alessandro and Kroening, Daniel},
  journal={arXiv preprint arXiv:1801.08099},
  year={2018}
}

@inproceedings{Hasanbeig2019LTL,
author = {Hasanbeig, M. and Kantaros, Y. and Abate, A. and Kroening, D. and Pappas, G. J. and Lee, I.},
title = {Reinforcement Learning for Temporal Logic Control Synthesis with Probabilistic Satisfaction Guarantees},
year = {2019},
publisher = {IEEE Press},
url = {https://doi.org/10.1109/CDC40024.2019.9028919},
booktitle = {2019 IEEE 58th Conference on Decision and Control (CDC)},
pages = {5338–5343},
numpages = {6},
location = {Nice, France}
}

@inproceedings{Qiu2023Instructing,
  title={Instructing goal-conditioned reinforcement learning agents with temporal logic objectives},
  author={Qiu, Wenjie and Mao, Wensen and Zhu, Hao},
  booktitle={Advances in Neural Information Processing Systems},
  year={2023}
}

@inproceedings{Jothimurugan2021Compositional,
  title={Compositional reinforcement learning from logical specifications},
  author={Jothimurugan, Kishor and Bansal, Suguman and Bastani, Osbert and Alur, Rajeev},
  booktitle={Advances in Neural Information Processing Systems},
  volume={34},
  pages={10026--10039},
  year={2021}
}

@article{Yalcinkaya2024Compositional,
  title={Compositional automata embeddings for goal-conditioned reinforcement learning},
  author={Yalcinkaya, Mete and Xu, Zheng},
  journal={arXiv preprint arXiv:2401.13242},
  year={2024}
}

@book{Baier2008Principles,
  title={Principles of Model Checking},
  author={Baier, Christel and Katoen, Joost-Pieter},
  publisher={MIT Press},
  address={Cambridge, MA, USA},
  year={2008}
}

@inproceedings{Hasanbeig2021DeepSynth,
  title={DeepSynth: Automata synthesis for automatic task segmentation in deep reinforcement learning},
  author={Hasanbeig, Mohammadhosein and Jeppu, Natasha Yogananda and Abate, Alessandro and Melham, Tom F. and Kroening, Daniel},
  booktitle={Proceedings of the AAAI Conference on Artificial Intelligence},
  year={2021}
}

@inproceedings{Alshiekh2018Safe,
  title={Safe reinforcement learning via shielding},
  author={Alshiekh, Mohammed and Bloem, Roderick and Ehlers, R{\"u}diger and K{\"o}nighofer, Bettina and Niekum, Scott and Topcu, Ufuk},
  booktitle={Proceedings of the AAAI Conference on Artificial Intelligence},
  volume={32},
  number={1},
  year={2018}
}

@inproceedings{Hahn2019Omega,
  title={Omega-regular objectives in model-free reinforcement learning},
  author={Hahn, Ernst Moritz and Perez, Mateo and Schewe, Sven and Somenzi, Fabio and Trivedi, Ashutosh and Wojtczak, Dominik},
  booktitle={Proceedings of the International Conference on Tools and Algorithms for the Construction and Analysis of Systems},
  pages={395--412},
  year={2019}
}

@inproceedings{Barreto2017SuccessorFeatures,
  title={Successor features for transfer in reinforcement learning},
  author={Barreto, Andr{\'e} and Dabney, Will and Munos, R{\'e}mi and Hunt, Jonathan J. and Schaul, Tom and van Hasselt, Hado and Silver, David},
  booktitle={Advances in Neural Information Processing Systems},
  volume={30},
  year={2016}
}

@inproceedings{Finn2017MAML,
  title={Model-agnostic meta-learning for fast adaptation of deep networks},
  author={Finn, Chelsea and Abbeel, Pieter and Levine, Sergey},
  booktitle={Proceedings of the 34th International Conference on Machine Learning},
  pages={1126--1135},
  year={2017}
}

@inproceedings{Andreas2017Modular,
  title={Modular multitask reinforcement learning with policy sketches},
  author={Andreas, Jacob and Klein, Dan and Levine, Sergey},
  booktitle={Proceedings of the 34th International Conference on Machine Learning},
  pages={166--175},
  year={2017}
}

@inproceedings{Fujimoto2018Td3,
  title={Addressing function approximation error in actor-critic methods},
  author={Fujimoto, Scott and van Hoof, Herke and Meger, David},
  booktitle={Proceedings of the 35th International Conference on Machine Learning},
  pages={1587--1596},
  year={2018}
}

@article{singh2010intrinsically,
  title={Intrinsically motivated reinforcement learning: An evolutionary perspective},
  author={Singh, Satinder and Lewis, Richard L. and Barto, Andrew G. and Sorg, Jonathan},
  journal={IEEE Transactions on Autonomous Mental Development},
  volume={2},
  number={2},
  pages={70--82},
  year={2010}
}

@article{konidaris2009skill,
  title={Robot learning from demonstration by constructing skill trees},
  author={Konidaris, George and Kuindersma, Scott and Barto, Andrew G. and Grupen, Robert},
  journal={International Journal of Robotics Research},
  volume={31},
  number={3},
  pages={360--375},
  year={2012}
}

@article{dietterich2000hierarchical,
  title={Hierarchical reinforcement learning with the MAXQ value function decomposition},
  author={Dietterich, Thomas G.},
  journal={Journal of Artificial Intelligence Research},
  volume={13},
  pages={227--303},
  year={2000}
}

@article{singireddy2023automaton,
  title={Automaton distillation: Neuro-symbolic transfer learning for deep reinforcement learning},
  author={Singireddy, Suraj and Nwaorgu, Prince and Beckus, Adrien and McKinney, Aaron and Enyioha, Chinwendu and Jha, Susmit Kumar and Atia, George K. and Velasquez, Alvaro},
  journal={arXiv preprint arXiv:2310.19137},
  year={2023}
}

@inproceedings{alinejad2025bidirectional,
  title={Bidirectional end-to-end framework for transfer from abstract models in non-Markovian reinforcement learning},
  author={Alinejad, Mahyar and Nwaorgu, Prince and Enyioha, Chinwendu and Wang, Yue and Velasquez, Alvaro and Atia, George K.},
  booktitle={Proceedings of the International Conference on Neuro-symbolic Systems},
  volume={288},
  pages={643--660},
  year={2025},
  organization={PMLR}
}

@article{Watkins1992,
  title={Q-learning},
  author={Watkins, Christopher J. C. H. and Dayan, Peter},
  journal={Machine Learning},
  volume={8},
  number={3},
  pages={279--292},
  year={1992}
}

@article{Kaelbling1998Planning,
  title={Planning and acting in partially observable stochastic domains},
  author={Kaelbling, Leslie Pack and Littman, Michael L. and Cassandra, Anthony R.},
  journal={Artificial Intelligence},
  volume={101},
  number={1-2},
  pages={99--134},
  year={1998}
}

@article{Russo2018Learning,
  title={A tutorial on Thompson sampling},
  author={Russo, Daniel J. and Van Roy, Benjamin and Kazerouni, Abbas and Osband, Ian and Wen, Zheng},
  journal={Foundations and Trends in Machine Learning},
  volume={11},
  number={1},
  pages={1--96},
  year={2018}
}

\appendix

\newpage
\section{Distillation with Reward Shaping}
\label{appdx:distill}

Here, we provide more details on the \emph{Distillation with Reward Shaping} method. In contrast to the preference-based algorithms or known reward function method, this approach relies on a given automaton structure to shape a minimal numeric reward. It neither compares trajectories pairwise nor uses a handcrafted reward function. Instead, it leverages automaton transitions to provide positive feedback upon subgoal completion.

\subsection{Method Overview}
We assume the agent must satisfy a sequence of subgoals described by a DFA $\mathcal{A}$. Instead of assigning custom numeric rewards for each subgoal or comparing pairs of trajectories, the agent obtains:
\begin{itemize}
    \item \(+1\) reward whenever it triggers an automaton transition associated with completing a subgoal step;
    \item \(-0.1\) penalty per action step otherwise, discouraging wandering.
\end{itemize}
No additional reward shaping is introduced. Thus, once the agent transitions from automaton state \(\omega\) to \(\omega'\), it knows it has \emph{correctly} advanced the subgoal sequence. If \(q'\) is an accepting automaton state, the episode terminates successfully.

\begin{remark}
If the automaton is \emph{learned} (e.g., via RPNI~\cite{Oncina1992}), the learning step happens \emph{outside} the RL loop. In our experiments, we treat the automaton as \emph{given}, so the policy training procedure focuses on the environment plus this minimal numeric shaping.
\end{remark}

\subsection{Knowledge Transfer (Teacher--Student Setup)}
\label{appdx:knowledge_transfer}
To accelerate learning in more complex environments, we adopt a teacher--student paradigm:
\begin{enumerate}
    \item \textbf{Teacher phase (simpler or well-understood environment).} We run Q-learning in a product MDP \(\mathcal{M}_d \otimes \mathcal{A}\), where \(\mathcal{M}_d\) is a discrete ``teacher'' environment with the same subgoal structure. The teacher records Q-values for each automaton transition, effectively capturing how ``valuable'' it was to advance subgoals from one automaton state \(\omega\) to another.

    \item \textbf{Distillation of automaton transitions.} We aggregate the teacher's Q-values over transitions \((q, \sigma)\), where \(\sigma\) is the atomic proposition labeling the next subgoal. This yields an \emph{average} or \emph{representative} value, \(\overline{Q}_{\text{teacher}}(q,\sigma)\), for each automaton transition.

    \item \textbf{Student phase (target environment).} We then train a student agent in the product MDP \(\mathcal{M}_s \otimes \mathcal{A}\), where \(\mathcal{M}_s\) is the new (potentially more complex) environment but \emph{with the same} automaton \(\mathcal{A}\). The student still uses minimal numeric shaping (\(+1\) on subgoal transitions, \(-0.1\) otherwise), \emph{but} it incorporates the teacher's knowledge in its Q-updates:
    \[
    Q'_{\text{student}}((s,q), a) 
    = \beta(q, \sigma) \overline{Q}_{\text{teacher}}(q, \sigma)
    + \bigl(1 - \beta(q, \sigma)\bigr) \Big( r + \gamma \max_{a'} Q_{\text{student}}((s',q'), a') \Big),
    \]
    where \(\beta(q, \sigma) = \rho^{\eta(q, \sigma)}\) controls the balance between the teacher's guidance and the student's own experience, annealing as the agent visits each automaton transition more frequently (\(\eta(q, \sigma)\) is the visit count).

    \item \textbf{Policy optimization.} The student updates its Q-table as:
    \[
    Q((s_t,q_t), a_t) \leftarrow Q((s_t,q_t), a_t) + \alpha \Big( r(q_t,\sigma_t) 
    + \gamma \max_{a'} Q((s_{t+1},q_{t+1}), a') - Q((s_t,q_t), a_t) \Big),
    \]
    gradually converging to a policy \(\pi^*\) that balances teacher guidance with environment-specific exploration.
\end{enumerate}

\subsection{Relation to Our Main Approaches}
Unlike the \textbf{static} and \textbf{dynamic} preference-based methods, which \emph{learn} a reward model from trajectory comparisons, \emph{Distillation with Reward Shaping} simply exploits the automaton structure to produce local numeric signals. It also differs from the \textbf{known reward function} approach in that we do not hand-specify the reward terms for each subgoal or final goal. Instead, the DFA transitions themselves indicate subgoal completions, providing minimal positive increments to guide the agent. Experiments illustrate how this method can serve as a teacher for subsequent knowledge transfer, ultimately enabling efficient learning in new or larger environments with the same subgoal logic.

\section{Reward Machine Baseline}
\label{appdx:rm}

Here we provide more details on the \emph{Reward Machine (RM)} baseline method. This approach directly defines rewards based on automaton state transitions, offering a formalized alternative to preference-based learning for non-Markovian tasks.

\subsection{Method Overview}
RMs provide an interpretable mechanism for encoding non-Markovian reward functions~\cite{Icarte2022reward, ToroIcarte2018UsingRL}. An RM is defined as a tuple $\mathcal{R} = (U, u_0, F, \delta_u, \delta_r)$, where:
\begin{itemize}
    \item $U$ is a finite set of states representing task progression
    \item $u_0 \in U$ is the initial state
    \item $F \subseteq U$ is the set of terminal states
    \item $\delta_u: U \times \mathcal{L} \rightarrow U$ is the state-transition function triggered by environmental events in $\mathcal{L}$
    \item $\delta_r: U \times \mathcal{L} \times U \rightarrow \mathbb{R}$ is the reward-transition function
\end{itemize}

Here, $\mathcal{L}$ denotes the set of high-level events or labels that can be detected in the environment. For example, in our tasks, $\mathcal{L}$ includes events such as "wood" (for collecting wood), "key" (for picking up a key), or "factory" (for reaching the factory location). These events serve as triggers for state transitions in the reward machine and are extracted from the environment state through a labeling function.

In our implementation, the agent obtains:
\begin{itemize}
    \item \(+5.0\) reward for each valid transition that represents subgoal completion in the correct order
    \item \(-1.0\) penalty for attempting to complete subgoals out of order
    \item \(-0.1\) penalty per action step to encourage efficiency
    \item \(+10.0\) reward upon reaching the terminal state after all subgoals are completed in order
\end{itemize}

This approach explicitly encodes the temporal structure of the task, creating a more informative reward landscape while maintaining the logical constraints of sequential subtask completion~\cite{Icarte2022reward}.

\subsection{Implementation Details}
Learning proceeds in the product MDP $\mathcal{M}_{\text{prod}} = \mathcal{M} \times \mathcal{R}$, where the agent's state is augmented with the current RM state. This creates a Markovian decision process from the originally non-Markovian task, as the RM state captures the necessary history information.

For policy learning in this product MDP, we use Q-learning with the following update rule:
\[
Q((s,u), a) \leftarrow Q((s,u), a) + \alpha \left( r + \gamma \max_{a'} Q((s',u'), a') - Q((s,u), a) \right)
\]
where $(s,u)$ is the current state-RM state pair, $a$ is the action, $r$ is the reward received from the RM's reward-transition function $\delta_r$, and $(s',u')$ is the next state-RM state pair with $u' = \delta_u(u, L(s'))$, where $L(s')$ maps the environment state to relevant atomic propositions.

\subsection{Relation to Our Approach}
While both the RM baseline and our preference-based method leverage automaton structures to handle non-Markovian rewards, they differ fundamentally in their approach:

\begin{itemize}
    \item \textbf{Reward Specification:} RMs directly map automaton transitions to specific numeric rewards, requiring careful engineering of reward values. Our approach learns these values from preferences, eliminating the need for manual reward tuning.
    
    \item \textbf{Learning Mechanism:} RMs create a product MDP where learning occurs in an augmented state space. Our method separates preference generation from reward learning, enabling more flexible adaptation to task structure.
    
    \item \textbf{Transferability:} While both approaches can leverage automaton structures for transfer learning, our method's preference-based nature provides additional flexibility when transferring across environments with different dynamics but the same logical structure.
\end{itemize}

The RM baseline represents a state-of-the-art approach for structured reward specification in complex non-Markovian tasks, providing a formal alternative to our preference-based method.

\section{LTL-Guided Reinforcement Learning (LPOPL) Baseline}
\label{appdx:lpopl}

Here we provide more details on the \emph{LTL-Guided RL (LPOPL)} baseline method. This approach leverages Linear Temporal Logic (LTL) to guide RL through formal task specifications, providing theoretically-grounded task guidance without relying on preference elicitation.

\subsection{Method Overview}
Logic-Guided Policy Optimization~\cite{Li2017TLTL} uses LTL formulas to encode temporal task requirements. For sequential subgoal tasks, we encode the requirement as an LTL formula:
\[
\varphi = \Diamond(g_1 \land \Diamond(g_2 \land \Diamond(\ldots \land \Diamond g_n)))
\]
where $\Diamond$ is the "eventually" operator and $g_i$ represents the $i$-th subgoal. This formula is then translated into a Deterministic Finite Automaton (DFA) $\mathcal{A}_\varphi = (Q, q_0, F, \Sigma, \delta)$, where states in $Q$ track progress toward satisfying $\varphi$.

The key innovation in LPOPL is the use of potential-based reward shaping~\cite{Ng1999PolicyInvariant} derived from the automaton structure:
\begin{itemize}
    \item A potential function $\Phi: Q \rightarrow \mathbb{R}$ assigns higher values to states closer to acceptance
    \item The shaped reward is defined as $R'(s, a, s') = R(s, a, s') + \gamma \Phi(\delta(q, L(s'))) - \Phi(q)$
    \item A small negative base reward $R(s, a, s') = -0.1$ encourages efficient paths
    \item Terminal states that satisfy the LTL formula receive a large positive reward (+10.0)
\end{itemize}

This approach preserves optimality guarantees while providing denser feedback aligned with the LTL specification~\cite{Camacho2019LTLMOP}.

\subsection{Implementation Details}
We implement the potential function $\Phi$ based on the shortest path distance from each automaton state to an accepting state:
\[
\Phi(q) = 
\begin{cases}
    \max_{q' \in F} \{d(q, q')\} & \text{if } q \notin F \\
    0 & \text{if } q \in F
\end{cases}
\]
where $d(q, q')$ is the minimum number of transitions required to reach $q'$ from $q$ in the automaton. This creates a potential gradient that guides the agent toward satisfying the LTL specification.

The agent learns in the product MDP between the environment and the DFA, with the following Q-learning update:

\[
Q((s,q), a) \leftarrow Q((s,q), a) + \alpha \left( R'(s, a, s') + \gamma \max_{a'} Q((s',q'), a') - Q((s,q), a) \right)
\]
where $(s,q)$ is the current state-automaton state pair, $a$ is the action, $R'$ is the shaped reward, and $(s',q')$ is the next state-automaton state pair with $q' = \delta(q, L(s'))$.

\subsection{Relation to Our Approach}
Both LPOPL and our preference-based method use automaton structures to address non-Markovian rewards, but they differ in several key aspects:

\begin{itemize}
    \item \textbf{Formal Specification:} LPOPL requires LTL formulas to be manually translated into reward signals through potential functions. Our approach learns reward functions directly from automaton-generated preferences without requiring explicit reward engineering.
    
    \item \textbf{Reward Shaping:} LPOPL uses potential-based reward shaping to create dense rewards that preserve optimality. Our method learns a reward model directly from pairwise preferences, avoiding the need for careful shaping design.
    
    \item \textbf{Theoretical Properties:} LPOPL provides theoretical guarantees on optimality preservation through proper potential-based shaping. Our approach offers convergence guarantees based on preference consistency and sufficient exploration.
\end{itemize}

\section{Additional Experimental Results}
\label{appdx:exp}

This appendix provides additional experimental results that complement the findings presented in the main text. While the main paper focuses on reward per cumulative steps, here we present the complete set of metrics: reward per episode and steps per episode.

Figure~\ref{fig:reward_per_episode_all} shows \textbf{the reward per episode} for the four environments considered. The \emph{dynamic method} achieves stronger long-term gains through iterative refinement, though it starts with less stability. The \emph{static method} converges rapidly with a fixed learned reward, often matching or slightly outperforming the \emph{known reward function} baseline initially. Ultimately, the dynamic method surpasses all once it completes enough preference-driven updates.

\begin{figure}[!htbp]
\centering
\begin{minipage}{0.48\columnwidth}
  \centering
  \includegraphics[width=0.99\columnwidth]{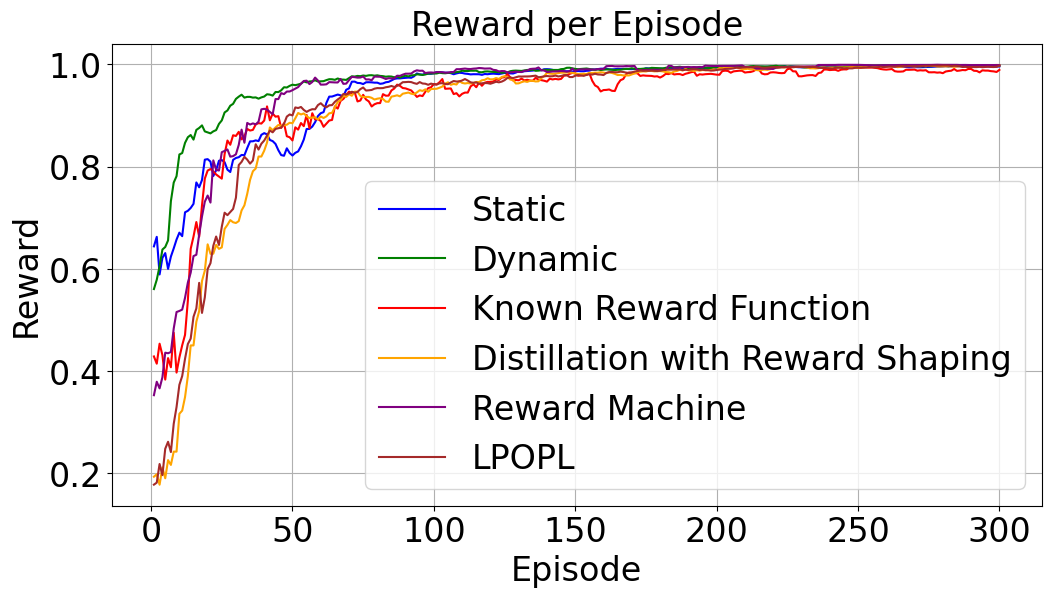}

  \label{fig:reward_minecraft}
\end{minipage}%
\hfill
\begin{minipage}{0.48\columnwidth}
  \centering
  \includegraphics[width=0.99\columnwidth]{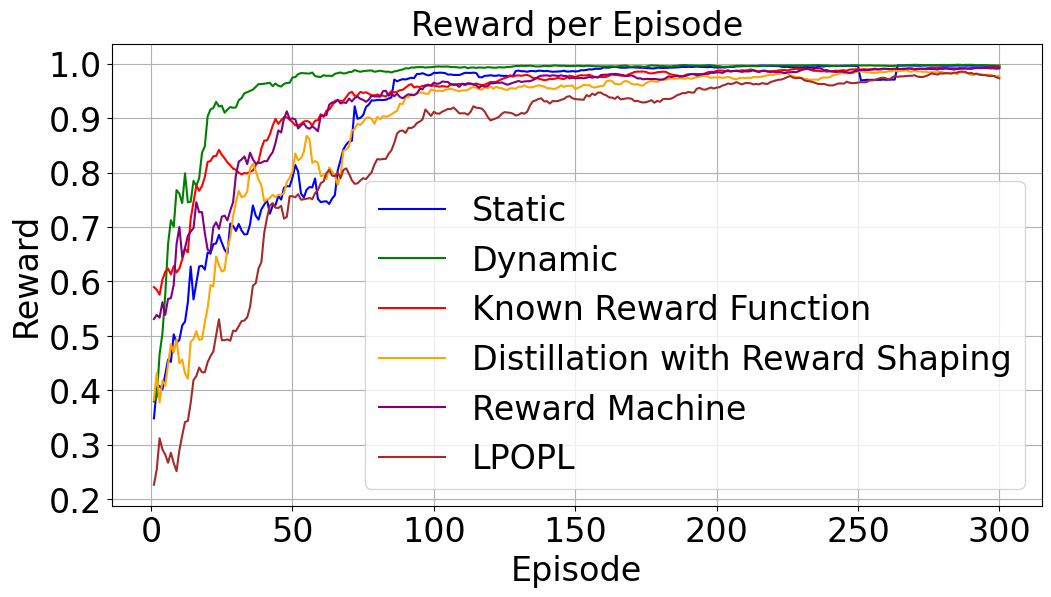}

  \label{fig:reward_dungeon}
\end{minipage}


\begin{minipage}{0.48\columnwidth}
  \centering
  \includegraphics[width=0.99\columnwidth]{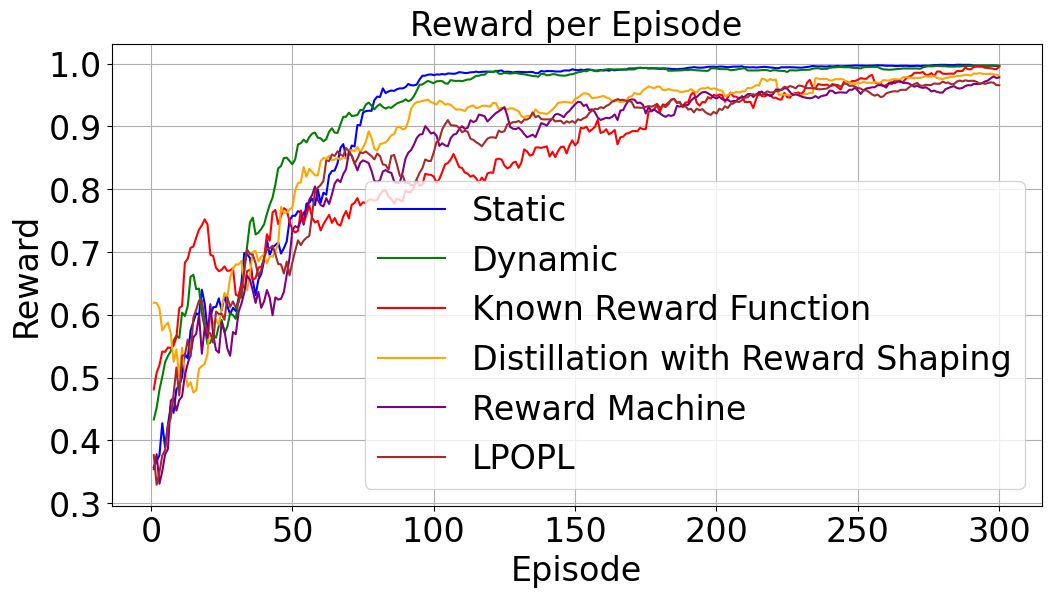}

  \label{fig:reward_craftsman}
\end{minipage}%
\hfill
\begin{minipage}{0.48\columnwidth}
  \centering
  \includegraphics[width=0.99\columnwidth]{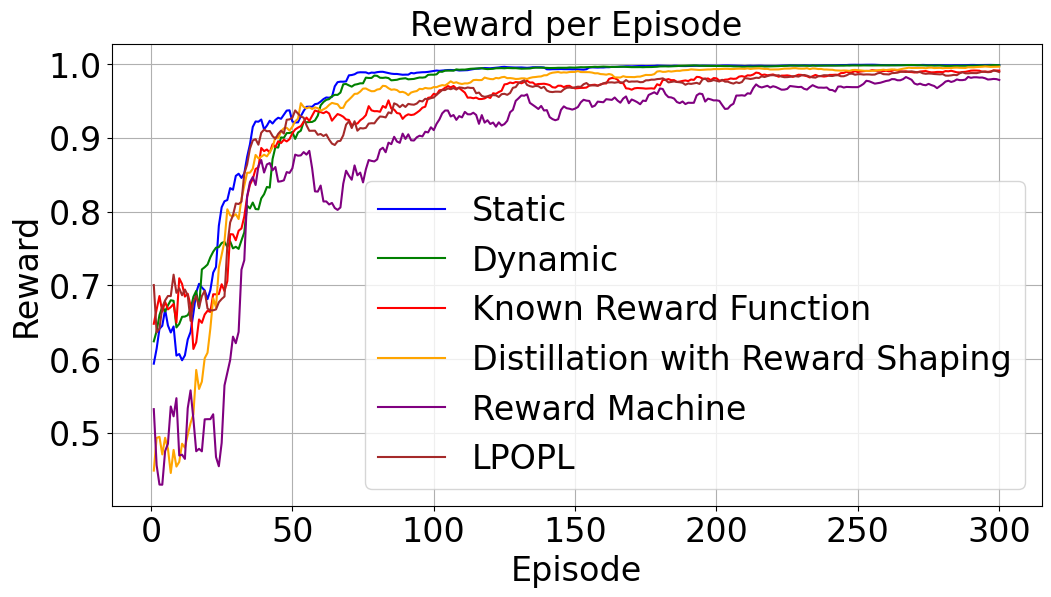}

  \label{fig:reward_bridge}
\end{minipage}
\vspace{-.3cm}
\caption{\small{Reward per episode for all environments.}} 
\label{fig:reward_per_episode_all}
\end{figure}

The \textbf{steps per episode} for each environment are shown in Figure~\ref{fig:steps_per_episode_all}. The static method tends to require fewer steps early on once the reward function is learned, but lacks further refinement. The dynamic method initially takes extra steps during its exploratory updates, yet ultimately converges to more efficient paths in two environments. However, in the other two environments, both methods still outperform the known reward function baseline. The \emph{distillation-based, reward machine and LTL based methods} do not converge as quickly as the dynamic method but maintains stable performance.

\begin{figure}[!htbp]
\centering
\begin{minipage}{0.48\columnwidth}
  \centering
  \includegraphics[width=0.99\columnwidth]{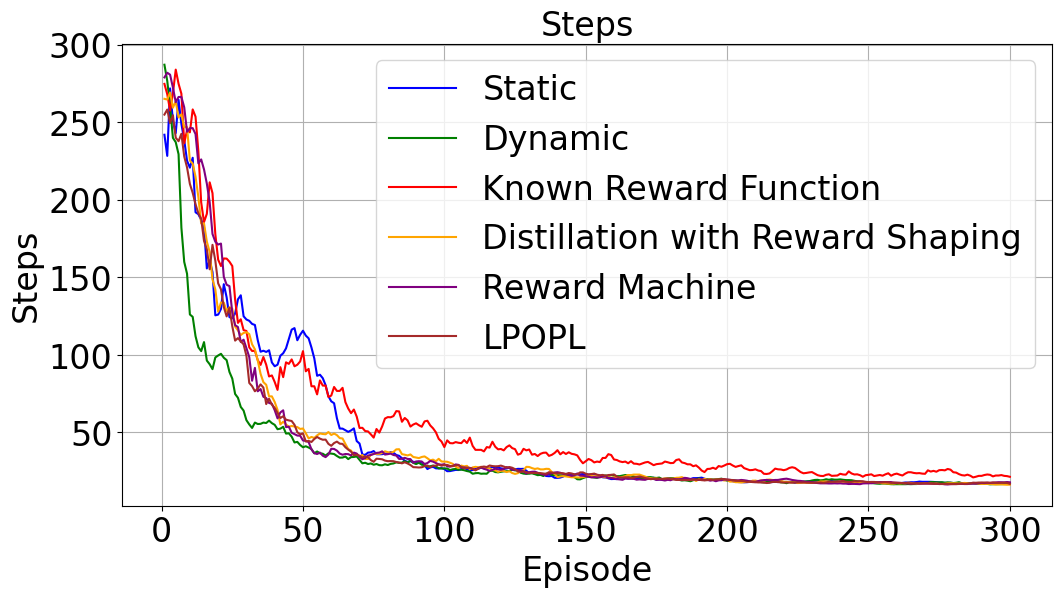}

  \label{fig:steps_minecraft}
\end{minipage}%
\hfill
\begin{minipage}{0.48\columnwidth}
  \centering
  \includegraphics[width=0.99\columnwidth]{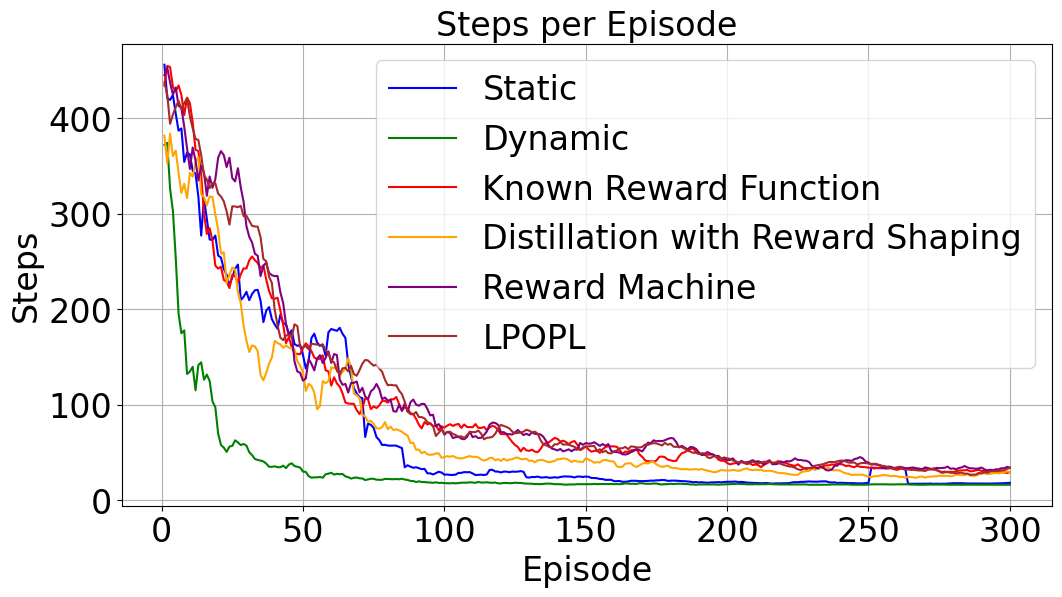}

  \label{fig:steps_dungeon}
\end{minipage}


\begin{minipage}{0.48\columnwidth}
  \centering
  \includegraphics[width=0.99\columnwidth]{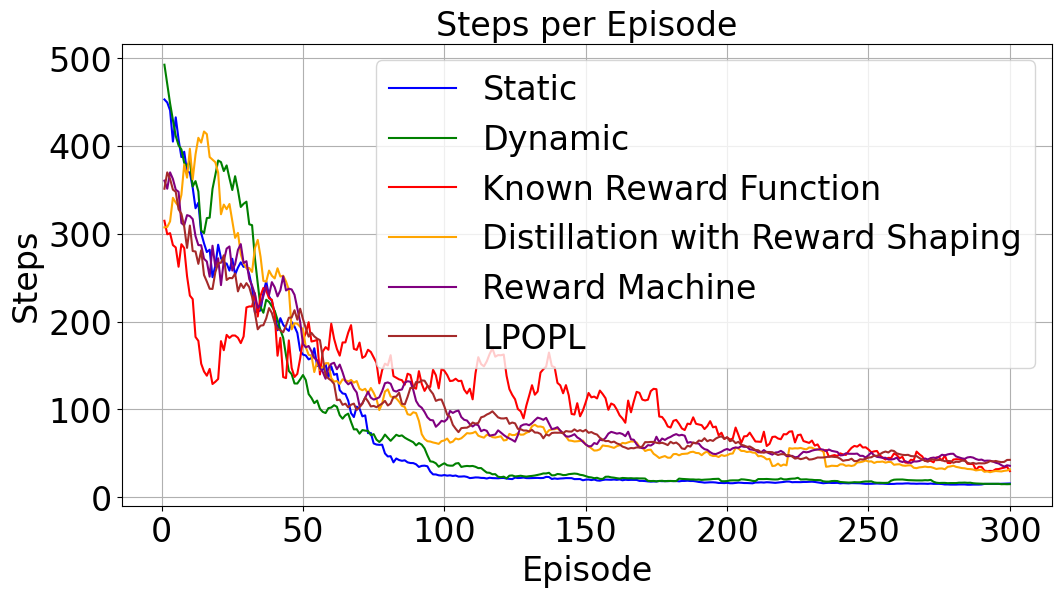}

  \label{fig:steps_craftsman}
\end{minipage}%
\hfill
\begin{minipage}{0.48\columnwidth}
  \centering
  \includegraphics[width=0.99\columnwidth]{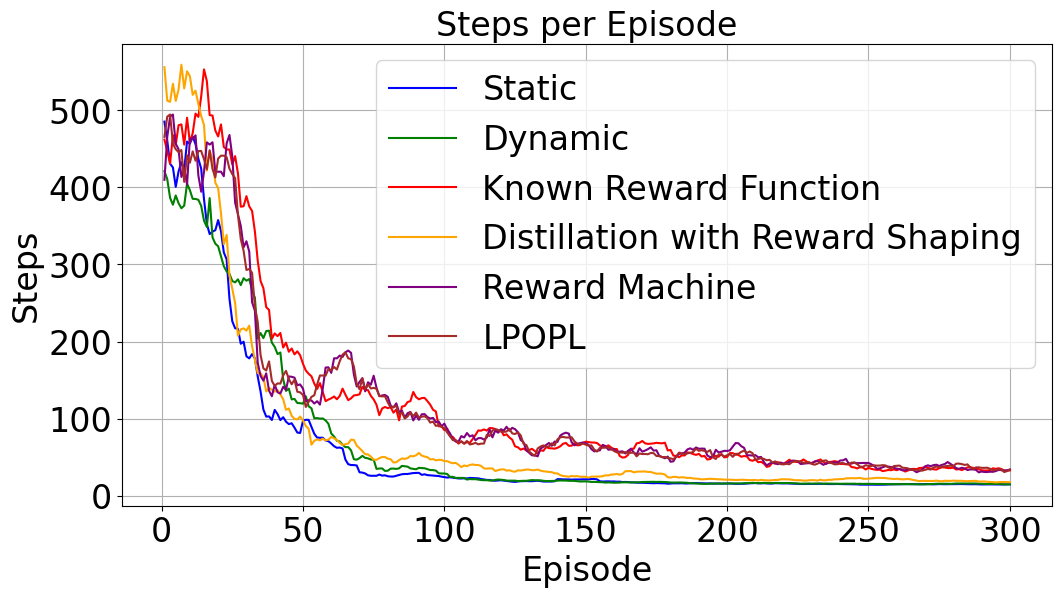}

  \label{fig:steps_bridge}
\end{minipage}
\vspace{-.3cm}
\caption{\small{Steps per episode for all environments.}}

\label{fig:steps_per_episode_all}
\end{figure}

\textbf{Combined scoring with transfer.}
We evaluate the effectiveness of combining subtask-based preferences with value-based knowledge transferred from a teacher using the scoring function.

\begin{figure}[!htbp]
\centering
\begin{minipage}{0.48\columnwidth}
  \centering
  \includegraphics[width=0.99\columnwidth]{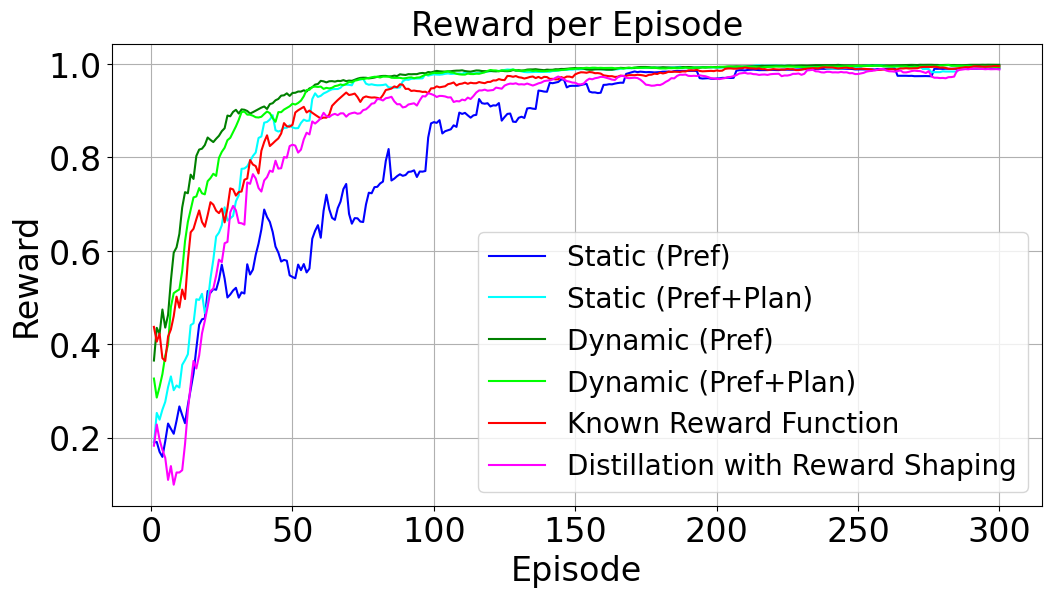}

  \label{fig:adjusted_reward_per_episode_env1}
\end{minipage}%
\hfill
\begin{minipage}{0.48\columnwidth}
  \centering
  \includegraphics[width=0.99\columnwidth]{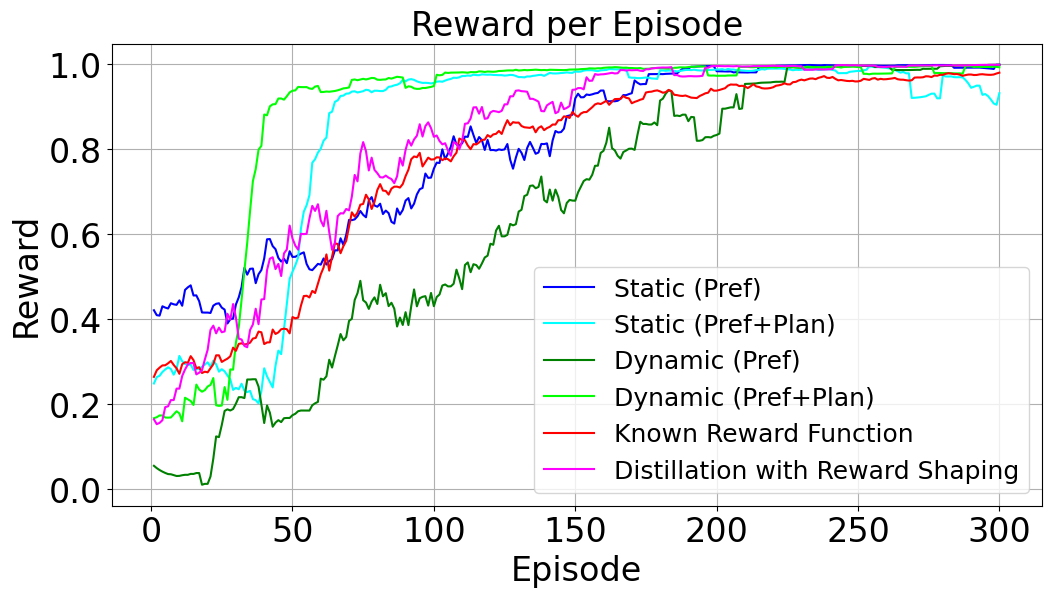}

  \label{fig:adjusted_reward_per_episode_env2}
\end{minipage}

 \vspace{0.5cm}

\begin{minipage}{0.48\columnwidth}
  \centering
  \includegraphics[width=0.99\columnwidth]{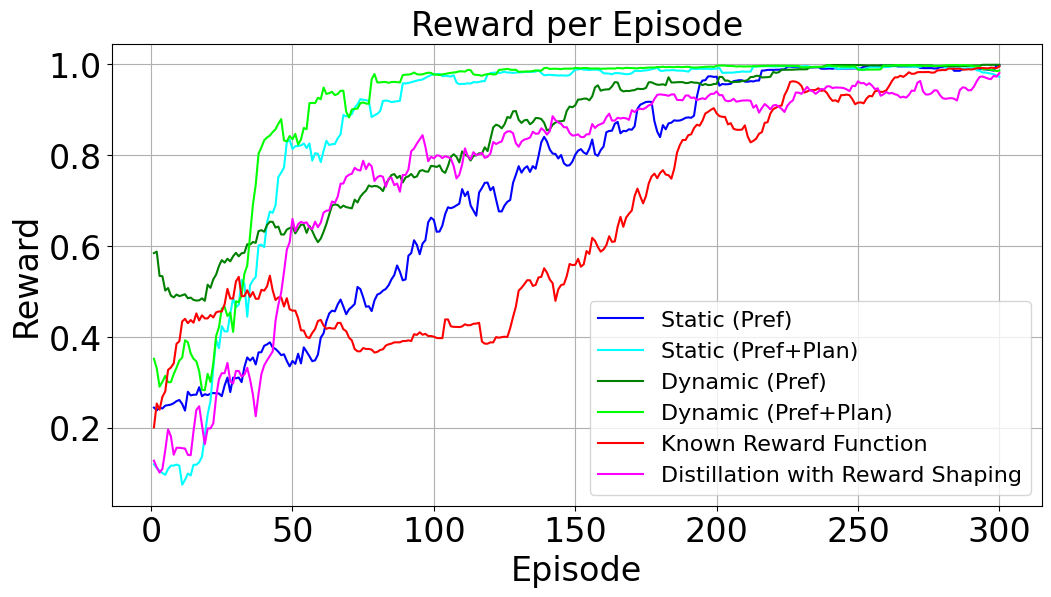}

  \label{fig:adjusted_reward_per_episode_env3}
\end{minipage}%
\hfill
\begin{minipage}{0.48\columnwidth}
  \centering
  \includegraphics[width=0.99\columnwidth]{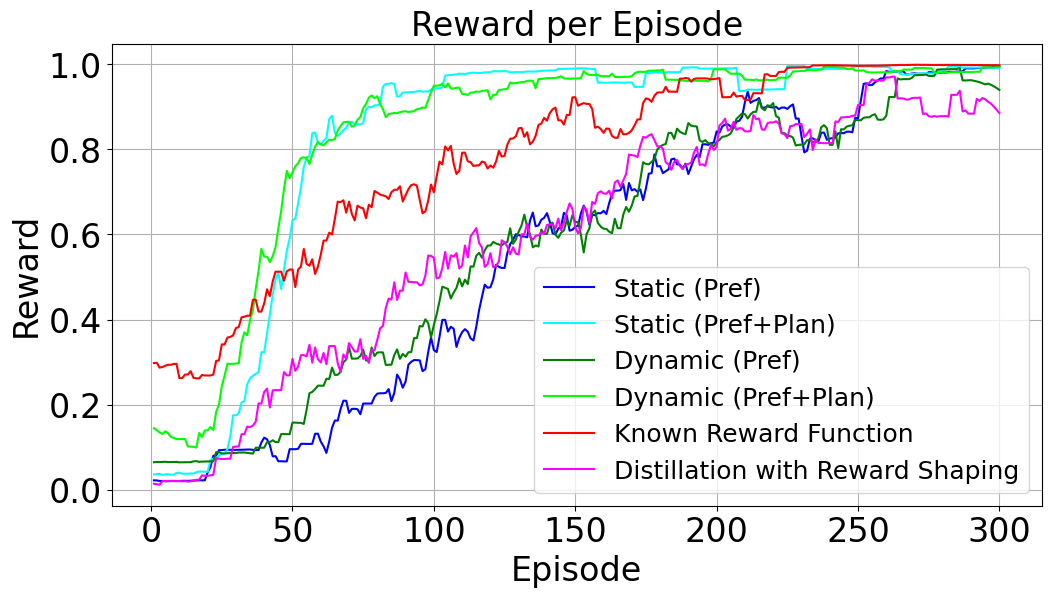}

  \label{fig:adjusted_reward_per_episode_env4}
\end{minipage}
\vspace{-.3cm}
\caption{\small{Reward per episode with combined scoring for all methods and environments.}}

\label{fig:reward_per_episode_all_transfer}
\end{figure}

\begin{figure}[!htbp]
\centering
\begin{minipage}{0.48\columnwidth}
  \centering
  \includegraphics[width=0.99\columnwidth]{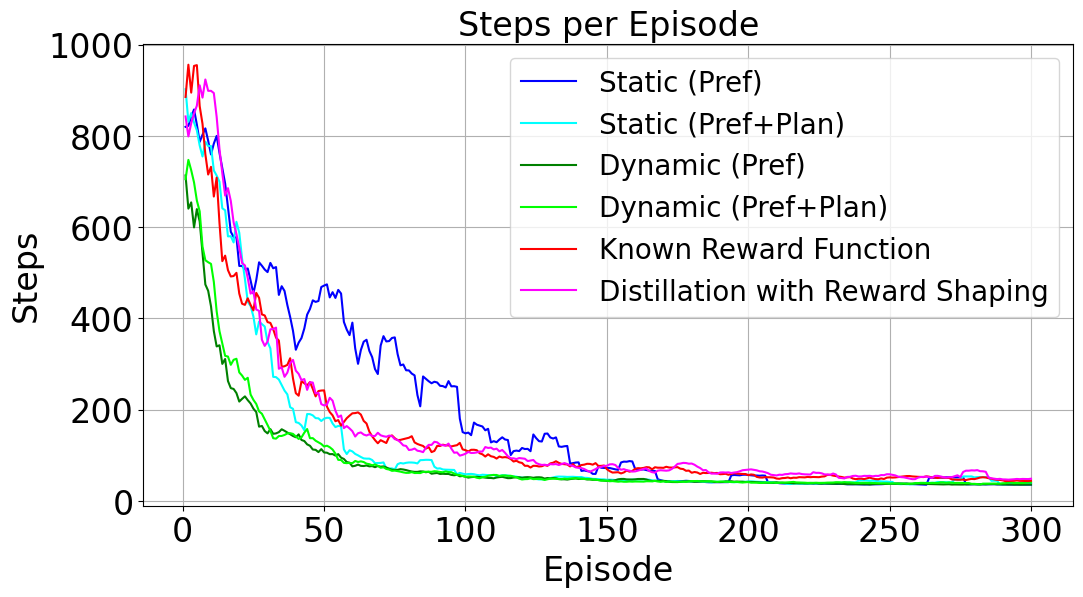}

  \label{fig:steps_per_episode_env1}
\end{minipage}%
\hfill
\begin{minipage}{0.48\columnwidth}
  \centering
  \includegraphics[width=0.99\columnwidth]{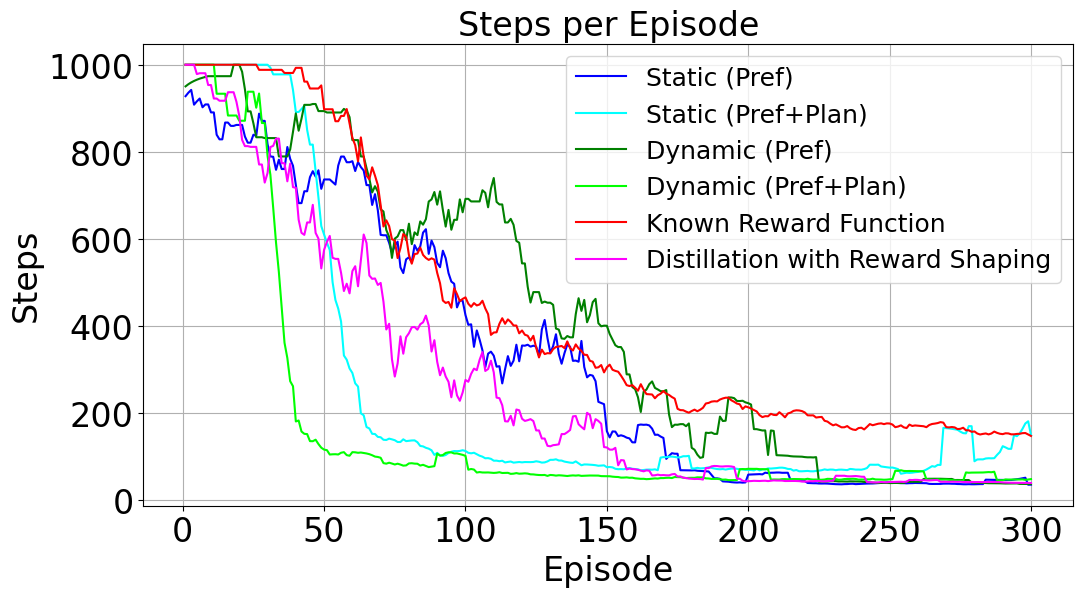}

  \label{fig:steps_per_episode_env2}
\end{minipage}


\begin{minipage}{0.48\columnwidth}
  \centering
  \includegraphics[width=0.99\columnwidth]{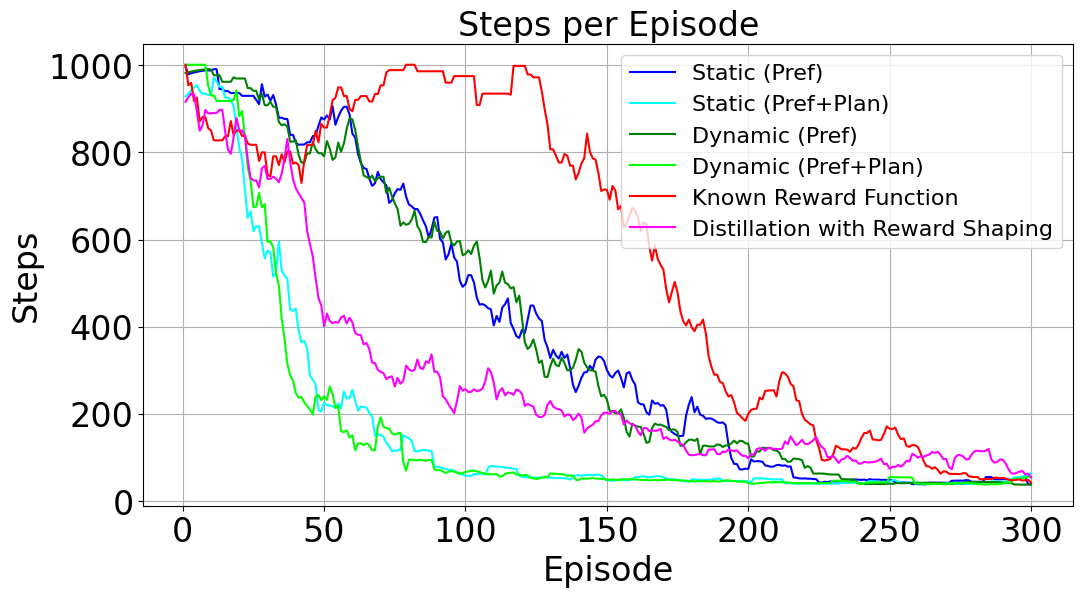}

  \label{fig:steps_per_episode_env3}
\end{minipage}%
\hfill
\begin{minipage}{0.48\columnwidth}
  \centering
  \includegraphics[width=0.99\columnwidth]{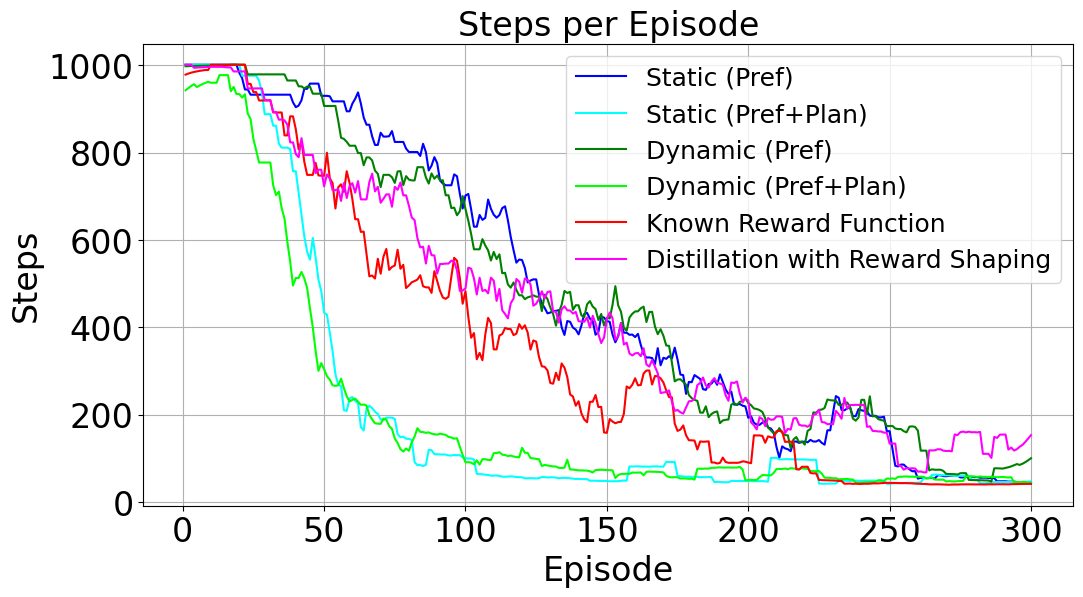}

  \label{fig:steps_per_episode_env4}
\end{minipage}
\vspace{-.3cm}
\caption{\small{Steps per episode with combined scoring for all methods and environments.}}

\label{fig:steps_per_episode_all_transfer}
\end{figure}

In all environments, distillation agents consistently outperform non-distilled methods, as indicated by higher episode rewards (Figure~\ref{fig:reward_per_episode_all_transfer}), fewer steps per episode (Figure~\ref{fig:steps_per_episode_all_transfer}), and a steeper reward increase over total steps (shown in the main text). These comprehensive results demonstrate the enhanced performance achieved by leveraging teacher-derived automaton knowledge across evaluation metrics.

\section{Belief-State Based Preference Elicitation for Partially Observable Environments}
\label{appdx:pomdp}

Agents may operate with limited information about their environment, in which case the fully observable assumption does not hold. We extend our framework to address this challenge by introducing a belief-state based preference elicitation approach for partially observable environments.

\textbf{Belief-State Based Scoring.}
In partially observable environments, the agent has limited information about the true state of the environment. We formalize a partially observable environment as a tuple $\mathcal{M}_{\text{po}} = (S, s_0, A, T, \Omega, O, R^*)$, where $S$, $s_0$, $A$, and $T$ are defined as before, $\Omega$ is the set of observations, $O: S \times A \to \Delta(\Omega)$ is the observation function mapping state-action pairs to distributions over observations, and $R^*$ is the unknown reward function.

Given that the agent cannot directly observe the full state, we maintain a belief state $b \in \Delta(S)$, representing a probability distribution over possible environment states. For each unobserved subgoal $g \in G$ (where $G$ is the set of subgoals defined by the DFA), we define a belief value $b(g) \in [0,1]$ indicating the agent's confidence that $g$ is at a particular location.

The belief-state trajectory scoring function integrates three components: (i) \textit{observed subgoal completion}, (ii) \textit{expected progress toward unobserved subgoals}, and (iii) \textit{information gain}. Formally, for a trajectory $\tau = (s_0, a_0, o_0, b_0, s_1, a_1, o_1, b_1, \ldots, s_T)$, where $o_t$ represents the observation and $b_t$ the belief state at time $t$, the score is computed as:
\begin{align}
\text{score}(\tau) = w_c \cdot N_c(\tau) + w_b \cdot P_b(\tau) + w_i \cdot I(\tau),
\label{eq:score_belief_based}
\end{align}
where $N_c(\tau)$ is the number of subgoals confirmed to be completed (directly observed), $P_b(\tau)$ is a measure of expected progress based on belief states, and $I(\tau)$ is the information gain achieved throughout the trajectory. The weights $w_c$, $w_b$, and $w_i$ prioritize different aspects, with typically $w_c > w_b > w_i$.

The expected progress term $P_b(\tau)$ is defined as:
\begin{align}
P_b(\tau) = \sum_{g \in G} \sum_{t=0}^{T-1} \max(0, b_{t+1}(g) - b_t(g)),
\end{align}
which captures improvements in belief confidence for each subgoal over time. The information gain $I(\tau)$ quantifies the reduction in state uncertainty:
\begin{align}
I(\tau) = \sum_{t=0}^{T-1} \bigl(H(b_t) - H(b_{t+1})\bigr),
\end{align}
where $H(b) = -\sum_{s \in S} b(s) \log b(s)$ is the entropy of belief state $b$.

\begin{remark}
The belief-state approach effectively handles partial observability by: (1) rewarding confirmed subgoal completions, (2) encouraging trajectories that improve belief accuracy about unobserved subgoals, and (3) promoting systematic exploration to reduce state uncertainty. This results in more efficient exploration in environments where complete state information is unavailable.
\end{remark}

The belief state $b_t$ is updated after each action and observation using standard Bayesian inference:
\begin{align}
b_{t+1}(s') = \eta \cdot O(o_{t+1} | s', a_t) \sum_{s \in S} T(s' | s, a_t) b_t(s),
\end{align}
where $\eta$ is a normalization constant ensuring $\sum_{s' \in S} b_{t+1}(s') = 1$.

For applications in automaton-guided environments, we augment the belief state to incorporate confidence in DFA states. For each DFA state $q \in Q$, we maintain a confidence value $c(q) \in [0,1]$, representing our degree of certainty that the automaton is in state $q$ given partial observations. The DFA state confidence is updated based on observed evidence of subgoal completion and the structure of the automaton.

Given two trajectories $\tau_1$ and $\tau_2$, their scores under Equation \eqref{eq:score_belief_based} are compared to establish preferences. These belief-aware preferences enable the agent to learn effective policies even with limited observability, by focusing exploration on regions with high expected information gain and by intelligently tracking progress through the automaton's states even when direct confirmation is not available.

\subsection{Belief-State Policy Optimization}

Building upon our belief-state preference elicitation method, we extend the policy optimization approach to handle partially observable environments. Since the agent cannot directly observe the full state, we optimize over belief states rather than raw environment states.

\textbf{Belief-State Value Function Approximation.}
We parameterize the reward function as $\hat{r}_\theta(s, b, q, c, a)$, where $s$ is the directly observable component of the state, $b$ is the belief state over unobserved elements, $q$ is the observed part of the current DFA state, $c$ is the confidence level in the unobserved part of the DFA state, and $a$ is the action. This reward function is trained using the same pairwise ranking loss as in Equation \eqref{eq:ranking_loss}, but operating over belief-augmented trajectories.

The Q-function for the belief-state MDP is learned as:
\begin{align}
Q((s,b,q,c), a) &= \mathbb{E}\bigg[\hat{r}_\theta(s, b, q, c, a) + \gamma \max_{a'} Q((s',b',q',c'), a') \bigg],
\end{align}
where $(s',b',q',c')$ is the next belief-augmented state after taking action $a$ in state $(s,b,q,c)$.

\textbf{Intrinsic Exploration Bonuses.}
To encourage efficient exploration under partial observability, we can augment the learned reward function with intrinsic motivation terms:
\begin{align}
r_{\text{total}}(s, b, q, c, a) = \hat{r}_\theta(s, b, q, c, a) + \beta_1 r_{\text{info}}(b, a) + \beta_2 r_{\text{conf}}(c),
\end{align}
where $r_{\text{info}}(b, a)$ is an information gain bonus for actions that are expected to reduce uncertainty in the belief state, and $r_{\text{conf}}(c)$ rewards increases in DFA state confidence. The hyperparameters $\beta_1$ and $\beta_2$ control the relative importance of exploration versus exploitation.

The information gain bonus is formally defined as:
\begin{align}
r_{\text{info}}(b, a) = \mathbb{E}_{o \sim O(\cdot|s,a)}\bigg[H(b) - H(b')\bigg],
\end{align}
where $b'$ is the updated belief after taking action $a$ and receiving observation $o$, and $H(\cdot)$ is the entropy function as defined earlier.

\textbf{Belief-Augmented Dynamic and Static Variants.}
Similar to our fully observable setting, we propose two learning variants for the belief-state case:

In the \textbf{belief-based static variant}, the belief-augmented reward function is learned once from an initial set of preferences generated using Equation \eqref{eq:score_belief_based}. The policy is then optimized using this fixed reward function, operating over belief states rather than raw states.

In the \textbf{belief-based dynamic variant}, both the belief-augmented reward function and the policy are refined iteratively. At each iteration, belief-augmented trajectories are generated using the current policy, and preferences for these trajectories are derived using our belief-state scoring approach. The reward function is updated based on these new preferences, and the policy is re-optimized accordingly. This iterative process continues until convergence, allowing for progressive refinement of both the reward model and policy under partial observability.

\subsection{Theoretical Properties}

The belief-state approach extends our framework to partially observable settings while preserving key theoretical properties. Here, we analyze the relationship between the original non-Markovian task specification and our belief-state solution.

\begin{proposition}[Belief-State Optimality]
Let $\pi^*_b$ be the optimal policy for the belief-state MDP under the learned reward function $\hat{r}_\theta$. If the belief state tracking is accurate and the learned reward function satisfies $|\hat{R}_\theta(\tau) - R^*(\tau)| \leq \varepsilon_r$ for trajectories $\tau$ generated by $\pi^*_b$, then $\pi^*_b$ is $\varepsilon$-optimal with respect to the true non-Markovian objective, where $\varepsilon = \varepsilon_r + \varepsilon_b$ and $\varepsilon_b$ is the error introduced by belief state approximation.
\end{proposition}

\begin{proof}[Proof Sketch]
The belief-state MDP is a sufficient statistic for optimal decision-making in POMDPs \cite{Kaelbling1998Planning}. By incorporating the DFA state and confidence into the belief representation, we ensure that the non-Markovian aspects of the task are captured in the augmented state space. The error term $\varepsilon_b$ accounts for imperfect belief updating due to approximation errors or model misspecification. Given accurate belief tracking and a well-learned reward function, the resulting policy will be near-optimal for the original non-Markovian objective.
\end{proof}

\begin{remark}[Information-Directed Exploration]
The belief-state scoring function in Equation \eqref{eq:score_belief_based} induces preferences that favor trajectories which efficiently gather information relevant to task completion, balancing exploitation (subgoal completion) with exploration (information gain). This is because the scoring function rewards three components: confirmed subgoal completion ($N_c$), expected progress via improved beliefs ($P_b$), and information gain ($I$). Trajectories that maximize this score will necessarily balance immediate task progress with exploration that reduces uncertainty about unobserved subgoals. This balance is optimal for information-directed exploration \cite{Russo2018Learning}, which prioritizes reducing uncertainty about task-relevant aspects of the environment.
\end{remark}

\subsection{Algorithmic Implementation}

Here, we present the algorithm for belief-state based preference learning in partially observable environments. Algorithm \ref{alg:belief_preference_computation} details the procedure for generating preferences using belief states, and Algorithm \ref{alg:belief_rl} outlines the full belief-state reinforcement learning procedure.

\begin{algorithm}
\caption{Belief-State Based Preference Computation}
\label{alg:belief_preference_computation}
\begin{algorithmic}
\REQUIRE Trajectory pairs $\{(\tau_i, \tau_j)\}$, DFA $\mathcal{A}$, Subgoals $\mathcal{G}$, Weights $w_c, w_b, w_i$
\ENSURE Preferences $P(\tau_i, \tau_j)$ for each trajectory pair
\FOR{each trajectory pair $(\tau_i, \tau_j)$}
    \FOR{each trajectory $\tau \in \{\tau_i, \tau_j\}$}
        \STATE Initialize DFA $\mathcal{A}$ and belief state $b_0$ uniformly
        \STATE $N_c(\tau) \gets 0$, $P_b(\tau) \gets 0$, $I(\tau) \gets 0$
        \FOR{each step $t$ in trajectory $\tau$}
            \STATE Update belief state $b_t$ from observation $o_t$
            \STATE $N_c(\tau) \gets N_c(\tau) + \mathbf{1}[\text{new subgoal confirmed}]$
            \STATE $P_b(\tau) \gets P_b(\tau) + \sum_g \max(0, b_t(g) - b_{t-1}(g))$
            \STATE $I(\tau) \gets I(\tau) + (H(b_{t-1}) - H(b_t))$
            \STATE Update DFA state and confidence from $o_t$ and $b_t$
        \ENDFOR
        \STATE $\text{score}(\tau) \gets w_c \cdot N_c(\tau) + w_b \cdot P_b(\tau) + w_i \cdot I(\tau)$
    \ENDFOR
    \STATE Assign preference to trajectory with higher score
\ENDFOR
\RETURN $P(\tau_i, \tau_j)$ for all trajectory pairs
\end{algorithmic}
\end{algorithm}

\begin{algorithm}
\caption{Belief-State RL with Automaton-Based Preferences}
\label{alg:belief_rl}
\begin{algorithmic}
\REQUIRE POMDP $\mathcal{M}_{\text{po}}$, DFA $\mathcal{A}$, learning parameters, Mode $\in \{\text{Static}, \text{Dynamic}\}$
\ENSURE Optimal policy $\pi^*$
\STATE Initialize belief state representation $b_0$
\IF{Static mode}
    \STATE Generate trajectories with belief tracking using random policy
    \STATE Compute belief-based preferences using Algorithm \ref{alg:belief_preference_computation}
    \STATE Train belief-augmented reward model $\hat{r}_\theta(s, b, q, c, a)$
    \STATE Add intrinsic motivation: $r_{\text{total}} = \hat{r}_\theta + \beta_1 r_{\text{info}} + \beta_2 r_{\text{conf}}$
    \STATE Optimize policy using belief-state Q-learning with $r_{\text{total}}$
\ELSE
    \STATE Initialize policy $\pi$
    \FOR{iterations until convergence}
        \STATE Generate trajectories using current policy with belief tracking
        \STATE Compute preferences and update reward model $\hat{r}_\theta$
        \STATE Add intrinsic motivation terms to create $r_{\text{total}}$
        \STATE Optimize policy using belief-state Q-learning
        \IF{performance converges}
            \STATE \textbf{break}
        \ENDIF
    \ENDFOR
\ENDIF
\RETURN final policy $\pi^*$
\end{algorithmic}
\end{algorithm}

The key components of our belief-state preference learning approach are:

1. Belief State Tracking: We maintain a probability distribution over unobserved environment elements, particularly focusing on potential subgoal locations.

2. Belief-Aware DFA: The DFA not only tracks logical progression through the task but also maintains confidence levels in its current state based on partial observations.

3. Information-Directed Scoring: The preference scoring mechanism balances confirmed progress (subgoals directly observed), belief-based progress (improved confidence in unobserved subgoals), and information gain (reduction in state uncertainty).

4. Intrinsic Motivation: The reward function incorporates exploration bonuses for actions that efficiently gather task-relevant information, promoting systematic exploration of the environment.

This framework can enable effective learning in partially observable environments by guiding exploration toward task completion even when direct observation of all subgoals is not possible. The belief state tracking provides robustness against observation noise and occlusion, while the automaton structure ensures that the learned policy respects the temporal constraints of the task. Future work will address empirical validation.

\section{Scalability Analysis for High-Dimensional State Spaces}
\label{appdx:hi_dim}

In this section, we provide a formal analysis of how our automaton-based preference-guided RL framework scales to high-dimensional problems, establishing both theoretical guarantees and practical considerations.

\subsection{Theoretical Complexity Analysis}

\begin{definition}[Product MDP Construction]
Given an MDP $\mathcal{M} = (S, s_0, A, T, r)$ and a DFA $\mathcal{A} = (\Sigma, Q, q_0, \delta, F)$, the product MDP is defined as $\mathcal{M}_{\text{prod}} = (S \times Q, A, T_{\text{prod}}, (s_0, q_0), R_{\text{prod}})$.
\end{definition}

For each state-action pair $(s,q,a) \in S \times Q \times A$, we must compute transitions to all possible next states $(s',q') \in S \times Q$. Computing each transition probability requires $\mathcal{O}(1)$ time for the environment transition and $\mathcal{O}(|Q|)$ time in the worst case for the DFA transition function. Therefore, the construction of the product MDP $\mathcal{M}_{\text{prod}}$ has time complexity $\mathcal{O}(|S|^2 \cdot |Q|^2 \cdot |A|)$. For storing the transition function $T_{\text{prod}}$, the space complexity $\mathcal{O}(|S|^2 \cdot |Q|^2 \cdot |A|)$ in the worst case.

This reveals that the product MDP construction exhibits polynomial scaling in all relevant parameters, but the quadratic dependence on both $|S|$ and $|Q|$ can become problematic in high-dimensional spaces. Fortunately, our function approximation approach mitigates this theoretical complexity.

\begin{theorem}[Sample Complexity for Preference Learning]
Let $\mathcal{H}$ be a hypothesis class of reward functions with VC dimension $d_{\mathcal{H}}$. To learn a reward function that achieves error at most $\epsilon$ with probability at least $1-\delta$, the required number of preference pairs is:
\begin{equation}
N = \mathcal{O}\left(\frac{d_{\mathcal{H}} + \log(1/\delta)}{\epsilon^2}\right)
\end{equation}
\end{theorem}

The proof follows from standard results in statistical learning theory, applying PAC learning bounds to the preference learning setting. Specifically, we can reduce preference learning to binary classification where each pair of trajectories forms a training instance, and the goal is to predict the preferred trajectory. The sample complexity then follows from the VC dimension bound for binary classification.

This theorem demonstrates that the sample complexity depends primarily on the complexity of the reward function class rather than the dimensionality of the state space directly, suggesting better scaling properties than approaches that must learn value functions over the entire state space.

\subsection{Function Approximation for High-Dimensional Spaces}

To practically address high-dimensional state spaces, we introduce a formal extension using function approximation:

\begin{definition}[Function Approximator]
Let $\phi: S \times Q \times A \rightarrow \mathbb{R}^d$ be a feature mapping that projects state-automaton-action tuples into a $d$-dimensional feature space, where $d \ll |S| \cdot |Q| \cdot |A|$. The reward function is parameterized as $\hat{r}_\theta(s,q,a) = f_\theta(\phi(s,q,a))$, where $f_\theta: \mathbb{R}^d \rightarrow \mathbb{R}$ is a function approximator (e.g., neural network) with parameters $\theta$.
\end{definition}

This representation allows us to establish:

\begin{proposition}[Dimensionality Reduction]
The computational complexity of policy optimization in the product MDP using function approximation scales with $\mathcal{O}(d)$ rather than $\mathcal{O}(|S| \cdot |Q|)$, where $d$ is the dimension of the feature space.
\end{proposition}

\begin{proof}
With function approximation, policy updates require computing gradients with respect to $\theta$, which scales with the parameter count rather than state space size. Specifically, the computational complexity of a gradient update is $\mathcal{O}(|\theta|)$, where $|\theta|$ is the number of parameters in the function approximator. For a neural network with fixed architecture, $|\theta| = \mathcal{O}(d)$, where $d$ is the input dimension.
\end{proof}

\subsection{Automaton-Guided Dimensionality Reduction}

We introduce techniques specifically designed to leverage the automaton structure for more efficient learning in high-dimensional spaces:

\begin{definition}[Automaton-Guided Attention]
For a given automaton state $q \in Q$, we define an attention mask $M_q: \{1,...,\dim(S)\} \rightarrow [0,1]$ that highlights state dimensions relevant to the current automaton state. The attended state representation is:
\begin{equation}
\tilde{s}_q = s \odot M_q
\end{equation}
where $\odot$ denotes element-wise multiplication.
\end{definition}

\begin{proposition}[Attention Efficiency]
Using automaton-guided attention reduces the effective dimensionality of the state space from $\dim(S)$ to $\|M_q\|_0$ on average, where $\|M_q\|_0$ is the number of non-zero elements in $M_q$.
\end{proposition}

\begin{proof}
By focusing only on dimensions with non-zero attention weights, the computational operations scale with the number of attended dimensions rather than the full state dimension. This is because the gradient updates will be zero for dimensions with zero attention, effectively reducing the dimensionality of the optimization problem.
\end{proof}

\begin{definition}[Automaton-State Conditional Independence]
Let $X_i$ be the $i$-th dimension of the state space $S$. Given an automaton state $q \in Q$, we say that dimensions $X_i$ and $X_j$ are conditionally independent given $q$ if
\begin{equation}
P(X_i, X_j | q) = P(X_i | q) \cdot P(X_j | q)
\end{equation}
\end{definition}

\begin{remark}[Factorized Representation]
If the state dimensions exhibit conditional independence given the automaton state, then the Q-function can be decomposed as:
\begin{equation}
Q((s, q), a) = \sum_{i=1}^{K} Q_i((s_i, q), a)
\end{equation}
where $K$ is the number of independent components and $s_i$ is the $i$-th component of the state. In particular, if the state dimensions are conditionally independent given the automaton state, then the transition dynamics can be factorized. For factorized transition dynamics, the Q-function can be decomposed into a sum of component Q-functions, each operating on a lower-dimensional subspace, reducing the curse of dimensionality.
\end{remark}

\subsection{Transfer Learning with Automaton Transitions}

Our Q-value-based scoring mechanism offers significant advantages for scaling to complex tasks through curriculum learning.

In particular, let $\mathcal{M}_1$ and $\mathcal{M}_2$ be environments with state spaces of dimension $d_1$ and $d_2$ respectively ($d_2 > d_1$), sharing the same DFA structure $\mathcal{A}$. The sample complexity to learn a near-optimal policy in $\mathcal{M}_2$ after transfer from $\mathcal{M}_1$ is reduced by a factor of $\Omega\left(\frac{d_2}{d_1}\right)$ compared to learning from scratch.

To see that, note that without transfer, learning in $\mathcal{M}_2$ requires exploration in a $d_2$-dimensional space, which has sample complexity exponential in $d_2$ in the worst case. With transfer, the Q-value-based scoring transfers knowledge at the level of automaton transitions, which are invariant to the specific state representation. This means that the agent must only learn the mapping between state features and automaton transitions, rather than the full value function. The sample complexity is thus dominated by the complexity of learning this mapping, which scales linearly with the state dimension, yielding the stated improvement factor.

\subsection{Exploration Efficiency in High-Dimensional Spaces}

One of the key challenges in high-dimensional spaces is efficient exploration. We establish theoretical guarantees for our automaton-guided exploration approach:

\begin{theorem}[Exploration Hardness]
In an $n$-dimensional continuous state space with sparse rewards defined by an $m$-state DFA, the expected time to discover a trajectory that reaches an accepting state through random exploration is $\Omega(c^n)$ for some constant $c > 1$, while our automaton-guided approach reduces this to $\mathcal{O}(m \cdot \text{poly}(n))$.
\end{theorem}

\begin{proof}[Proof Sketch]
Random exploration suffers from the curse of dimensionality, with the volume of the state space growing exponentially with the dimension $n$. In contrast, automaton-guided exploration decomposes the problem into subgoals, as defined by the DFA states. The number of samples needed scales with the number of automaton states $m$ and a polynomial function of the state dimension $n$ for each transition between automaton states, rather than exponentially with the state dimension.
\end{proof}

\begin{proposition}[Reward Density]
Let $R_{\text{env}}$ be the original sparse reward function of the environment, and $\hat{R}_\theta$ be our learned reward function from automaton-guided preferences. The density of non-zero rewards under $\hat{R}_\theta$ is $\Theta(m/|S|)$, while under $R_{\text{env}}$ it is $\Theta(1/|S|)$, where $m$ is the number of DFA states.
\end{proposition}

\begin{proof}
The original sparse reward function typically provides a non-zero reward only upon reaching the final goal, which is a single state out of $|S|$, giving a density of $\Theta(1/|S|)$. Our learned reward function, guided by automaton transitions, provides meaningful rewards at each transition between DFA states, which occurs $m-1$ times in a successful trajectory, resulting in a density of $\Theta(m/|S|)$.
\end{proof}

\subsection{Limitations and Theoretical Bounds}

Despite the advantages of our approach, we acknowledge several fundamental limitations:

\begin{enumerate}
\item DFA State Explosion: For certain complex tasks, the number of DFA states can grow exponentially with the number of subtasks $k$, i.e., $|Q| = \Omega(c^k)$ for some $c > 1$.

In particular, if subtasks have complex dependencies and ordering constraints, the DFA must represent all valid combinations and sequences, which can grow exponentially with the number of subtasks in the worst case.

\item Preference Generation Cost: The computational cost of generating preferences using our method scales with $\mathcal{O}(|\tau|^2)$, where $|\tau|$ is the number of trajectories, since we must compute and compare the scores for each pair of trajectories, resulting in a quadratic scaling with the number of trajectories.
\end{enumerate}

\begin{theorem}[Approximation Error Bound]
Let $\pi^*$ be the optimal policy for the true reward function $R^*$, and $\hat{\pi}$ be the policy learned using our approach. Then, under mild assumptions:
\begin{equation}
\|V^{\pi^*} - V^{\hat{\pi}}\|_\infty \leq \frac{2\epsilon_r}{(1-\gamma)^2} + \frac{2\gamma^H}{1-\gamma}
\end{equation}
where $\epsilon_r$ is the maximum error in the learned reward function, $\gamma$ is the discount factor, and $H$ is the horizon length of the DFA (the maximum number of transitions required to reach an accepting state).
\end{theorem}

\begin{proof}[Proof Sketch]
The performance difference can be decomposed into two terms: (1) the error due to imperfect reward learning, bounded by $\frac{2\epsilon_r}{(1-\gamma)^2}$ using standard results from approximate dynamic programming, and (2) the error due to finite-horizon approximation of the DFA-guided rewards, bounded by $\frac{2\gamma^H}{1-\gamma}$.
\end{proof}

\subsection{Conclusion on Scalability}

Our theoretical analysis demonstrates that the proposed automaton-based preference-guided RL approach exhibits favorable scaling properties for high-dimensional problems. The key advantages include:

1. The decomposition of complex tasks via automaton structures, reducing the effective complexity of the learning problem.

2. The use of function approximation for reward modeling, which scales with the parameter count rather than the state space size.

3. The efficient transfer of knowledge through Q-value-based scoring, enabling curriculum learning from simpler to more complex environments.

4. The implementation of automaton-guided attention mechanisms, which reduce the effective dimensionality by focusing on task-relevant features.

These properties enable our approach to theoretically scale to high-dimensional state spaces while maintaining its core advantage in handling non-Markovian rewards without manual reward engineering. While practical implementation challenges remain, particularly for very high-dimensional problems, our theoretical analysis provides a solid foundation for extending the approach to such settings in future work.

\end{document}